\documentclass[11pt, a4paper, oneside, reqno]{article}
\usepackage[utf8]{inputenc}    
\usepackage{cmap}              
\usepackage[T1]{fontenc}       
\usepackage{lmodern}           
\usepackage{natbib}
\usepackage[colorlinks=true]{hyperref}
\usepackage{amssymb, amsthm, amsmath, amsfonts, verbatim}
\usepackage{authblk, mathtools, longfbox, dsfont, mathrsfs, color, bm, enumitem, url, upgreek}
\usepackage{calc}
\usepackage{graphicx, transparent, wrapfig, subcaption, caption}
\usepackage{titlesec}
\usepackage{algpseudocode}
\usepackage{multicol}
\usepackage{tcolorbox}
\usepackage{algorithm}
\usepackage{todonotes}
\titlelabel{\thetitle.\;}

\DeclareMathOperator*{\argmin}{argmin}
\DeclareMathOperator*{\argmax}{argmax}
\DeclareMathOperator*{\spmax}{spmax}
\DeclareMathOperator{\lip}{lip}
\DeclarePairedDelimiter{\ceil}{\lceil}{\rceil}

\newcommand{\EE}{\mathbb{E}}

\newcommand{\bs}{\boldsymbol}

\newcommand{\acc}{\delta}

\newcommand{\R}{\mathbb{R}}
\newcommand{\eps}{\varepsilon}
\newcommand{\mpx}{\mathcal{P}(\x)}
\newcommand{\mpy}{\mathcal{P}(\y)}

\newcommand{\x}{\mathcal{X}}
\newcommand{\y}{\mathcal{Y}}

\newcommand{\lx}{\mathcal{L}(\x, \mu)}
\newcommand{\ly}{\mathcal{L}(\y, \nu)}

\newcommand{\asinh}{\textrm{arcsinh}}

\newcommand{\mc}{\mathcal}

\newcommand{\DS}{\displaystyle}
\newcommand{\st}{\mathrm{s.t.}}
\newcommand{\diff}{\mathrm{d}}

\newcommand{\opt}{}

\theoremstyle{definition}
\newtheorem{theorem}{Theorem}[section]

\newtheorem{assumption}[theorem]{Assumption}
\newtheorem{proposition}[theorem]{Proposition}
\newtheorem{corollary}[theorem]{Corollary}
\newtheorem{lemma}[theorem]{Lemma}

\newtheorem{example}[theorem]{Example}

\graphicspath{{figures/}}

\usepackage{xcolor}
\hypersetup{colorlinks, linkcolor={red!50!black}, citecolor={blue!45!black}, urlcolor={green!50!black}}
\usepackage[a4paper, total={7in, 10in}]{geometry}

\usepackage{accents}
\newcommand{\ubar}[1]{\underaccent{\bar}{#1}}
\allowdisplaybreaks
\author[$\dagger$]{Bahar Ta{\c{s}}kesen}
\author[$\ddagger$]{Soroosh Shafieezadeh-Abadeh}
\author[$\dagger$]{Daniel Kuhn}

\affil[$\dagger$]{Risk Analytics and Optimization Chair, EPFL Lausanne \authorcr \texttt{bahar.taskesen,daniel.kuhn@epfl.ch}}
\affil[$\ddagger$]{Automatic Control Laboratory, ETH Zurich \authorcr \texttt{shafiee@ethz.ch}}

\date{}

\begin{document}
	\title{Semi-Discrete Optimal Transport:\\ Hardness, Regularization and Numerical Solution
	}
	\maketitle
	
	\begin{abstract}
		Semi-discrete optimal transport problems, which evaluate the Wasserstein distance between a discrete and a generic (possibly non-discrete) probability measure, are believed to be computationally hard. Even though such problems are ubiquitous in statistics, machine learning and computer vision, however, this perception has not yet received a theoretical justification. To fill this gap, we prove that computing the Wasserstein distance between a discrete probability measure supported on two points and the Lebesgue measure on the standard hypercube is already $\#$P-hard. This insight prompts us to seek {\em approximate} solutions for semi-discrete optimal transport problems. We thus perturb the underlying transportation cost with an additive disturbance governed by an ambiguous probability distribution, and we introduce a distributionally robust {\em dual} optimal transport problem whose objective function is smoothed with the most adverse disturbance distributions from within a given ambiguity set. We further show that smoothing the dual objective function is equivalent to regularizing the primal objective function, and we identify several ambiguity sets that give rise to several known and new regularization schemes. As a byproduct, we discover an intimate relation between semi-discrete optimal transport problems and discrete choice models traditionally studied in psychology and economics. To solve the regularized optimal transport problems efficiently, we use a stochastic gradient descent algorithm with imprecise stochastic gradient oracles. A new convergence analysis reveals that this algorithm improves the best known convergence guarantee for semi-discrete optimal transport problems with entropic regularizers.\\

	\end{abstract}
	
	\section{Introduction}
	Optimal transport theory has a long and distinguished history in mathematics dating back to the seminal work of \citet{monge1781memoire} and \citet{kantorovich1942transfer}. While originally envisaged for applications in civil engineering, logistics and economics, optimal transport problems provide a natural framework for comparing probability measures and have therefore recently found numerous applications in statistics and machine learning. Indeed, the minimum cost of transforming a probability measure~$\mu$ on $\x$ to some other probability measure~$\nu$ on $\y$ with respect to a prescribed cost function on~$\x\times \y$ can be viewed as a measure of distance between~$\mu$ and $\nu$. If $\x=\y$ and the cost function coincides with (the $p^{\text{th}}$ power of) a metric on~$\x\times \x$, then the resulting optimal transport distance represents (the $p^{\text{th}}$ power of) a Wasserstein metric on the space of probability measures over~$\x$ \citep{villani}. In the remainder of this paper we distinguish {\em discrete}, {\em semi-discrete} and {\em continuous} optimal transport problems in which either both, only one or none of the two probability measures~$\mu$ and~$\nu$ are discrete, respectively.

	In the wider context of machine learning, {\em discrete} optimal transport problems are nowadays routinely used, for example, in the analysis of mixture models \citep{kolouri2017optimal, nguyen2013convergence} as well as in image processing \citep{alvarez2017structured, ferradans2014regularized, kolouri2015transport, papadakis2017convex, tartavel2016wasserstein}, computer vision and graphics \citep{pele2008linear, pele2009fast, rubner2000earth, solomon2014earth, solomon2015convolutional}, data-driven bioengineering \citep{feydy2017optimal, kundu2018discovery, wang2010optimal}, clustering \citep{ho2017multilevel}, dimensionality reduction \citep{cazelles2018geodesic, flamary2018wasserstein, rolet2016fast, schmitzer2016sparse, seguy2015principal}, domain adaptation \citep{courty2016optimal, murez2018image}, distributionally robust optimization \citep{esfahani2018data, nguyen2020distributionally, NIPS2015_5745, shafieezadeh2019regularization}, scenario reduction \citep{heitsch2007note, rujeerapaiboon2018scenario}, scenario generation \citep{pflug2001scenario, hochreiter2007financial}, the assessment of the fairness properties of machine learning algorithms \citep{gordaliza2019obtaining, taskesen2020distributionally, taskesen2020statistical} and signal processing \citep{thorpe2017transportation}. 
	
	The discrete optimal transport problem represents a tractable linear program that is susceptible to the network simplex algorithm \citep{orlin1997polynomial}. Alternatively, it can be addressed with dual ascent methods \citep{bertsimas1997introduction}, the Hungarian algorithm for assignment problems \citep{kuhn1955hungarian} or customized auction algorithms \citep{bertsekas1981new, bertsekas1992auction}.
	The currently best known complexity bound for computing an {\em exact} solution is attained by modern interior-point algorithms. 
	Indeed, if $N$ denotes the number of atoms in~$\mu$ or in~$\nu$, whichever is larger, then the discrete optimal transport problem can be solved in time\footnote{We use the soft-O notation $\tilde{\mc O}(\cdot)$ to hide polylogarithmic factors.}~$\mathcal{\tilde{O}}(N^{2.5})$ with an interior point algorithm by \citet{lee2014path}. 
	The need to evaluate optimal transport distances between increasingly fine-grained histograms has also motivated efficient approximation schemes. \citet{blanchet2018towards} and \citet{quanrud2018approximating} show that an $\epsilon$-optimal solution can be found in time~$\mathcal{O}(N^2/\epsilon)$ by reducing the discrete optimal transport problem to a matrix scaling or a positive linear programming problem, which can be solved efficiently by a Newton-type algorithm.
    \citet{jambulapati2019direct} describe a parallelizable primal-dual first-order method that achieves a similar convergence~rate.
	
	The tractability of the discrete optimal transport problem can be improved by adding an entropy regularizer to its objective function, which penalizes the entropy of the transportation plan for morphing~$\mu$ into~$\nu$. When the weight of the regularizer grows, this problem reduces to the classical Schr\"odinger bridge problem of finding the most likely random evolution from~$\mu$ to~$\nu$ \citep{schrodinger1931umkehrung}. Generic linear programs with entropic regularizers were first studied by \citet{fang1992unconstrained}. \citet{cominetti1994asymptotic} prove that the optimal values of these regularized problems converge exponentially fast to the optimal values of the corresponding unregularized problems as the regularization weight drops to zero. 
	Non-asymptotic convergence rates for entropy regularized linear programs are derived by \citet{weed2018explicit}.
	\citet{sinkhorn} was the first to realize that entropic penalties are computationally attractive because they make the discrete optimal transport problem susceptible to a fast matrix scaling algorithm by \citet{sinkhorn1967diagonal}. This insight has spurred widespread interest in machine learning and led to a host of new applications of optimal transport in	color transfer \citep{chizat2016scaling}, inverse problems \citep{karlsson2017generalized, adler2017learning}, texture synthesis \citep{peyre2017quantum}, the analysis of crowd evolutions \citep{peyre2015entropic} and shape interpolation \citep{solomon2015convolutional} to name a few. This surge of applications inspired in turn several new algorithms for the entropy regularized discrete optimal transport problem such as a greedy dual coordinate descent method also known as the Greenkhorn algorithm \citep{altschuler2017near, chakrabarty2018better, abid2018greedy}.
	\citet{dvurechensky2018computational} and \citet{lin2019efficient} prove that both the Sinkhorn and the Greenkhorn algorithms are guaranteed to find an $\epsilon$-optimal solution in time $\tilde{\mathcal{O}}({N^2}/{\epsilon^2})$. In practice, however, the Greenkhorn algorithm often outperforms the Sinkhorn algorithm~\citep{lin2019efficient}. The runtime guarantee of both algorithms can be improved to $\tilde{\mathcal{O}}(N^{7/3}/\epsilon)$ via a randomization scheme \citep{lin2019acceleration}.
	In addition, the regularized discrete optimal transport problem can be addressed by tailoring general-purpose optimization algorithms such as accelerated gradient descent algorithms \citep{dvurechensky2018computational}, iterative Bregman projections \citep{benamou2015iterative}, quasi-Newton methods \citep{blondel2017smooth} or stochastic average gradient descent algorithms \citep{genevay2016stochastic}. While the original optimal transport problem induces sparse solutions, the entropy penalty forces the optimal transportation plan of the regularized optimal transport problem to be strictly positive and thus completely dense. In applications where the interpretability of the optimal transportation plan is important, the lack of sparsity could be undesirable; examples include color transfer \citep{pitie2007automated}, domain adaptation \citep{courty2016optimal} or
	ecological inference \citep{muzellec2017tsallis}. 
	Hence, there is merit in exploring alternative regularization schemes that retain the attractive computational properties of the entropic regularizer but induce sparsity. Examples that have attracted significant interest include smooth convex regularization and Tikhonov regularization \citep{dessein2018regularized, blondel2017smooth, seguy2017large, essid2018quadratically}, Lasso regularization \citep{li2016fast}, Tsallis entropy regularization \citep{muzellec2017tsallis} or group Lasso regularization \citep{courty2016optimal}. 
	
	Much like the discrete optimal transport problems, the significantly more challenging \emph{semi-discrete} optimal transport problems emerge in numerous applications including variational inference \citep{ambrogioni2018wasserstein}, blue noise sampling \citep{qin2017wasserstein}, computational geometry \citep{levy2015numerical}, image quantization \citep{de2012blue} or deep learning with generative adversarial networks \citep{arjovsky2017wasserstein, genevay2017learning, gulrajani2017improved}. Semi-discrete optimal transport problems are also used in fluid mechanics to simulate incompressible fluids \citep{de2015power}.
	
	Exact solutions of a semi-discrete optimal transport problem can be constructed by solving an incompressible Euler-type partial differential equation discovered by \citet{brenier1991polar}. Any optimal solution is known to partition the support of the non-discrete measure into cells corresponding to the atoms of the discrete measure \citep{aurenhammer1998minkowski}, and the resulting tessellation is usually referred to as a power diagram. 
	\citet{ref:mirebeau2015discretization} uses this insight to solve Monge-Amp{\`e}re equations with a damped Newton algorithm, and \citet{kitagawa2016convergence} show that a closely related algorithm with a global linear convergence rate lends itself for the numerical solution of generic semi-discrete optimal transport problems. In addition, \citet{merigot2011multiscale} proposes a quasi-Newton algorithm for semi-discrete optimal transport, which improves a method due to~\citet{aurenhammer1998minkowski} by exploiting Llyod's algorithm to iteratively simplify the discrete measure.
	If the transportation cost is quadratic, \citet{bonnotte2013knothe} relates the optimal transportation plan to the Knothe-Rosenblatt rearrangement for mapping~$\mu$ to~$\nu$, which is very easy to compute.
	
	As usual, regularization improves tractability. \citet{genevay2016stochastic} show that the dual of a semi-discrete optimal transport problem with an entropic regularizer is susceptible to an averaged stochastic gradient descent algorithm that enjoys a convergence rate of~$\mathcal O(1/\sqrt{T})$, $T$ being the number of iterations. {\color{black} \citet{ref:altschuler2022asymptotics} show that the optimal value of the entropically regularized problem converges to the optimal value of the unregularized problem at a quadratic rate as the regularization weight drops to zero. Improved error bounds under stronger regularity conditions are derived by~\citet{ref:delalande2021nearly}.}

	{\em Continuous} optimal transport problems constitute difficult variational problems involving infinitely many variables and constraints. \citet{benamou2000computational} recast them as boundary value problems in fluid dynamics, 
	and \citet{papadakis2014optimal} solve discretized versions of these reformulations using first-order methods. For a comprehensive survey of the interplay between partial differential equations and optimal transport we refer to \citep{evans1997partial}. As nearly all numerical methods for partial differential equations suffer from a curse of dimensionality, current research focuses on solution schemes for {\em regularized} continuous optimal transport problems. 
	For instance, \citet{genevay2016stochastic} embed their duals into a reproducing kernel Hilbert space to obtain finite-dimensional optimization problems that can be solved with a stochastic gradient descent algorithm. \citet{seguy2017large} solve regularized continuous optimal transport problems by representing the transportation plan as a multilayer neural network. This approach results in finite-dimensional optimization problems that are non-convex and offer no approximation guarantees. However, it provides an effective means to compute approximate solutions in high dimensions. {\color{black} Indeed, the optimal value of the entropically regularized continuous optimal transport problem is known to converge to the optimal value of the unregularized problem at a linear rate as the regularization weight drops to zero \citep{ref:chizat2020faster, ref:conforti2021formula, ref:erbar2015large,ref:pal2019difference}.} Due to a lack of efficient algorithms, applications of continuous optimal transport problems are scarce in the extant literature. \citet{peyre2019computational} provide a comprehensive survey of numerous applications and solution methods for discrete, semi-discrete and continuous optimal transport problems.

	This paper focuses on semi-discrete optimal transport problems. Our main goal is to formally establish that these problems are computationally hard, to propose a unifying regularization scheme for improving their tractability and to develop efficient algorithms for solving the resulting regularized problems, {\color{black} assuming only that we have access to independent samples from the continuous probability measure~$\mu$}. Our regularization scheme is based on the observation that any {\em dual} semi-discrete optimal transport problem maximizes the expectation of a piecewise affine function with $N$ pieces, where the expectation is evaluated with respect to~$\mu$, and where $N$ denotes the number of atoms of the discrete probability measure~$\nu$. We argue that this piecewise affine function can be interpreted as the optimal value of a {\em discrete choice problem}, which can be smoothed by adding random disturbances to the underlying utility values \citep{thurstone1927law, mcfadden1973conditional}. As probabilistic discrete choice problems are routinely studied in economics and psychology, we can draw on a wealth of literature in choice theory to design various smooth (dual) optimal transport problems with favorable numerical properties. For maximal generality we will also study {\em semi-parametric} discrete choice models where the disturbance distribution is itself subject to uncertainty \citep{natarajan2009persistency, mishra2014theoretical, feng2017relation, ahipasaoglu2018convex}. Specifically, we aim to evaluate the best-case (maximum) expected utility across a Fr\'echet ambiguity set containing all disturbance distributions with prescribed marginals. Such models can be addressed with customized methods from modern distributionally robust optimization \citep{natarajan2009persistency}. For Fr\'echet ambiguity sets, we prove that smoothing the dual objective is equivalent to regularizing the primal objective of the semi-discrete optimal transport problem. The corresponding regularizer penalizes the discrepancy between the chosen transportation plan and the product measure~$\mu\otimes \nu$ with respect to a divergence measure constructed from the marginal disturbance distributions. Connections between primal regularization and dual smoothing were previously recognized by \citet{blondel2017smooth} and \citet{paty2020regularized} in discrete optimal transport and by \citet{genevay2016stochastic} in semi-discrete optimal transport. As they are constructed ad hoc or under a specific adversarial noise model, these existing regularization schemes lack the intuitive interpretation offered by discrete choice theory and emerge as special cases of our unifying scheme.

	The key contributions of this paper are summarized below.
	\begin{enumerate}[label=\roman*.]
		\item We study the computational complexity of semi-discrete optimal transport problems. Specifically, we prove that computing the optimal transport distance between two probability measures~$\mu$ and~$\nu$ on the same Euclidean space is $\#$P-hard {\color{black} even if only approximate solutions are sought and} even if~$\mu$ is the Lebesgue measure on the standard hypercube and~$\nu$ is supported on merely two points. 
		\item We propose a unifying framework for regularizing semi-discrete optimal transport problems by leveraging ideas from distributionally robust optimization and discrete choice theory \citep{natarajan2009persistency, mishra2014theoretical, feng2017relation, ahipasaoglu2018convex}. Specifically, we perturb the transportation cost to every atom of the discrete measure~$\nu$ with a random disturbance, and we assume that the vector of all disturbances is governed by an uncertain probability distribution from within a Fr\'echet ambiguity set that prescribes the marginal disturbance distributions. Solving the dual optimal transport problem under the least favorable disturbance distribution in the ambiguity set amounts to smoothing the dual and regularizing the primal objective function. We show that numerous known and new regularization schemes emerge as special cases of this framework, and we derive a priori approximation bounds for the resulting regularized optimal transport problems.
		\item We derive new convergence guarantees for an averaged stochastic gradient descent (SGD) algorithm that has only access to a {\em biased} stochastic gradient oracle. Specifically, we prove that this algorithm enjoys a convergence rate of $\mathcal O(1/\sqrt{T})$ for Lipschitz continuous and of~$\mathcal O(1/T)$ for generalized self-concordant objective functions. We also show that this algorithm lends itself to solving the smooth dual optimal transport problems obtained from the proposed regularization scheme. When the smoothing is based on a semi-parametric discrete choice model with a Fr\'echet ambiguity set, the algorithm's convergence rate depends on the smoothness properties of the marginal noise distributions, and its per-iteration complexity depends on our ability to compute the optimal choice probabilities. We demonstrate that these choice probabilities can indeed be computed efficiently via bisection or sorting, and in special cases they are even available in closed form. As a byproduct, we show that our algorithm can improve the state-of-the-art $\mathcal O(1/\sqrt{T})$ convergence guarantee of~\citet{genevay2016stochastic} for the semi-discrete optimal transport problem with an {\em entropic} regularizer.
	\end{enumerate}
	
	The rest of this paper unfolds as follows. In Section~\ref{sec:complexity} we study the computational complexity of semi-discrete optimal transport problems, and in Section~\ref{sec:smooth_ot} we develop our unifying regularization scheme.
	In Section~\ref{sec:computation} we analyze the convergence rate of an averaged SGD algorithm with a biased stochastic gradient oracle that can be used for solving smooth dual optimal transport problems, and in Section~\ref{sec:numerical} we compare its empirical convergence behavior against the theoretical convergence guarantees.
	
	\paragraph{Notation.}
	We denote by $\|\cdot\|$ the 2-norm, by $[N] = \{1, \ldots, N \}$ the set of all integers up to $N\in\mathbb N$ and by $\Delta^d = \{\bs x \in \mathbb R_+^d : \sum_{i = 1}^d x_i =1\}$ the probability simplex in $\mathbb R^d$. 
	For a logical statement $\mathcal E$ we define $\mathds{1}_{\mathcal E} = 1$ if $\mathcal E$ is true and $\mathds{1}_{\mathcal E} = 0$ if $\mathcal E$ is false.
	For any closed set $\x\subseteq \R^d$ we define $\mathcal{M}(\x)$ as the family of all Borel measures and $\mathcal{P}(\x)$ as its subset of all Borel probability measures on~$\x$. 
	For $\mu\in\mathcal{P}(\x)$, we denote by $\mathbb{E}_{\bs x \sim \mu}[\cdot]$ the expectation operator under $\mu$ and define $\mathcal{L}(\x, \mu)$ as the family of all $\mu$-integrable functions $f:\x\rightarrow\R$, that is, $f \in \mathcal{L}(\x, \mu)$ if and only if $\int_{\x} |f(\bs x)| \mu(\diff \bs x)<\infty$. 
	The Lipschitz modulus of a function $f: \R^d \to \R$ is defined as $\lip(f) = \sup_{\bs x, \bs{x}'}\{|f(\bs x) - f(\bs{x}')|/\|\bs x - \bs{x}'\|: \bs x \neq \bs{x}'\}$. The convex conjugate of $f: \R^d \to [-\infty,+\infty]$ is the function $f^*:\R^d\rightarrow [-\infty,+\infty]$ defined through $f^{*}(\bs y) = \sup_{\bs x \in \R^d}\bs y^\top \bs x - f(\bs x)$. 
	
	\section{Hardness of Computing Optimal Transport Distances}
	\label{sec:complexity}
	If $\x$ and $\y$ are closed subsets of finite-dimensional Euclidean spaces and $c: \x \times \y \to [0,+\infty]$ is a lower-semicontinuous cost function, then the Monge-Kantorovich {\em optimal transport distance} between two probability measures $\mu\in\mathcal P(\x)$ and $\nu\in\mathcal P(\y)$ is defined as
	\begin{equation}
	W_c(\mu, \nu) = \min\limits_{\pi \in \Pi(\mu,\nu)} ~ \mathbb{E}_{(\bs x, \bs y) \sim \pi}\left[{c(\bs x, \bs y)}\right],
	\label{eq:primal}
	\end{equation}
	where $\Pi(\mu,\nu)$ denotes the family of all  {\em couplings} of $\mu$ and $\nu$, that is, the set of all probability measures on $\x \times \y$ with marginals $\mu$ on $\x$ and $\nu$ on $\y$. One can show that the minimum in~\eqref{eq:primal} is always attained \citep[Theorem~4.1]{villani}. If $\x=\y$ is a metric space with metric $d:\x\times \x\rightarrow \R_+$ and the transportation cost is defined as $c(\bs x, \bs y)=d^p(\bs x,\bs y)$ for some $p \geq 1$, then $W_c(\mu, \nu)^{1/p}$ is termed the $p$-th Wasserstein distance between $\mu$ and $\nu$. The optimal transport problem~\eqref{eq:primal} constitutes an infinite-dimensional linear program over measures and admits a strong dual linear program over functions \citep[Theorem~5.9]{villani}.
	\begin{proposition}[Kantorovich duality]
		\label{prop:kantorovich}
		The optimal transport distance between $\mu \in \mpx$ and $\nu \in \mpy$ admits the dual representation
		\begin{equation}
		\label{eq:dual}
		W_c(\mu, \nu) =\left\{ \begin{array}{c@{\quad}l@{\qquad}l}
		\sup & \DS \mathbb{E}_{\bs y \sim \nu}\left[ {\phi(\bs y)}\right] - \mathbb{E}_{\bs x \sim \mu}\left[{\psi(\bs x)}\right] & \\ [0.5em]
		\st & \psi \in \lx,~ \phi \in \ly & \\ [0.5em]
		& \phi(\bs y) - \psi(\bs x) \leq c(\bs x, \bs y) \quad \forall \bs x \in \x,~ \bs y \in \y.
		\end{array}\right.
		\end{equation}
	\end{proposition}
	The linear program~\eqref{eq:dual} optimizes over the two {\em Kantorovich potentials} $\psi \in \lx$ and $\phi \in \ly$, but it can be reformulated as the following non-linear program over a single potential function,
	\begin{equation}
	\label{eq:ctans_dual}
	W_c(\mu, \nu) =\sup_{\phi \in \mathcal{L}(\y, \nu)} ~ \DS \mathbb{E}_{\bs y \sim \nu}\left[\phi(\bs y)\right] - \mathbb{E}_{\bs x \sim \mu}\left[ \phi_c(\bs x) \right],
	\end{equation}
	where $\phi_c:\x\rightarrow [-\infty,+\infty]$ is called the \textit{$c$-transform} of $\phi$ and is defined through
	\begin{equation}
	\label{eq:$c$-transform}
	\phi_c(\bs x) = \sup_{\bs y \in \y} ~ \phi(\bs y) - c(\bs x, \bs y) \qquad \forall \bs x \in \x,
	\end{equation}
	see \citet[\S~5]{villani} for details. The Kantorovich duality is the key enabling mechanism to study the computational complexity of the optimal transport problem~\eqref{eq:primal}. 
	
	\begin{theorem}[Hardness of computing optimal transport distances]
		\label{theorem:hard}
		Computing $W_c(\mu, \nu)$ is \#P-hard even if $\x=\y=\R^d$, $c(\bs x, \bs y) = \|\bs x-\bs y\|^{p}$ for some $p\geq 1$, $\mu$ is the Lebesgue measure on the standard hypercube~$[0,1]^d$, and $\nu$ is a discrete probability measure supported on only two points.
	\end{theorem}
	

	To prove Theorem~\ref{theorem:hard}, we will show that computing the optimal transport distance $W_c(\mu, \nu)$ is at least as hard computing the volume of the knapsack polytope $P( \bs{w}, b) = \{\bs x\in [0,1]^d :  \bs{w}^\top  \bs x\leq b\}$ for a given $\bs{w}\in \R^d_+$ and $ b \in \R_+$, which is known to be $\#$P-hard \citep[Theorem~1]{dyer1988complexity}. Specifically, we will leverage the following variant of this hardness result, which establishes that approximating the volume of the knapsack polytope $P( \bs{w}, b)$ to a sufficiently high accuracy is already $\#$P-hard.

	
	\begin{lemma}[{\citet[Lemma~1]{Grani}}]
		\label{lemma:Grani}
		Computing the volume of the knapsack polytope $P( \bs{w}, b)$ for a given $\bs{w}\in \R^d_+$ and $ b \in \R_+$ to within an absolute accuracy of $\delta>0$ is $\#$P-hard
		whenever 
		\begin{equation}
		\label{eq:Granis-delta}
		\acc <\frac{1}{ {2d!(\|\bs{w}\|_1+2)^d(d+1)^{d+1}\prod_{i = 1}^{d}w_i}}.
		\end{equation}
	\end{lemma}
	Fix now any knapsack polytope $P( \bs{w}, b)$ encoded by $\bs{w}\in \R_+^d$ {\color{black}and $ b \in  \R_+$. Without loss of generality, we may assume that~$\bm w \neq \bm 0$ and~$b > 0$. Indeed, we are allowed to exclude~$\bm w = \bm 0 $ because the volume of~$P(\bm  0, b) $ is trivially equal to 1. On the other hand, $b= 0$ can be excluded by applying a suitable rotation and translation, which are volume-preserving transformations.} In the remainder, we denote by $\mu$ the Lebesgue measure on the standard hypercube $[0,1]^d$ and by ${\nu}_ t  = t \delta_{\bs y_1} + (1-t) \delta_{\bs y_2}$ a family of discrete probability measures with two atoms at $\bs y_1=\bs 0$ and $\bs y_2=2b\bs{w}/ \|\bs w\|^2$, respectively, whose probabilities are parameterized by $t \in [0, 1]$.
	The following preparatory lemma relates the volume of $P( \bs{{w}},b)$ to the optimal transport problem~\eqref{eq:primal} and is thus instrumental for the proof of Theorem~\ref{theorem:hard}.
	\begin{lemma}
		\label{lem:vol}
		If $c(\bs x, \bs y)=\|\bs x- \bs y \|^p$ for some $p\ge 1$, then we have $\mathop{\rm Vol}(P( \bs{{w}},b)) = \argmin_{ t  \in [0,1]} W_c(\mu,  {\nu}_ t )$. 
	\end{lemma}
	\begin{proof}
		By the definition of the optimal transport distance in~\eqref{eq:primal} and our choice of~$c(\bs x, \bs y)$, we have
		\begin{align*}
		\underset{ t  \in [0,1]}{\min}W_c(\mu,  {\nu}_ t )&= \underset{ t  \in [0,1]}{\min} ~ \min\limits_{\pi \in \Pi(\mu,\nu_t)} ~ \mathbb{E}_{(\bs x, \bs y)\sim\pi}\left[\|\bs x- \bs y \|^p \right] \\[0.5ex] &=\min\limits_{ t  \in [0,1]}~ \left\{\begin{array}{cl} \min\limits_{q_1, q_2 \in \mathcal{P}(\R^d)}& t \displaystyle\int_{\R^d} \|\bs x-\bs y_1\|^p q_1(\diff \bs x) + (1-t) \displaystyle\int_{\R^d}\left \|\bs x-\bs y_2 \right\|^p q_2(\diff \bs x)\\ [3ex]
		\textrm{s.t.}& t \cdot  q_1 + (1-t) \cdot q_2 = \mu,
		\end{array}\right.
		\end{align*}
		where the second equality holds because any coupling $\pi$ of $\mu$ and $\nu_t$ can be constructed from the marginal probability measure $\nu_t$ of $\bs y$ and the probability measures $q_1$ and $q_2$ of $\bs x$ conditional on $\bs y =\bs y_1$ and $\bs y = \bs y_2$, respectively, that is, we may write $\pi= t\cdot q_1\otimes \delta_{\bs y_1} + (1-t)\cdot q_2\otimes \delta_{\bs y_2}$. The constraint of the inner minimization problem ensures that the marginal probability measure of $\bs x$ under $\pi$ coincides with $\mu$. By applying the variable transformations $q_1\leftarrow t \cdot q_1 $ and $q_2 \leftarrow (1-t)\cdot q_2$ to eliminate all bilinear terms, we then obtain
		\begin{equation*}
		\underset{ t  \in [0,1]}{\min}W_c(\mu,  {\nu}_ t )=\left\{\begin{array}{cll}
		\underset{\substack{ t \in [0,1] \\ q_1, q_2 \in \mathcal{M}(\R^d)}}{\min} &\displaystyle\int_{\R^d} \|\bs x -\bs y_1\|^p q_1(\diff \bs x) + \displaystyle\int_{\R^d} \left\|\bs x-\bs y_2 \right\|^p  q_2(\diff \bs x)\\[3ex]
		\textrm{s.t.} &\displaystyle\int_{\R^d} q_1(\diff \bs x) =  t  \\[3ex]
		&\displaystyle\int_{\R^d} q_2(\diff \bs x) = 1-  t  \\[3ex]
		&  q_1 +  q_2 = \mu.
		\end{array}\right.
		\end{equation*}
		Observe next that the decision variable $t$ and the two normalization constraints can be eliminated without affecting the optimal value of the resulting infinite-dimensional linear program because the Borel measures $q_1$ and $q_2$ are non-negative and because the constraint $q_1+q_2=\mu$ implies that $q_1(\R^d)+q_2(\R^d)=\mu(\R^d)=1$. Thus, there always exists $t\in[0,1]$ such that $q_1(\R^d)=t$ and $q_2(\R^d)=1-t$. This reasoning implies that
		\begin{equation*}
		\underset{ t  \in [0,1]}{\min}W_c(\mu,  {\nu}_ t )=\left\{\begin{array}{ccll}
		&\min\limits_{q_1,q_2\in \mathcal{M}(\R^d)}\; &
		\displaystyle\int_{\R^d} \|\bs x -\bs y_1\|^p q_1(\diff \bs x) + \displaystyle\int_{\R^d}\left \|\bs x-\bs y_2 \right\|^p  q_2(\diff \bs x) \\[3ex]
		& \textrm{s.t.} & q_1 +  q_2= \mu.
		\end{array}\right.
		\end{equation*} 
		The constraint $q_1+q_2=\mu$ also implies that $q_1$ and $q_2$ are absolutely continuous with respect to $\mu$, and~thus 
		\begin{align}
		\underset{ t  \in [0,1]}{\min}W_c(\mu,  {\nu}_ t )& =\left\{\begin{array}{ccll}
		&\min\limits_{q_1,q_2\in \mathcal{M}(\R^d)}\; &
		\displaystyle\int_{\R^d} \|\bs x -\bs y_1\|^p \frac{\diff q_1}{\diff \mu}(\bs x) +\left \|\bs x-\bs y_2 \right\|^p \, \frac{\diff q_2}{\diff \mu}(\bs x)\, \mu(\diff \bs x) \\[3ex]
		& \textrm{s.t.} & \displaystyle \frac{\diff q_1}{\diff \mu}(\bs x) +  \frac{\diff q_2}{\diff \mu}(\bs x)= 1 \quad \forall \bs x\in [0,1]^d
		\end{array}\right. \nonumber\\
		&= \int_{\R^d} \min\left\{\|\bs x -\bs y_1 \|^p,\left\|\bs x - \bs y_2 \right\|^p \right\}\,\mu(\diff\bs x),
		\label{eq:min-Wc}
		\end{align} 
		where the second equality holds because at optimality the Radon-Nikodym derivatives must satisfy
		\[
		\frac{\diff q_i}{\diff \mu}(\bs x)=\left\{\begin{array}{cl}
		1 & \text{if } \|\bs x-\bs y_i\|^p \le \|\bs x-\bs y_{3-i}\|^p \\
		0 & \text{otherwise}
		\end{array} \right.
		\]
		for $\mu$-almost every $\bs x\in \R^d$ and for every $i=1,2$. 

		In the second part of the proof we will demonstrate that the minimization problem $\min_{t\in[0,1]} W_c(\mu, \nu_ t )$ is solved by $t^\star=\textrm{Vol}(P(\bs{w}, b))$. By Proposition~\ref{prop:kantorovich} and the definition of the $c$-transform, we first note that
		\begin{align}
		W_c(\mu, \nu_ {t^\star} ) 
		&=\underset{\phi \in \mathcal{L}(\R^d, \nu_{t^\star})}{\max} ~ \mathbb{E}_{\bs y\sim \nu_{t^\star}}[\phi(\bs y)] - \mathbb{E}_{\bs x\sim\mu}[\phi_c(\bs x)] \nonumber \\
		\label{eq:min_wass}
		&= \underset{\bs \phi \in \R^2}{\max} ~ t^\star \cdot \phi_1   + (1- t^\star ) \cdot\phi_2- \int_{\R^d}\max_{i=1,2}\left\{\phi_i- \|\bs x-\bs y_i \|^p\right\}\mu(\diff \bs x)\\
		&= \max\limits_{\bs \phi \in \R^2} ~ t^\star \cdot \phi_1   + (1-t^\star)\cdot \phi_2- \sum\limits_{i = 1}^2 \int_{\x_i(\bs \phi)}(\phi_i - \|\bs x - \bs {y_i}\|^p)\,\mu(\diff \bs x), \nonumber
		\end{align}
		where 
		\[
		\x_i(\bs \phi) = \{\bs x\in\R^d: \phi_i - \|\bs x-\bs y_i \|^p \geq  \phi_{3-i} - \left \|\bs x - \bs y_{3-i} \right\|^p\}\quad \forall i=1,2.
		\] 
		The second equality in~\eqref{eq:min_wass} follows from the construction of $\nu_{t^\star}$ as a probability measure with only two atoms at the points $\bs y_i$ for $i=1,2$. Indeed, by fixing the corresponding function values $\phi_i=\phi(\bs y_i)$ for $i=1,2$, the expectation $\mathbb{E}_{\bs y \sim \nu_{t\opt}}[\phi(\bs y)]$ simplifies to $t^\star \cdot \phi_1   + (1-t^\star)\cdot \phi_2$, while the negative expectation $-\mathbb{E}_{\bs x \sim \mu}[\phi_c(\bs x)]$ is maximized by setting $\phi(\bm y)$ to a large negative constant for all $\bs y\notin\{\bs y_1,\bs y_2\}$, which implies~that
		\[
		\phi_c(\bs x) = \sup_{\bs y \in \R^d} \phi(\bs y) - \|\bs x - \bs y\|^p = \max_{i=1,2}\left\{\phi_i- \|\bs x-\bs y_i \|^p\right\} \quad \forall \bs x\in [0,1]^d.
		\]
		Next, we will prove that any $\bs \phi^\star\in\R^2$ with $\phi^\star_1=\phi^\star_2$ attains the maximum of the unconstrained convex optimization problem on the last line of~\eqref{eq:min_wass}. To see this, note that
		\[
		\nabla_{\bs \phi} \left[\sum\limits_{i = 1}^2 \int_{\x_i(\bs \phi)}(\phi_i - \|\bs x - \bs {y}_i\|^p)\,\mu(\diff \bs x)\right] = \sum\limits_{i = 1}^2 \int_{\x_i(\bs \phi)} \nabla_{\bs \phi}(\phi_i - \|\bs x - \bs {y}_i\|^p)\,\mu(\diff \bs x) =\begin{bmatrix} \mu(\x_1(\bs \phi))\\ \mu(\x_2(\bs \phi)) \end{bmatrix}
		\]
		by virtue of the Reynolds theorem. Thus, the first-order optimality condition\footnote{Note that the first-order condition $1-t^\star=\mu(\x_2(\bs \phi))$ for $\phi_2$ is redundant in view of the first-order condition $t^\star=\mu(\x_1(\bs \phi))$ for $\phi_1$ because $\mu$ is the Lebesgue measure on $[0,1]^d$, whereby $\mu(\x_1(\bs \phi)\cup\x_2(\bs \phi))=\mu(\x_1(\bs \phi))+\mu(\x_2(\bs \phi))=1$.} $t^\star=\mu(\x_1(\bs \phi))$ is necessary and sufficient for global optimality. Fix now any $\bs \phi^\star\in\R^2$ with $\phi^\star_1=\phi^\star_2$ and observe that
		\begin{align*}
		t^\star=\textrm{Vol}(P(\bs{w}, b)) =& \mu\left(\left\{\bs x\in\R^d: \bs w^\top\bs x\leq b \right\}\right)\\
		=&\mu\left( \left\{\bs x\in\R^d: \|\bs x \|^2\leq \|\bs x-2b \bs w/\|\bs w\|^2\|^2 \right\}\right)\\
		=& \mu\left(\left\{\bs x\in\R^d: \|\bs x -\bs y_1\|^p\leq \|\bs x-\bs y_2\|^p \right\}\right)=\mu(\x_1(\bs \phi^\star)),
		\end{align*}
		where the first and second equalities follow from the definitions of $t^\star$ and the knapsack polytope $P(\bs{w}, b)$, respectively, the fourth equality holds because $\bs y_1=\bs 0$ and $\bs y_2=2b\bs w/\|\bs w\|^2$, and the fifth equality follows from the definition of $\x_1(\bs \phi^\star)$ and our assumption that $\phi^\star_1=\phi^\star_2$. This reasoning implies that $\bs \phi^\star$ attains indeed the maximum of the optimization problem on the last line of~\eqref{eq:min_wass}. Hence, we find
		\begin{align*}
		W_c(\mu, \nu_ {t^\star} ) &= t^\star \cdot \phi^\star_1   + (1-t^\star)\cdot \phi^\star_2- \sum\limits_{i = 1}^2 \int_{\x_i(\bs \phi^\star)}(\phi^\star_i - \|\bs x - \bs {y_i}\|^p)\,\mu(\diff \bs x)\\
		&= \sum\limits_{i = 1}^2 \int_{\x_i(\bs \phi^\star)} \|\bs x - \bs {y_i}\|^p \,\mu(\diff \bs x) = \int_{\R^d} \min_{i=1,2}\left\{\|\bs x -\bs y_i \|^p\right\}\,\mu(\diff\bs x) =\underset{ t  \in [0,1]}{\min}W_c(\mu,  {\nu}_ t ),
		\end{align*}
		where the second equality holds because $\phi^\star_1=\phi^\star_2$, the third equality exploits the definition of $\x_1(\bs \phi^\star)$, and the fourth equality follows from~\eqref{eq:min-Wc}. We may therefore conclude that $t^\star=\textrm{Vol}(P(\bs{w}, b))$ solves indeed the minimization problem $\min_{t\in[0,1]} W_c(\mu, \nu_ t )$. {\color{black} Using similar techniques, one can further prove that~$\partial_t W_c(\mu, \nu_t)$ exists and is strictly increasing in~$t$, which ensures that~$W_c(\mu, \nu_t)$ is strictly convex in~$t$ and, in particular, that~$t^\star $ is the unique solution of $\min_{t\in[0,1]} W_c(\mu, \nu_ t )$. Details are omitted for brevity.}
	\end{proof}

	\begin{proof}[Proof of Theorem~\ref{theorem:hard}.]
		Lemma~\ref{lem:vol} applies under the assumptions of the theorem, and therefore the volume of the knapsack polytope~$P(\bs{w}, b)$ coincides with the unique minimizer of
		\begin{equation}
		\label{eq:min-Wass}
		\min_{ t  \in [0,1]} W_c(\mu,  {\nu}_ t ).
		\end{equation}
		{\color{black} From the proof of Lemma~\ref{lem:vol} we know that the Wasserstein distance $W_c(\mu,{\nu}_ t )$ is strictly convex in~$t$, which implies that the minimization problem~\eqref{eq:min-Wass} constitutes a one-dimensional convex program with a unique minimizer. A near-optimal solution that approximates the exact minimizer to within an absolute accuracy~$\delta=(6d!(\|\bs{w}\|_1+2)^d(d+1)^{d+1}\prod_{i = 1}^{d}w_i)^{-1}$ can readily be computed with a binary search method such as  Algorithm~\ref{algorithm:binary} described in Lemma~\ref{lemma:strictly_convex_min}\,(i), which evaluates~$g(t)=W_c(\mu,\nu_t)$ at exactly~$2L=2(\ceil{\log_2(1/\delta)} + 1)$ test points. Note that~$\delta$ falls within the interval $(0, 1)$ and satisfies the strict inequality~\eqref{eq:Granis-delta}. Note also that~$L$ grows only polynomially with the bit length of $\bm w$ and $b$; see Appendix~\ref{appendix:polynomial_calls} for details. One readily verifies that all operations in Algorithm~\ref{algorithm:binary} except for the computation of~$W_c(\mu, \nu_t)$ can be carried out in time polynomial in the bit length of~$\bs{w}$ and~$b$.} Thus, if we could compute $W_c(\mu, \nu_t)$ in time polynomial in the bit length of $\bs{w}$, $b$ and $t$, then we could efficiently compute the volume of the knapsack polytope~$P( \bs{w}, b)$ to within accuracy $\delta$, which is $\#$P-hard by Lemma~\ref{lemma:Grani}. We have thus constructed a polynomial-time Turing reduction from the $\#$P-hard problem of (approximately) computing the volume of a knapsack polytope to computing the Wasserstein distance $W_c(\mu,  {\nu}_ t )$. By the definition of the class of $\#$P-hard problems (see, {\em e.g.}, \citep[Definition~1]{ref:van1990handbook}), we may thus conclude that computing $W_c(\mu, \nu_t)$ is $\#$P-hard.
	\end{proof}  
	{\color{black} 
	\begin{corollary}[Hardness of computing approximate optimal transport distances]
	\label{corollary:approximate-hard}
	Computing $W_c(\mu, \nu)$ to within an absolute accuracy of
	\[
	    \eps =\frac{1}{4} \min\limits_{l\in [ 2^L]} \left\{ |W_c(\mu, \nu_{t_{l}}) - W_c(\mu, \nu_{t_{l-1}})| : W_c(\mu, \nu_{t_{l}}) \neq W_c(\mu, \nu_{t_{l-1}})\right\},
	\]
	where $L = \ceil{\log_2(1/ \delta)} + 1$, $\delta = (6 d!(\|\bs{w}\|_1+2)^d(d+1)^{d+1}\prod_{i = 1}^{d}w_i)^{-1} $ and $t_l = l/ 2^{L}$ for all~$l =0, \ldots, 2^L$, is \#P-hard even if $\x=\y=\R^d$, $c(\bs x, \bs y) = \|\bs x-\bs y\|^{p}$ for some $p\geq 1$, $\mu$ is the Lebesgue measure on the standard hypercube~$[0,1]^d$, and $\nu$ is a discrete probability measure supported on only two points.
	\end{corollary}
	\begin{proof}
		{\color{black} 
		Assume that we have access to an inexact oracle that outputs, for any fixed~$t\in[0,1]$, an approximate optimal transport distance~$\widetilde W_c(\mu, \nu_t)$ with~$|\widetilde W_c(\mu, \nu_t) - W_c(\mu, \nu_t) |\leq \eps$. By Lemma~\ref{lemma:strictly_convex_min}\,(ii), which applies thanks to the definition of~$\eps$, we can then find a $2\delta$-approximation for the unique minimizer of~\eqref{eq:min-Wass} using $2L$ oracle calls. Note that $\delta'=2\delta$ falls within the interval $(0, 1)$ and satisfies the strict inequality~\eqref{eq:Granis-delta}. Recall also that $L$ grows only polynomially with the bit length of $\bm w$ and $b$; see Appendix~\ref{appendix:polynomial_calls} for details. Thus, if we could compute $\widetilde W_c(\mu, \nu_t)$ in time polynomial in the bit length of $\bs{w}$, $b$ and $t$, then we could efficiently compute the volume of the knapsack polytope~$P( \bs{w}, b)$ to within accuracy $\delta'$, which is $\#$P-hard by Lemma~\ref{lemma:Grani}. Computing $W_c(\mu, \nu)$ to within an absolute accuracy of~$\eps$ is therefore also $\#$P-hard.
		}
	\end{proof}
	}
	The hardness of optimal transport established in Theorem~\ref{theorem:hard} {\color{black} and Corollary~\ref{corollary:approximate-hard}} is predicated on the hardness of numerical integration. A popular technique to reduce the complexity of numerical integration is smoothing, whereby an initial (possibly discontinuous) integrand is approximated with a differentiable one \citep{dick2013high}.
	Smoothness is also a desired property of objective functions when designing scalable optimization algorithms \citep{bubeck2015convex}.
	These observations prompt us to develop a systematic way to smooth the optimal transport problem that leads to efficient approximate numerical solution schemes. 
	\section{Smooth Optimal Transport}
	\label{sec:smooth_ot}
	The semi-discrete optimal transport problem evaluates the optimal transport distance~\eqref{eq:primal} between an arbitrary probability measure~$\mu$ supported on $\x$ and a discrete probability measure $\nu = \sum_{i=1}^N {\nu}_i\delta_{\bs {y_i}}$ with atoms~$\bs y_1,\ldots, \bs {y}_N \in \y$ and corresponding probabilities $\bs \nu =(\nu_1,\ldots, \nu_N)\in \Delta^N$ for some $N\ge 2$. 
	In the following, we define the {\em discrete $c$-transform} $\psi_c:\R^N\times \x \rightarrow [-\infty,+\infty)$ of $\bs\phi\in\R^N$ through
	\begin{equation}
	\label{eq:disc-c-transform}
	\psi_c(\bs \phi, \bs x) = \max\limits_{i \in [N]} \phi_i - c(\bs x, \bs y_i) \quad \forall \bs x \in \x.
	\end{equation}
	Armed with the discrete $c$-transform, we can now reformulate the semi-discrete optimal transport problem as a finite-dimensional maximization problem over a single dual potential vector.
	\begin{lemma}[Discrete $c$-transform]
		\label{lem:disc_ctrans}
		The semi-discrete optimal transport problem is equivalent to 
		\begin{equation}
		\label{eq:ctans_dual_semidisc}
		W_c(\mu, \nu) = \sup_{ \bs{\phi} \in \R^N} \bs {\nu}^\top \bs {\phi} - \mathbb{E}_{\bs x \sim \mu}[{\psi_c(\bs \phi, \bs x) } ].
		\end{equation}
	\end{lemma}
	\begin{proof}
		As $\nu = \sum_{i=1}^N {\nu}_i\delta_{\bs {y_i}}$ is discrete, the dual optimal transport problem~\eqref{eq:ctans_dual} simplifies to
		\begin{align*}
		W_c(\mu, \nu) & =\sup_{\bs \phi\in\R^N} \sup_{\phi \in \mathcal{L}(\y, \nu)} \left\{ \bs \nu^\top\bs \phi - \mathbb{E}_{\bs x \sim \mu}\left[ \phi_c(\bs x) \right]\;:\;\phi(\bs y_i)=\phi_i~\forall i\in[N] \right\}\\
		&=\sup_{\bs \phi\in\R^N}~ \bs \nu^\top\bs \phi - \inf_{\phi \in \mathcal{L}(\y, \nu)} \Big\{  \mathbb{E}_{\bs x \sim \mu}\left[ \phi_c(\bs x) \right]\;:\;\phi(\bs y_i)=\phi_i~\forall i\in[N] \Big\} . 
		\end{align*}
		Using the definition of the standard $c$-transform, we can then recast the inner minimization problem as
		\begin{align*}
		&\inf_{\phi \in \mathcal{L}(\y, \nu)} \left\{  \mathbb{E}_{\bs x \sim \mu}\left[ \sup_{\bs y \in \mathcal{Y}} \phi(\bs y) - c(\bs x, \bs y) \right]\;:\;\phi(\bs y_i)=\phi_i~\forall i\in[N] \right\}\\
		&\quad = ~\mathbb{E}_{\bs x \sim \mu} \left[\max_{i \in [N]}\left\{\phi_i- c(\bs x, \bs y_i)\right\} \right] ~=~ \mathbb{E}_{\bs x \sim \mu} \left[{\psi_c(\bs \phi, \bs x) } \right],
		\end{align*}
		where the first equality follows from setting $\phi(\bm y)=\underline \phi$ for all $\bs y\notin\{\bs y_1, \ldots, \bs y_N\}$ and letting~$\underline\phi$ tend to $-\infty$, while the second equality exploits the definition of the discrete $c$-transform. Thus, \eqref{eq:ctans_dual_semidisc} follows.
	\end{proof}
	
	
	The discrete $c$-transform~\eqref{eq:disc-c-transform} can be viewed as the optimal value of a {\em discrete choice model}, where a utility-maximizing agent selects one of $N$ mutually exclusive alternatives with utilities $\phi_i - c(\bs x, \bs y_i)$, $i\in[N]$, respectively. Discrete choice models are routinely used for explaining the preferences of travelers selecting among different modes of transportation \citep{ben1985discrete}, but they are also used for modeling the choice of residential location \citep{mcfadden1978modeling}, the interests of end-users in engineering design \citep{wassenaar2003approach} or the propensity of consumers to adopt new technologies \citep{hackbarth2013consumer}.
	
	In practice, the preferences of decision-makers and the attributes of the different choice alternatives are invariably subject to uncertainty, and it is impossible to specify a discrete choice model that reliably predicts the behavior of multiple individuals. Psychological theory thus models the utilities as random variables \citep{thurstone1927law}, in which case the optimal choice becomes random, too. The theory as well as the econometric analysis of probabilistic discrete choice models were pioneered by \citet{mcfadden1973conditional}.

	
	The availability of a wealth of elegant theoretical results in discrete choice theory prompts us to add a random noise term to each deterministic utility value $\phi_i - c(\bs x, \bs y_i)$ in~\eqref{eq:disc-c-transform}. We will argue below that the expected value of the resulting maximal utility with respect to the noise distribution provides a smooth approximation for the $c$-transform $\psi_c(\bs \phi, \bs x)$, which in turn leads to a smooth optimal transport problem that displays favorable numerical properties. For a comprehensive survey of additive random utility models in discrete choice theory we refer to \citet{dubin1984econometric} and \citet{daganzo2014multinomial}. Generalized semi-parametric discrete choice models where the noise distribution is itself subject to uncertainty are studied by \citet{natarajan2009persistency}. Using techniques from modern distributionally robust optimization, these models evaluate the best-case (maximum) expected utility across an ambiguity set of multivariate noise distributions. Semi-parametric discrete choice models are studied in the context of appointment scheduling \citep{mak2015appointment}, traffic management \citep{ahipacsaouglu2016flexibility} and product line pricing \citep{qi2019product}.

	We now define the {\em smooth} ({\em discrete}) {\em $c$-transform} as a best-case expected utility of the type studied in semi-parametric discrete choice theory, that is,
	\begin{equation}
	\label{eq:smooth_c_transform}
	\overline{\psi}_c(\bs \phi, \bs x) = \sup_{\theta \in \Theta}\;\mathbb{E}_{\bs z 
		\sim\theta}\left[\max_{i \in [N]} \phi_i -c(\bs x, \bs {y_i}) +z_i \right],
	\end{equation}
	where $\bs z$ represents a random vector of perturbations that are independent of $\bs x$ and $\bs y$. Specifically, we assume  that~$\bs z$ is governed by a Borel probability measure $\theta$ from within some ambiguity set $\Theta\subseteq\mathcal{P}(\R^N)$. Note that if~$\Theta$ is a singleton that contains only the Dirac measure at the origin of~$\R^N$, then the smooth $c$-transform collapses to ordinary $c$-transform defined in \eqref{eq:disc-c-transform}, which is piecewise affine and thus non-smooth in~$\bs \phi$. For many commonly used ambiguity sets, however, we will show below that the smooth $c$-transform is indeed differentiable in $\bs \phi$. In practice, the additive noise~$z_i$ in the transportation cost could originate, for example, from uncertainty about the position~$\bs y_i$ of the $i$-th atom of the discrete distribution~$\nu$. This interpretation is justified if $c(\bs x,\bs y)$ is approximately affine in~$\bs y$ around the atoms~$\bs y_i$, $i\in[N]$. The  smooth $c$-transform gives rise to the following {\em smooth} ({\em semi-discrete}) {\em optimal transport problem} in dual form.
	\begin{equation}
	\label{eq:smooth_ot}
	\overline W_c (\mu, \nu) 
	= \sup\limits_{\bs {\phi} \in \R^N} \mathbb{E}_{\bs x \sim \mu} \left[ \bs \nu^\top \bs \phi - \overline\psi_c(\bs \phi, \bs x)\right]
	\end{equation}
	Note that~\eqref{eq:smooth_ot} is indeed obtained from the original dual optimal transport problem~\eqref{eq:ctans_dual_semidisc} by replacing the original $c$-transform $\psi_c(\bs \phi, \bs x)$ with the smooth $c$-transform $\overline\psi_c(\bs \phi, \bs x)$. As smooth functions are susceptible to efficient numerical integration, we expect that~\eqref{eq:smooth_ot} is easier to solve than~\eqref{eq:ctans_dual_semidisc}. A key insight of this work is that the smooth {\em dual} optimal transport problem~\eqref{eq:smooth_ot} typically has a primal representation of the~form
	\begin{equation}
	\label{eq:reg_ot_pri_abstract}
	\min\limits_{\pi \in \Pi(\mu,\nu)}\mathbb E_{(\bs x, \bs y) \sim \pi}\left[ c(\bs x, \bs y)\right] + R_\Theta(\pi),
	\end{equation}
	where $R_\Theta(\pi)$ can be viewed as a regularization term that penalizes the complexity of the transportation plan~$\pi$. 
	In the remainder of this section we will prove~\eqref{eq:reg_ot_pri_abstract} and derive~$R_\Theta(\pi)$ for different ambiguity sets~$\Theta$. We will see that this regularization term is often related to an $f$-divergence, where $f:\R_+ \to \R \cup \{\infty\}$ constitutes a lower-semicontinuous convex function with $f(1) = 0$. If $\tau$ and $\rho$ are two Borel probability measures on a closed subset $\mathcal Z$ of a finite-dimensional Euclidean space, and if~$\tau$ is absolutely continuous with respect to~$\rho$, then the continuous $f$-divergence form~$\tau$ to $\rho$ is defined as 
	$D_f(\tau \parallel \rho) = \int_{\mathcal Z}  f({\diff \tau}/{\diff \rho}(\bs z)) \rho(\diff \bs z)$, where ${\diff \tau}/{\diff \rho}$ stands for the Radon-Nikodym derivative of $\tau$ with respect to $\rho$. By slight abuse of notation, if $\bs \tau$ and $\bs \rho$ are two probability vectors in~$\Delta^N$ and if $\bs\rho>\bs 0$, then the discrete $f$-divergence form $\bs \tau$ to $\bs \rho$ is defined as
	$D_f(\bs \tau \parallel \bs \rho) = \sum_{i =1}^N f({\tau_i}/{\rho_i}) \rho_i$. The correct interpretation of $D_f$ is usually clear from the context.
	
	The following lemma shows that the smooth optimal transport problem~\eqref{eq:reg_ot_pri_abstract} equipped with an $f$-divergence regularization term is equivalent to a finite-dimensional convex minimization problem. This result will be instrumental to prove the equivalence of~\eqref{eq:smooth_ot} and~\eqref{eq:reg_ot_pri_abstract} for different ambiguity sets~$\Theta$.
	
	\begin{lemma}[Strong duality]
		\label{lem:strong_dual_reg_ot}
		If $\bs\eta\in\Delta^N$ with $\bs\eta>\bs 0$ and $\eta = \sum_{i=1}^N \eta_i \delta_{\bs y_i}$ is a discrete probability measure on~$\y$, then problem~\eqref{eq:reg_ot_pri_abstract} with regularization term $R_\Theta (\pi ) =  D_{f}(\pi\|\mu \otimes \eta)$ is equivalent to
		\begin{align}
		\label{eq:dual_regularized_ot}
		\sup\limits_{ \bs {\phi}\in \R^N} ~
		\mathbb{E}_{\bs x \sim \mu}\left[\min\limits_{\bs p\in \Delta^N} \sum\limits_{i=1}^N{\phi_i\nu_i}- (\phi_i - c(\bs x, \bs {y_i}))p_i +  D_f(\bs p \parallel \bs \eta)
		\right].
		\end{align}
	\end{lemma}
	\begin{proof}[Proof of Lemma~\ref{lem:strong_dual_reg_ot}]
		If $\mathbb{E}_{\bs x \sim \mu}[c(\bs x,\bs y_i)]=\infty$ for some $i\in[N]$, then both~\eqref{eq:reg_ot_pri_abstract} and~\eqref{eq:dual_regularized_ot} evaluate to infinity, and the claim holds trivially. In the remainder of the proof we may thus assume without loss of generality that $\mathbb{E}_{\bs x \sim \mu}[c(\bs x,\bs y_i)]<\infty$ for all $i\in[N]$. Using \citep[Theorem~14.6]{rockafellar2009variational} to interchange the minimization over $\bs p$ with the expectation over $\bs x$, problem~\eqref{eq:dual_regularized_ot} can first be reformulated as
		\begin{equation*}
		\begin{array}{ccccll}
		& &\sup\limits_{ \bs{\phi}\in \R^N} &\min\limits_{\bs p\in\mathcal L_\infty^N(\x,\mu)} ~ &\mathbb{E}_{\bs x \sim \mu}\left[\displaystyle\sum\limits_{i=1}^N{\phi_i\nu_i} - (\phi_i - c(\bs x, \bs {y_i}))p_i(\bs x)+  D_f(\bs p(\bs x)\|\bs \eta)\right]  \\[3ex]
		&&&\textrm{s.t.} &\displaystyle \bs p(\bs x)\in\Delta^N \quad \mu\text{-a.s.},
		\end{array}
		\end{equation*}
		where $\mathcal L_\infty^N(\x,\mu)$ denotes the Banach space of all Borel-measurable functions from~$\x$ to~$\R^N$ that are essentially bounded with respect to~$\mu$. Interchanging the supremum over~$\bs \phi$ with the minimum over~$\bs p$ and evaluating the resulting unconstrained linear program over~$\bs \phi$ in closed form then yields the dual problem 
		\begin{equation}
		\label{eq:primal_dual_relation_final}
		\begin{array}{ccl}
		& \min\limits_{\bs p\in\mathcal L_\infty^N(\x,\mu)} &\displaystyle \mathbb{E}_{\bs x \sim \mu}\Bigg[ \sum\limits_{i=1}^Nc(\bs x, \bs {y_i})p_{i}(\bs x) +\displaystyle D_f (\bs p(\bs x) \! \parallel \!\bs \eta) \Bigg] \\[3ex]
		&\textrm{s.t.} &\displaystyle\mathbb{E}_{\bs x \sim \mu}\left[ \bs p(\bs x)\right] = \bs \nu,\quad \bs p(\bs x)\in\Delta^N \quad \mu\text{-a.s.}
		\end{array}
		\end{equation}
		Strong duality holds for the following reasons. As $c$ and $f$ are lower-semicontinuous and $c$ is non-negative, we may proceed as in~\citep[\S~3.2]{shapiro2017distributionally} to show that the dual objective function is weakly${}^*$ lower semicontinuous in $\bs p$. Similarly, as $\Delta^N$ is compact, one can use the Banach-Alaoglu theorem to show that the dual feasible set is weakly${}^*$ compact. Finally, as $f$ is real-valued and $\mathbb{E}_{\bs x \sim \mu}[c(\bs x,\bs y_i)]<\infty$ for all $i\in[N]$, the constant solution $\bs p(\bs x)=\bs \nu$ is dual feasible for all $\bs \nu \in\Delta^N$. Thus, the dual problem is solvable and has a finite optimal value. This argument remains valid if we add a perturbation $\bs \delta\in H=\{\bs\delta'\in\R^N: \sum_{i=1}^N\delta'_i=0\}$ to the right hand side vector~$\bs \nu$ as long as $\bs \delta>-\bs \nu$. The optimal value of the perturbed dual problem is thus pointwise finite as well as convex and---consequently---continuous and locally bounded in~$\bs \delta$ at the origin of~$H$. As $\bs \nu>\bs 0$, strong duality therefore follows from~\citep[Theorem~17\,(a)]{rockafellar1974conjugate}.
		
		Any dual feasible solution $\bs p\in \mathcal L^N_\infty(\x,\mu)$ gives rise to a Borel probability measure $\pi \in \mc P(\mc X \times \mc Y)$ defined through $\pi( \bs y \in \mc B) = \nu(\bs y \in \mc B)$ for all Borel sets $\mc B \subseteq \mathcal Y$ and $\pi(\bs x \in \mc A | \bs y = \bs y_i) = \int_{ \mc A} p_i(\bs x) \mu(\diff \bs x) / \nu_i$ for all Borel sets $\mc A \subseteq \mathcal X$ and $i \in [N]$. This follows from the law of total probability, whereby the joint distribution of~$\bs x$ and~$\bs y$ is uniquely determined if we specify the marginal distribution of~$\bs y$ and the conditional distribution of~$\bs x$ given~$\bs y=\bs y_i$ for every~$i\in[N]$. By construction, the marginal distributions of $\bs x$ and $\bs y$ under $\pi$ are determined by $\mu$ and $\nu$, respectively. Indeed, note that for any Borel set $\mc A \subseteq \mc X$ we have 
		\begin{align*}
		\pi(\bs x \in \mc A) &= \sum\limits_{i=1}^N \pi(\bs x \in \mc A | \bs y = \bs y_i) \cdot \pi(\bs y = \bs y_i) = \sum\limits_{i=1}^N \pi(\bs x \in \mc A | \bs y = \bs y_i) \cdot \nu_i\\ &= \sum\limits_{i=1}^N \int_{\mc A} {p_i(\bs x)}\mu(\diff \bs x) = \int_{\mc A} \mu(\diff \bs x) = \mu(\bs x\in \mathcal A),
		\end{align*}
		where the first equality follows from the law of total probability, the second and the third equalities both exploit the construction of~$\pi$, and the fourth equality holds because $\bs p(\bs x)\in\Delta^N$ $\mu$-almost surely due to dual feasibility. This reasoning implies that $\pi$ constitutes a coupling of $\mu$ and $\nu$ (that is, $\pi \in \Pi(\mu, \nu)$) and is thus feasible in~\eqref{eq:reg_ot_pri_abstract}. Conversely, any $\pi\in\Pi(\mu,\nu)$ gives rise to a function $\bs p\in\mathcal L_\infty^N(\x,\mu)$ defined through  
		\[
		p_i(\bs x) =\nu_i\cdot \frac{\diff \pi}{\diff (\mu \otimes \nu)} (\bs x, \bs y_i)\quad \forall i\in [N].
		\]
		By the properties of the Randon-Nikodym derivative, we have $p_i(\bs x)\ge 0$ $\mu$-almost surely for all $i\in[N]$. In addition, for any Borel set $\mathcal A\subseteq \x$ we have
		\begin{align*}
		\int_{\mathcal A}\sum_{i=1}^N p_i(\bs x)\,\mu(\diff\bs x) & = \int_{\mathcal A} \sum_{i=1}^N \nu_i\cdot \frac{\diff \pi}{\diff (\mu \otimes \nu)} (\bs x, \bs y_i)\,\mu(\diff\bs x)\\ & = \int_{\mathcal A\times \y} \frac{\diff \pi}{\diff (\mu\otimes \nu)} (\bs x, \bs y)\,(\mu\otimes \nu)(\diff \bs x,\diff\bs y) \\&= \int_{\mathcal A\times \y} \pi(\diff\bs x, \diff \bs y) = \int_{\mathcal A}\mu(\diff\bs x),
		\end{align*}
		where the second equality follows from Fubini's theorem and the definition of $\nu=\sum_{i=1}^N\nu_i\delta_{\bs y_i}$, while the fourth equality exploits that the marginal distribution of $\bs x$ under $\pi$ is determined by $\mu$. As the above identity holds for all Borel sets $\mathcal A\subseteq \x$, we find that $\sum_{i=1}^N p_i(\bs x)=1$ $\mu$-almost surely. Similarly, we have
		\begin{align*}
		\mathbb E_{\bs x\sim\mu}\left[ p_i(\bs x)\right] &=\int_\x \nu_i\cdot  \frac{\diff \pi}{\diff (\mu \otimes \nu)} (\bs x, \bs y_i) \,\mu(\diff\bs x) \\ &=\int_{\x\times\{\bs y_i\}} \frac{\diff \pi}{\diff (\mu \otimes \nu)} (\bs x, \bs y) \,(\mu\otimes\nu)(\diff\bs x,\diff\bs y) \\ &= \int_{\x\times\{\bs y_i\}} \pi(\diff\bs x,\diff\bs y)=\int_{\{\bs y_i\}}\nu(\diff \bs y)=\nu_i
		\end{align*}
		for all $i\in[N]$. In summary, $\bs p$ is feasible in~\eqref{eq:primal_dual_relation_final}. Thus, we have shown that every probability measure $\pi$ feasible in~\eqref{eq:reg_ot_pri_abstract} induces a function $\bs p$ feasible in~\eqref{eq:primal_dual_relation_final} and vice versa. We further find that the objective value of~$\bs p$ in~\eqref{eq:primal_dual_relation_final} coincides with the objective value of the corresponding~$\pi$ in~\eqref{eq:reg_ot_pri_abstract}. Specifically, we have
		\begin{align*}
		& \mathbb{E}_{\bs x \sim \mu}\Bigg[ \sum\limits_{i=1}^N c(\bs x, \bs {y_i})\, p_{i}(\bs x) +\displaystyle D_f (\bs p(\bs x) \| \bs \eta) \Bigg] =\displaystyle\int_\x \sum\limits_{i=1}^N c(\bs x, \bs y_i) p_i(\bs x) \,\mu( \diff\bs x) + \displaystyle\int_\x\sum_{i=1}^N f\left(\frac{p_i(\bs x)}{\eta_i}\right) \eta_i \, \mu (\diff \bs x) \\ &\hspace{1cm}=\displaystyle\int_\x \sum\limits_{i=1}^N c(\bs x, \bs y_i) \cdot\nu_i\cdot \frac{\diff \pi}{\diff(\mu \otimes \nu)}(\bs x, \bs y_i)\, \mu( \diff\bs x) +  \int_\x \sum_{i=1}^N f\left( \frac{\nu_i}{\eta_i} \cdot \frac{\diff \pi}{\diff(\mu \otimes \nu)}(\bs x, \bs y_i)\right) \cdot \eta_i \,\mu ( \diff \bs x) \\ &\hspace{1cm}=\displaystyle\int_{\x\times \y} c(\bs x, \bs y)\frac{\diff \pi}{\diff(\mu \otimes \nu)}(\bs x, \bs y) \,(\mu \otimes \nu)(\diff\bs x,  \diff\bs y) +  \displaystyle\int_{\x\times\y} f\left( \frac{\diff \pi}{\diff(\mu \otimes \eta)}(\bs x, \bs y)\right) (\mu\otimes \eta)(\diff \bs x,\diff \bs y) \\[1ex] &\hspace{1cm}=\mathbb E_{(\bs x, \bs y) \sim \pi} \left[c(\bs x, \bs y)\right] +  D_f(\pi \| \mu \otimes \eta),
		\end{align*}
		where the first equality exploits the definition of the discrete $f$-divergence, the second equality expresses the function~$\bs p$ in terms of the corresponding probability measure~$\pi$, the third equality follows from Fubini's theorem and uses the definitions $\nu=\sum_{i=1}^N \nu_i\delta_{\bs y_i}$ and $\eta=\sum_{i=1}^N \eta_i\delta_{\bs y_i}$, and the fourth equality follows from the definition of the continuous $f$-divergence. In summary, we have thus shown that~\eqref{eq:reg_ot_pri_abstract} is equivalent to~\eqref{eq:primal_dual_relation_final}, which in turn is equivalent to~\eqref{eq:dual_regularized_ot}. This observation completes the proof.
	\end{proof}

	\begin{proposition}[Approximation bound]
		\label{prop:approx_bound}
		If $\bs\eta\in\Delta^N$ with $\bs\eta>\bs 0$ and $\eta = \sum_{i=1}^N \eta_i \delta_{\bs y_i}$ is a discrete probability measure on~$\y$, then problem~\eqref{eq:reg_ot_pri_abstract} with regularization term $R_\Theta (\pi ) =  D_{f}(\pi\|\mu \otimes \eta)$ satisfies
		\[|\overline W_c(\mu, \nu) - W_c(\mu, \nu)| \leq  \max\Bigg\{\bigg|\min_{\bs p \in \Delta^N} D_f(\bs p \| \bs \eta )\bigg|, \bigg|\max_{i \in [N]}\bigg\{ f\bigg(\frac{1}{\eta_i}\bigg) \eta_i+ f(0) \sum_{k \neq i} \eta_k\bigg\}\bigg|\Bigg\}.\]
	\end{proposition}
	\begin{proof}
		By Lemma~\ref{lem:strong_dual_reg_ot}, problem~\eqref{eq:reg_ot_pri_abstract} is equivalent to~\eqref{eq:dual_regularized_ot}. Note that the inner optimization problem in~\eqref{eq:dual_regularized_ot} can be viewed as an $f$-divergence regularized linear program with optimal value $\bs \nu^\top \bs\phi-\ell(\bs \phi, \bs x)$, where
		\[
		\ell(\bs \phi, \bs x) = \max\limits_{\bs p \in \Delta^N} \sum\limits_{i=1}^N (\phi_i - c(\bs x, \bs y_i)) p_i -  D_f(\bs p \| \bs \eta).
		\] 
		Bounding $D_f(\bs p \| \bs \eta)$ by its minimum and its maximum over $\bs p\in\Delta^N$ then yields the estimates
		\begin{equation}
		\label{eq:bound:c_trans_ineq}
		\psi_c(\bs \phi, \bs x) - \max_{ \bs p \in \Delta^N} D_f(\bs p \| \bs \eta)  \leq \ell(\bs \phi, \bs x) \leq \psi_c(\bs \phi, \bs x) - \min_{\bs p \in \Delta^N} D_f(\bs p \| \bs \eta).
		\end{equation}
		Here, $\psi_c(\bs \phi, \bs x)$ stands as usual for the discrete $c$-transform defined in~\eqref{eq:disc-c-transform}, which can be represented as
		\begin{equation}
		\label{eq:bound:ot_ineq}
		\psi_c(\bs \phi, \bs x) = \max\limits_{\bs p \in \Delta^N}\sum\limits_{i=1}^N (\phi_i - c(\bs x, \bs y_i)) p_i.
		\end{equation}
		Multiplying~\eqref{eq:bound:c_trans_ineq} by~$-1$, adding $\bs\nu^\top\bs \phi$, averaging over~$\bs x$ using the probability measure~$\mu$ and maximizing over~$\bs \phi \in \R^N$ further implies via~\eqref{eq:ctans_dual_semidisc} and~\eqref{eq:dual_regularized_ot} that
		\begin{equation}
		\label{eq:smooth_ot_approximation_bounds}
		W_c(\mu,\nu)+ \min_{ \bs p \in \Delta^N} D_f(\bs p \| \bs \eta)  \leq \overline W_c(\mu, \nu) \leq W_c(\mu,\nu) + \max_{\bs p \in \Delta^N} D_f(\bs p \| \bs \eta).
		\end{equation}
		As $D_f(\bs p \| \bs \eta)$ is convex in~$\bs p$, its maximum is attained at a vertex of~$\Delta^N$ \citep[Theorem~1]{hoffman1981method}, that~is,
		\[
		\max_{\bs p \in \Delta^N} D_f(\bs p \| \bs \eta) = \max_{i \in [N]}\bigg\{ f\bigg(\frac{1}{\eta_i}\bigg) \eta_i + f(0) \sum_{k \neq i} \eta_k\bigg\}.
		\]
		The claim then follows by substituting the above formula into~\eqref{eq:smooth_ot_approximation_bounds} and rearranging terms.
	\end{proof}
	In the following we discuss three different classes of ambiguity sets $\Theta$ for which the dual smooth optimal transport problem~\eqref{eq:smooth_ot} is indeed equivalent to the primal reguarized optimal transport problem~\eqref{eq:reg_ot_pri_abstract}.
	
	\subsection{Generalized Extreme Value Distributions}
	\label{sec:gevm}
	Assume first that the ambiguity set~$\Theta$ represents a singleton that accommodates only one single
	Borel probability measure $\theta$ on $\R^N$ defined through
	\begin{equation}
	\label{eq:dist_gevd}
	\theta(\bs z \leq \bs s) = \exp \left(-G \left( \exp(-s_1),\ldots, \exp(-s_N) \right) \right)\quad \forall \bs s\in\R^N,
	\end{equation}
	where $G:\R^N \to \R_+$ is a smooth generating function with the following properties. First, $G$ is homogeneous of degree $1/\lambda$ for some $\lambda>0$, that is, for any $\alpha \neq 0$ and $\bs s\in \R^N$ we have $G(\alpha \bs s) = \alpha^{1/\lambda}G(\bs s)$. In addition, $G(\bs s)$ tends to infinity as $s_i$ grows for any $i \in [N]$. Finally, the partial derivative of $G$ with respect to $k$ distinct arguments is non-negative if $k$ is odd and non-positive if $k$ is even. These properties ensure that the noise vector $\bs z$ follows a generalized extreme value distribution in the sense of \citep[\S~4.1]{train2009discrete}. 
	
	\begin{proposition}[Entropic regularization]
		\label{prop:gumbel}
		Assume that $\Theta$ is a singleton ambiguity set that contains only a generalized extreme value distribution with $G( \bs{s}) = \exp(-e)N\sum_{i=1}^N \eta_i s_i^{1/\lambda}$ for some $\lambda > 0$ and~$\bs \eta \in \Delta^N$, $\bs\eta> \bs 0${\color{black}, where~$e$ stands for Euler's constant}. Then, the components of $\bs z$ follow independent Gumbel distributions with means $\lambda \log(N \eta_i)$ and variances $\lambda^2 \pi^2 /6$ for all $i\in[N]$,
		while the smooth $c$-transform~\eqref{eq:smooth_c_transform} reduces to the $\log$-partition function \begin{equation}
		\label{eq:partition:function}
		\overline\psi(\bs \phi, \bs x) = \lambda \log\left(\sum_{i=1}^N \eta_i \exp\left( \frac{\phi_i -c(\bs x,\bs {y_i})}{\lambda} \right) \right).
		\end{equation}
		In addition, the smooth dual optimal transport problem~\eqref{eq:smooth_ot} is equivalent to the regularized primal optimal transport problem~\eqref{eq:reg_ot_pri_abstract} with $R_\Theta(\pi) =  D_f(\pi \| \mu \otimes \eta)$, where $f(s) =\lambda s\log(s)$ and $\eta = \sum_{i =1}^N \eta_i \delta_{\bs y_i}$.
	\end{proposition}
	
	Note that the log-partition function~\eqref{eq:partition:function} constitutes indeed a smooth approximation for the maximum function in the definition~\eqref{eq:disc-c-transform} of the discrete $c$-transform. As $\lambda$ decreases, this approximation becomes increasingly accurate. It is also instructive to consider the special case where $\mu=\sum_{i=1}^M\mu_i\delta_{\bs x_i}$ is a discrete probability measure with atoms $\bs x_1,\ldots,\bs x_M\in\x$ and corresponding vector of probabilities $\bs \mu\in \Delta^M$. In this case, any coupling $\pi\in\Pi(\mu,\nu)$ constitutes a discrete probability measure $\pi=\sum_{i=1}^M\sum_{j=1}^N \pi_{ij}\delta_{(\bs x_i,\bs y_j)}$ with matrix of probabilities $\bs 
	\pi\in\Delta^{M\times N}$. If $f(x)=s\log(s)$, then the continuous $f$-divergence reduces to
	\begin{align*}
	D_f(\pi \| \mu \otimes \eta)&=\sum_{i=1}^M\sum_{j=1}^N \pi_{ij}\log(\pi_{ij})-\sum_{i=1}^M\sum_{j=1}^N \pi_{ij}\log(\mu_i)-\sum_{i=1}^M\sum_{j=1}^N \pi_{ij}\log(\eta_j)\\ 
	&=\sum_{i=1}^M\sum_{j=1}^N \pi_{ij}\log(\pi_{ij})-\sum_{i=1}^M\mu_i\log(\mu_i)-\sum_{j=1}^N \nu_j\log(\eta_j),
	\end{align*}
	where the second equality holds because $\pi$ is a coupling of $\mu$ and $\nu$. Thus, $D_f(\pi \| \mu \otimes \eta)$ coincides with the negative entropy of the probability matrix~$\bs \pi$ offset by a constant that is independent of~$\bs \pi$. For $f(s)=s\log(s)$ the choice of $\bs \eta$ has therefore no impact on the minimizer of the smooth optimal transport problem~\eqref{eq:reg_ot_pri_abstract}, and we simply recover the celebrated entropic regularization proposed by \citet{sinkhorn, genevay2016stochastic, rigollet2018entropic, peyre2019computational} and \cite{ref:clason2019entropic}. 
	

	\begin{proof}[Proof of Proposition~\ref{prop:gumbel}]
		Substituting the explicit formula for the generating function $G$ into~\eqref{eq:dist_gevd} yields 
		\begin{align*}
		\theta(\bs z \leq \bs s) = \exp\left(-\exp(-e)N\sum\limits_{i=1}^N \eta_i \exp\left(-\frac{s_i}{\lambda}\right)\right) &=\prod\limits_{i=1}^N  \exp\left(-\exp(-e)N\eta_i \exp\left(-\frac{s_i}{\lambda}
		\right)\right)\\
		&= \prod\limits_{i=1}^N \exp\left(-\exp\left(-\frac{s_i - \lambda(\log(N\eta_i)-e)}{\lambda}\right)\right),
		\end{align*}
		where $e$ stands for Euler's constant. The components of the noise vector $\bs z$ are thus  independent under~$\theta$, and $z_i$ follows a Gumbel distribution with location parameter $\lambda(\log(N\eta_i)-e)$ and scale parameter $\lambda$ for every $i \in [N]$. Therefore, $z_i$ has mean $\lambda \log(N \eta_i)$ and variance $\lambda^2 \pi^2/6$. 
		
		If the ambiguity set $\Theta$ contains only one single probability measure~$\theta$ of the form~\eqref{eq:dist_gevd}, then Theorem~5.2 of \citet{mcfadden1981econometric} readily implies that the smooth $c$-transform \eqref{eq:smooth_c_transform} simplifies~to
		\begin{equation}
		\label{eq:gev_smooth_ctrans}
		\overline\psi(\bs \phi , \bs x) = \lambda \log G \left(\exp(\phi_1 -c(\bs x,\bs y_1)),\dots, \exp(\phi_N - c(\bs x, \bs {y}_N)) \right) + \lambda e.
		\end{equation}
		The closed-form expression for the smooth $c$-transform in~\eqref{eq:partition:function} follows immediately by substituting the explicit formula for the generating function~$G$ into~\eqref{eq:gev_smooth_ctrans}. One further verifies that~\eqref{eq:partition:function} can be reformulated~as
		\begin{equation}
		\label{eq:partition:reg_c_trans}
		\overline{\psi}_c(\bs \phi, \bs x) = \max\limits_{\bs p \in \Delta^N} \sum\limits_{i=1}^N (\phi_i - c(\bs x, \bs y_i)) p_i - \lambda \sum\limits_{i=1}^N p_i \log\left(\frac{p_i}{\eta_i}\right).
		\end{equation}  
		Indeed, solving the underlying Karush-Kuhn-Tucker conditions analytically shows that the optimal value of the nonlinear program~\eqref{eq:partition:reg_c_trans} coincides with the smooth $c$-transform~\eqref{eq:partition:function}. In the special case where $\eta_i = 1/N$ for all $i \in [N]$, the equivalence of~\eqref{eq:partition:function} and~\eqref{eq:partition:reg_c_trans} has already been recognized by \citet{anderson1988representative}. Substituting the representation~\eqref{eq:partition:reg_c_trans} of the smooth $c$-transform into the dual smooth optimal transport problem~\eqref{eq:smooth_ot} yields~\eqref{eq:dual_regularized_ot} with $f(s)= \lambda s \log(s)$. By Lemma~\ref{lem:strong_dual_reg_ot}, problem~\eqref{eq:smooth_ot} is thus equivalent to the regularized primal optimal transport problem~\eqref{eq:reg_ot_pri_abstract} with $R_\Theta(\pi) = D_f(\pi \| \mu \otimes \eta)$, where $\eta = \sum_{i =1}^N \eta_i \delta_{\bs y_i}$.
	\end{proof}
	
	
	\subsection{Chebyshev Ambiguity Sets}
	\label{sec:chebyshev}
	Assume next that $\Theta$ constitutes a Chebyshev ambiguity set comprising all Borel probability measures on~$\R^N$ with mean vector $\bs 0$ and positive definite covariance matrix $\lambda \bs \Sigma$ for some $\bs \Sigma\succ \bs 0$ and $\lambda> 0$. Formally, we thus set $\Theta = \{\theta \in \mathcal P(\R^N) : \mathbb{E}_\theta [\bs z] = \bs 0,\, \mathbb E_\theta [\bs z \bs z^\top] = \lambda \bs \Sigma\}$.
	In this case, \citep[Theorem~1]{ahipasaoglu2018convex} implies that the smooth $c$-transform~\eqref{eq:smooth_c_transform} can be equivalently expressed as
	\begin{equation}
	\label{eq:moment_ambig_ctrans}
	\overline\psi_c(\bs \phi, \bs x) = \max_{\bs p\in \Delta^N} \sum\limits_{i=1}^N(\phi_i -c(\bs x, \bs {y_i}))p_i + \lambda\,\textrm{tr}\left((\bs \Sigma^{1/2}(\textrm{diag}(\bs p)-\bs p\bs p^\top)\bs \Sigma^{1/2})^{1/2}\right),
	\end{equation}
	where $\textrm{diag}(\bs p)\in\R^{N\times N}$ represents the diagonal matrix with $\bs p$ on its main diagonal. Note that the maximum in~\eqref{eq:moment_ambig_ctrans} evaluates the convex conjugate of the extended real-valued regularization function
	\[
	    V(\bs p)=\left\{ \begin{array}{cl}
	         -\lambda\,\textrm{tr}\left((\bs \Sigma^{1/2}(\textrm{diag}(\bs p)-\bs p\bs p^\top)\bs \Sigma^{1/2})^{1/2}\right) & \text{if }\bs p\in\Delta^N \\
	         \infty & \text{if }\bs p\notin\Delta^N
	    \end{array}\right.
	\]
	at the point $(\phi_i -c(\bs x, \bs {y_i}))_{i\in [N]}$. As $\bs\Sigma\succ \bs 0$ and $\lambda>0$, \citep[Theorem~3]{ahipasaoglu2018convex} implies that~$V(\bs p)$ is strongly convex over its effective domain~$\Delta^N$. By~\cite[Proposition~12.60]{rockafellar2009variational}, the smooth discrete $c$-transform~$\overline\psi_c(\bs \phi, \bs x)$ is therefore indeed differentiable in~$\bs \phi$ for any fixed~$\bs x$. It is further known that problem~\eqref{eq:moment_ambig_ctrans} admits an exact reformulation as a tractable semidefinite program; see \citep[Proposition~1]{mishra2012choice}. If $\bs \Sigma = \bs I$, then the regularization function $V(\bs p)$ can be re-expressed in terms of a discrete $f$-divergence, which implies via Lemma~\ref{lem:strong_dual_reg_ot} that the smooth optimal transport problem is equivalent to the original optimal transport problem regularized with a continuous $f$-divergence.

	\begin{proposition}[Chebyshev regularization] 
	\label{prop:chebyshev-regularization}
	If $\Theta$ is the Chebyshev ambiguity set of all Borel probability measures with mean $\bs 0$ and covariance matrix~$\lambda\bs I$ with $\lambda> 0$, then the smooth $c$-transform~\eqref{eq:smooth_c_transform} simplifies~to
		\begin{equation}
		\label{eq:marginal_moment_ctrans}
		\overline \psi_c(\bs \phi, \bs x) = \max_{ \bs p\in \Delta^N} \sum\limits_{i=1}^N(\phi_i -c(\bs x, \bs {y_i})) p_i + \lambda\sum_{i=1}^N\sqrt{p_i(1-p_i)}.
		\end{equation}
		In addition, the smooth dual optimal transport problem~\eqref{eq:smooth_ot} is equivalent to the regularized primal optimal transport problem~\eqref{eq:reg_ot_pri_abstract} with $R_\Theta(\pi) = D_f(\pi \| \mu \otimes \eta)+ \lambda \sqrt{N-1}$, where $\eta = \frac{1}{N} \sum_{i =1}^N \delta_{\bs y_i}$ and \begin{equation}
		 \label{eq:chebychev_f}
		 f(s) = \begin{cases}
		-\lambda\sqrt{s(N - s)} + \lambda s \sqrt{N-1} \quad & \text{if }0 \leq s \leq N\\
		+\infty & \text{if }s>N.
		\end{cases}\end{equation}
	\end{proposition}
	\begin{proof}
		The relation~\eqref{eq:marginal_moment_ctrans} follows directly from~\eqref{eq:moment_ambig_ctrans} by replacing $\bs \Sigma$ with $\bs I$. Next, one readily verifies that $-\sum_{i \in [N]} \sqrt{p_i(1-p_i)} $ can be re-expressed as the discrete $f$-divergence $D_f(\bs p\| \bs \eta)$ from $\bs p$ to $\bs\eta=(\frac{1}{N},\ldots,\frac{1}{N})$, where $f(s) =-\lambda \sqrt{s (N - s)}+ \lambda \sqrt{N-1}$. This implies that~\eqref{eq:marginal_moment_ctrans} is equivalent to
		\[
		\overline \psi_c(\bs \phi, \bs x) = \max_{ \bs p\in \Delta^N} \sum\limits_{i=1}^N(\phi_i -c(\bs x, \bs {y_i})) p_i -  D_f(\bs p\| \bs \eta).
		\]
		Substituting the above representation of the smooth $c$-transform into the dual smooth optimal transport problem~\eqref{eq:smooth_ot} yields~\eqref{eq:dual_regularized_ot} with $f(s)= -\lambda \sqrt{s (N - s)} +\lambda s \sqrt{N-1} $. By Lemma~\ref{lem:strong_dual_reg_ot},  \eqref{eq:smooth_ot} thus reduces to the regularized primal optimal transport problem~\eqref{eq:reg_ot_pri_abstract} with $R_\Theta(\pi) = D_f(\pi \| \mu \otimes \eta)$, where $\eta = \frac{1}{N} \sum_{i =1}^N \delta_{\bs y_i}$.
	\end{proof}
	
	Note that the function  $f(s)$ defined in~\eqref{eq:chebychev_f} is indeed convex, lower-semicontinuous and satisfies $f(1)=0$. Therefore, it induces a standard $f$-divergence. Proposition~\ref{prop:chebyshev-regularization} can be generalized to arbitrary diagonal matrices $\bs \Sigma$, but the emerging $f$-divergences are rather intricate and not insightful. Hence, we do not show this generalization. We were not able to generalize Proposition~\ref{prop:chebyshev-regularization} to non-diagonal matrices $\bm \Sigma$.

	\subsection{Marginal Ambiguity Sets}
	\label{sec:marginal}
	We now investigate the class of marginal ambiguity sets of the form
	\begin{equation}
	    \label{eq:marginal_ambiguity_set}
	    \Theta = \Big\{ \theta \in \mathcal{P}(\R^N) \, : \, \theta(z_i \leq s) = F_i(s)\;\forall s\in \R, \; \forall i \in [N] \Big\},
	\end{equation}
	where~$F_i$ stands for the cumulative distribution function of the uncertain disturbance~$z_i$, $i\in[N]$. Marginal ambiguity sets completely specify the marginal distributions of the components of the random vector~$\bm z$ but impose no restrictions on their dependence structure ({\em i.e.}, their copula). Sometimes marginal ambiguity sets are also referred to as Fr\'echet ambiguity sets \citep{frechet1951}. We will argue below that the marginal ambiguity sets explain most known as well as several new regularization methods for the optimal transport problem. In particular, they are more expressive than the extreme value distributions as well as the Chebyshev ambiguity sets in the sense that they induce a richer family of regularization terms. Below we denote by $F_i^{-1} : [0, 1] \to \R$ the (left) quantile function corresponding to $F_i$, which is defined through
	$$ 
	    F_i^{-1}(t) = \inf \{s :F_i(s) \geq t \}\quad \forall t\in\R.
	$$
	We first prove that if $\Theta$ constitutes a marginal ambiguity set, then the smooth $c$-transform~\eqref{eq:smooth_c_transform} admits an equivalent reformulation as the optimal value of a finite convex program.
	
	\begin{proposition}[Smooth $c$-transform for marginal ambiguity sets]
		\label{proposition:regularized_ctrans}
		If $\Theta$ is a marginal ambiguity set of the form~\eqref{eq:marginal_ambiguity_set}, and if the underlying cumulative distribution functions $F_i$, $i\in[N]$, are continuous, then the smooth $c$-transform~\eqref{eq:smooth_c_transform} can be  equivalently expressed as
		\begin{equation}
		\label{eq:regularized_c_transform}
		\overline \psi_c(\bs \phi, \bs x) = \max_{ \bs p \in \Delta^N} \displaystyle \sum\limits_{i=1}^N ~ (\phi_i - c(\bs x, \bs {y_i}))p_i + \sum_{i=1}^N \int_{1-p_i}^1 F_i^{-1}(t) \diff t
		\end{equation}
		for all $\bm x\in\x$ and $\bm \phi\in \R^N$. In addition, the smooth $c$-transform is convex and differentiable with respect to $\bs \phi$, and $\nabla_{\bs \phi} \overline\psi_c(\bs \phi, \bs x)$ represents the unique solution of the convex maximization problem~\eqref{eq:regularized_c_transform}.
	\end{proposition}
	Recall that the smooth $c$-transform~\eqref{eq:smooth_c_transform} can be viewed as the best-case utility of a semi-parametric discrete choice model. Thus, \eqref{eq:regularized_c_transform} follows from~\cite[Theorem~1]{natarajan2009persistency}.
	To keep this paper self-contained, we provide a new proof of Proposition~\ref{proposition:regularized_ctrans}, which exploits a natural connection between the smooth $c$-transform induced by a marginal ambiguity set and the conditional value-at-risk (CVaR). 
	
	\begin{proof}[Proof of Proposition~\ref{proposition:regularized_ctrans}]
	Throughout the proof we fix $\bm x\in\x$ and $\bm \phi\in \R^N$, and we introduce the nominal utility vector~$\bs u \in \R^N$ with components~$u_i= \phi_i - c(\bs x, \bs y_i)$ in order to simplify notation. In addition, it is useful to define the binary function $\bs r: \R^N \to \{ 0, 1 \}^N$ with components
	\begin{align*}
	    r_i(\bs z) = 
	    \begin{cases}
	    1 & \text{if } i = \displaystyle \min \argmax_{j \in [N]} ~ u_j + z_j, \\
	    0 & \text{otherwise.}
	    \end{cases}
	\end{align*}
	For any fixed~$\theta \in \Theta$, we then have
	\begin{align*}
	    \EE_{\bs z \sim \theta} \Big[ \max\limits_{i \in [N]} u_i + z_{i} \Big] 
	    = \EE_{\bs z \sim \theta} \Big[ \; \sum_{i=1}^N ( u_i + z_i) r_i(\bs z) \Big]
	    &= \sum_{i=1}^N u_i p_i + \sum_{i=1}^N \EE_{\bs z \sim \theta} \left[ z_i q_i(z_i) \right],
	\end{align*}
	where $p_i = \EE_{\bs z \sim \theta} [ r_i(\bs z) ]$ and $q_i(z_i) = \EE_{\bs z \sim \theta} [ r_i(\bs z) | z_i ]$ almost surely with respect to~$\theta$. From now on we denote by $\theta_i$ the marginal probability distribution of the random variable $z_i$ under $\theta$. As $\theta$ belongs to a marginal ambiguity set of the form~\eqref{eq:marginal_ambiguity_set}, we thus have~$\theta_i (z_i \leq s) = F_i(s)$ for all $s \in \R$, that is, $\theta_i$ is uniquely determined by the cumulative distribution function $F_i$. The above reasoning then implies that
	\begin{align}
	\nonumber
	\overline \psi_c(\bs \phi, \bs x)
	= \sup_{\theta \in \Theta} ~ \EE_{\bs z \sim \theta} \Big[ \max_{i \in [N]} u_i + z_i \Big]
	&= \left\{
	\begin{array}{cll}
	\sup & \DS \sum_{i=1}^N u_i p_i + \sum_{i=1}^N \EE_{\bs z \sim \theta} \left[ z_i q_i(z_i) \right] \\[3ex]
	\text{s.t.} & \theta \in \Theta, ~\bs p \in \Delta^N, ~\bs q \in \mathcal L^N(\R) \\ [1ex]
	& \EE_{\bs z \sim \theta} \left[ r_i(\bs z) \right] = p_i & \forall i \in [N] \\[2ex]
	& \EE_{\bs z \sim \theta} [ r_i(\bs z) | z_i ] = q_i(z_i) \quad \theta\text{-a.s.} & \forall i \in [N]
	\end{array} \right.\\
	&\leq \left\{
	\begin{array}{cll}
	\sup & \DS \sum_{i=1}^N u_i p_i + \sum_{i=1}^N \EE_{z_i \sim \theta_i} \left[ z_i q_i(z_i) \right] \\[3ex]
	\text{s.t.} & \bs p \in \Delta^N,~ \bs q \in \mathcal L^N(\R) \\ [1ex]
	& \EE_{z_i \sim \theta_i} \left[ q_i(z_i) \right] = p_i & \forall i \in [N] \\[2ex]
	& 0 \leq q_i(z_i) \leq 1 \quad \theta_i\text{-a.s.} & \forall i \in [N].
	\end{array} \right.
	\label{eq:upper-bound}
	\end{align}
	The inequality can be justified as follows. One may first add the redundant expectation constraints~$p_i = \EE_{z_i \sim \theta} [q_i(z_i)]$ and the redundant $\theta_i$-almost sure constraints $0\leq q_i(z_i)\leq 1$ to the maximization problem over $
	\theta$, $\bm p$ and $\bm q$ without affecting the problem's optimal value. Next, one may remove the constraints that express $p_i$ and $q_i(z_i)$ in terms of $r_i(\bm z)$. The resulting relaxation provides an upper bound on the original maximization problem. Note that all remaining expectation operators involve integrands that depend on~$\bm z$ only through~$z_i$ for some $i\in[N]$, and therefore the expectations with respect to the joint probability measure~$\theta$ can all be simplified to expectations with respect to one of the marginal probability measures~$\theta_i$. As neither the objective nor the constraints of the resulting problem depend on~$\theta$, we may finally remove~$\theta$ from the list of decision variables without affecting the problem's optimal value. 
	For any fixed $\bs p \in \Delta^N$, the upper bounding problem~\eqref{eq:upper-bound} gives rise the following $N$ subproblems indexed by~$i\in[N]$.
	\begin{subequations}
	\begin{align}
	\label{eq:dual:CVaR}
    \sup_{q_i \in \mathcal L(\R)} \bigg\{ \EE_{z_i \sim \theta_i} \left[ z_i q_i(z_i) \right]:
     \EE_{z_i \sim \theta_i} \left[ q_i(z_i) \right] = p_i, ~
     0 \leq q_i(z_i) \leq 1 ~ \theta_i\text{-a.s.} \bigg\}
    \end{align}
	If $p_i > 0 $, the optimization problem~\eqref{eq:dual:CVaR} over the functions $q_i \in \mathcal L(\R)$ can be recast as an optimization problem over probability measures $\tilde \theta_i \in \mathcal P(\R)$ that are absolutely continuous with respect to~$\theta_i$,
    \begin{align}
    \label{eq:dual:CVaR2}
    \sup_{\tilde \theta_i \in \mathcal P(\R)} \bigg\{ p_i \; \EE_{z_i \sim \tilde \theta_i} \left[ z_i \right]: \frac{\diff \tilde \theta_i}{\diff \theta_i}(z_i) \leq \frac{1}{p_i} ~ \theta_i\text{-a.s.} \bigg\},
    \end{align}
    \end{subequations}
    where $\diff \tilde \theta_i / \diff \theta_i $ denotes as usual the Radon-Nikodym derivative of $\tilde \theta_i$ with respect to $\theta_i$. Indeed, if $q_i$ is feasible in~\eqref{eq:dual:CVaR}, then $\tilde \theta_i$ defined through $\tilde \theta_i[\mathcal B]= \frac{1}{p_i} \int_B q_i(z_i) \theta_i(\diff z_i)$ for all Borel sets $B\subseteq \R$ is feasible in~\eqref{eq:dual:CVaR2} and attains the same objective function value. Conversely, if $\tilde\theta_i$ is feasible in~\eqref{eq:dual:CVaR2}, then $q_i (z_i)= p_i \, \diff \tilde \theta_i / \diff \theta_i (z_i)$ is feasible in~\eqref{eq:dual:CVaR} and attains the same objective function value. Thus, \eqref{eq:dual:CVaR} and~\eqref{eq:dual:CVaR2} are indeed~equivalent. By \cite[Theorem~4.47]{follmer2004stochastic}, the optimal value of~\eqref{eq:dual:CVaR2} is given by $p_i \, \theta_i \text{-CVaR}_{p_i}(z_i) = \int_{1-p_i}^1 F_i^{-1}(t) \diff t$, where $\theta_i \text{-CVaR}_{p_i}(z_i)$ denotes the CVaR of~$z_i$ at level~$p_i$ under~$\theta_i$. 
    
    If $p_i = 0$, on the other hand, then the optimal value of~\eqref{eq:dual:CVaR} and the integral $\int_{1-p_i}^1 F_i^{-1}(t) \diff t$ both evaluate to zero. Thus, the optimal value of the subproblem~\eqref{eq:dual:CVaR} coincides with $\int_{1-p_i}^1 F_i^{-1}(t) \diff t$ irrespective of $p_i$. Substituting this optimal value into~\eqref{eq:upper-bound} finally yields the explicit upper bound 
    \begin{align}
        \label{eq:upper:bound:choice}
        \sup_{\theta \in \Theta} ~ \EE_{z \sim \theta} \Big[ \max\limits_{i \in [N]} u_i + z_i \Big] 
        &\leq \sup_{\bs p \in \Delta^N} ~ \sum_{i=1}^N u_i p_i + \sum_{i=1}^N \int_{1-p_i}^1 F_i^{-1}(t) \diff t.
    \end{align}
    Note that the objective function of the upper bounding problem on the right hand side of~\eqref{eq:upper:bound:choice} constitutes a sum of the strictly concave and differentiable univariate functions $u_i p_i + \int_{1-p_i}^1 F_i^{-1}(t)$. Indeed, the derivative of the $i^{\text{th}}$ function with respect to $p_i$ is given by $u_i + F_i^{-1}(1-p_i)$, which is strictly increasing in~$p_i$ because $F_i$ is continuous by assumption. The upper bounding problem in~\eqref{eq:upper:bound:choice} is thus solvable as it has a compact feasible set as well as a differentiable objective function. Moreover, the solution is unique thanks to the strict concavity of the objective function. In the following we denote this unique solution by~$\bs p^\star$.
    
    It remains to be shown that there exists a distribution $\theta^\star \in \Theta$ that attains the upper bound in~\eqref{eq:upper:bound:choice}.
    To this end, we define the functions $ q_i^\star(z_i) = \mathds{1}_{\{ z_i > F_i^{-1}(1 - p_i^\star) \}}$ for all $i \in [N]$.
    By \citep[Remark~4.48]{follmer2004stochastic}, $q_i^\star(z_i)$ is optimal in~\eqref{eq:dual:CVaR} for $p_i=p_i^\star$. In other words, we have $\EE_{z_i \sim \theta_i} [q_i^\star(z_i)] = p_i^\star$ and $\EE_{z_i \sim \theta_i}[z_i q_i^\star(z_i)] = \int_{1 - p_i^\star}^1 F_i^{-1}(t) \diff t$.
    In addition, we also define the Borel measures $\theta_i^+$ and $\theta_i^-$ through
    \begin{align*}
        \theta_i^+(B) = \theta_i(B | z_i > F_i^{-1}(1 - p_i^\star)) 
        \quad \text{and} \quad
        \theta_i^-(B) = \theta_i(B | z_i \leq F_i^{-1}(1 - p_i^\star)) 
    \end{align*}
    for all Borel sets $B \subseteq \R$, respectively. By construction, $\theta_i^+$ is supported on~$(F_i^{-1}(1 - p_i^\star), \infty)$, while $\theta_i^-$ is supported on~$(-\infty, (F_i^{-1}(1 - p_i^\star)]$. The law of total probability further implies that $\theta_i = p_i^\star \theta_i^+ + (1 - p_i^\star) \theta_i^-$. 
    In the remainder of the proof we will demonstrate that the maximization problem on the left hand side of~\eqref{eq:upper:bound:choice} is solved by the mixture distribution
    \begin{align*}
        \theta^\star = \sum_{j=1}^N p_j^\star \cdot \left( \otimes_{k=1}^{j-1} \theta_k^- \right) \otimes \theta_j^+ \otimes \left( \otimes_{k=j+1}^{N} \theta_k^- \right).
    \end{align*}
    This will show that the inequality in~\eqref{eq:upper:bound:choice} is in fact an equality, which in turn implies that the smooth $c$-transform is given by~\eqref{eq:regularized_c_transform}.
    We first prove that $\theta^\star \in \Theta$. To see this, note that for all $i \in [N]$ we have
    \begin{align*}
        \textstyle
        \theta^\star (z_i \leq s)
        = p_i^\star \theta_i^+ (z_i \leq s) + ( \sum_{j \neq i} p_j^\star ) \theta_i^- (z_i \leq s)
        = \theta_i (z_i \leq s)
        = F_i(s),
    \end{align*}
    where the second equality exploits the relation $\sum_{j \neq i} p_j^\star = 1 - p_i^\star$. This observation implies that $\theta^\star \in \Theta$. Next, we prove that $\theta^\star$ attains the upper bound in~\eqref{eq:upper:bound:choice}. By the definition of the binary function $\bs r$, we~have
    \begin{align*}
        \EE_{\bs z \sim \theta^\star} \Big[ \max\limits_{i \in [N]} u_i + z_{i} \Big] &=\EE_{\bs z \sim \theta^\star} \left[ ( u_i + z_i) r_i(\bs z) \right] \\
        &= \EE_{z_i \sim \theta_i} \left[ (u_i + z_i) \EE_{\bs z \sim \theta^\star} \left[r_i(\bs z) | z_i \right] \right] \\
	    &= \EE_{z_i \sim \theta_i} \Big[ ( u_i + z_i) \, \theta^\star \Big( i = \min \argmax\limits_{j \in [N]} ~ u_j + z_j \big| z_i \Big) \Big] \\
	    &= \EE_{z_i \sim \theta_i} \left[ ( u_i + z_i) \, \theta^\star \left( z_j < u_i + z_i - u_j~ \forall j \neq i \big| z_i  \right) \right],
    \end{align*}
    where the third equality holds because $r_i(\bs z)=1$ if and only if $i = \min \argmax_{j \in [N]} u_j + z_j$, and the fourth equality follows from the assumed continuity of the marginal distribution functions $F_i$, $i\in[N]$, which implies that $\theta^\star ( z_j = u_i + z_i - u_j~ \forall j \neq i \big| z_i ) = 0$ $\theta_i$-almost surely for all $i,j\in[N]$. 
    Hence, we find
    \begin{subequations}
    \label{eq:both:exp}
    \begin{align}
        \EE_{\bs z \sim \theta^\star} \Big[ \max\limits_{i \in [N]} u_i + z_{i} \Big]
        &= p_i^\star\, \EE_{z_i \sim \theta_i^+} \left[ ( u_i + z_i) \, \theta^\star \left( z_j < u_i + z_i - u_j~ \forall j \neq i \big| z_i \right) \right] \notag \\
        &\quad + (1 - p_i^\star)\, \EE_{z_i \sim \theta_i^-} \left[ ( u_i + z_i) \, \theta^\star \left( z_j < u_i + z_i - u_j~ \forall j \neq i \big| z_i \right) \right] \notag \\
        &= \DS p_i^\star\, \EE_{z_i \sim \theta_i^+} \Big[ (u_i + z_i) \Big( \prod_{j \neq i} \theta_j^-(z_j < z_i + u_i - u_j) \Big) \Big] \label{eq:first:exp}\\
        &\quad + \DS \sum_{j \neq i} p_j^\star \,\EE_{z_i \sim \theta_i^-} \Big[ (u_i + z_i) \Big( \!\prod_{k \neq i, j} \theta_k^-(z_k < z_i + u_i - u_k) \Big) \theta_j^+(z_j < z_i + u_i - u_j) \Big], \label{eq:second:exp}
    \end{align}
    \end{subequations}
    where the first equality exploits the relation $\theta_i = p_i^\star \theta_i^+ + (1 - p_i^\star) \theta_i^-$, while the second equality follows from the definition of $\theta^\star$. The expectations in~\eqref{eq:both:exp} can be further simplified by using the stationarity conditions of the upper bounding problem in~\eqref{eq:upper:bound:choice}, which imply that the partial derivatives of the objective function with respect to the decision variables $p_i$, $i\in[N]$, are all equal at $\bm p=\bm p^\star$. Thus, $\bs p^\star$ must satisfy
    \begin{align}
        \label{eq:KKT}
        u_i + F_i^{-1}(1 - p_i^\star) = u_j + F_j^{-1}(1 - p_j^\star) \quad \forall i, j \in [N].
    \end{align}
    Consequently, for every $z_i > F_i^{-1}(1 - p_i^\star)$ and $j\neq i$ we have
    \begin{align*}
        \theta_j^-(z_j < z_i + u_i - u_j) 
        \geq \theta_j^-(z_j \leq F_i^{-1}(1 - p_i^\star) + u_i - u_j)
        = \theta_j^-(z_j \leq F_j^{-1}(1 - p_j^\star)) = 1,
    \end{align*}
    where the first equality follows from~\eqref{eq:KKT}, and the second equality holds because $\theta_j^-$ is supported on $(-\infty, F_j^{-1}(1 - p_j^\star)]$. As no probability can exceed~1, the above reasoning implies that $\theta_j^-(z_j < z_i + u_i - u_j)=1$ for all $z_i > F_i^{-1}(1 - p_i^\star)$ and $j\neq i$. Noting that $q_i^\star(z_i)= \mathds{1}_{\{ z_i > F_i^{-1}(1 - p_i^\star) \}}$ represents the characteristic function of the set $(F_i^{-1}(1 - p_i^\star), \infty)$ covering the support of $\theta_i^+$, the term~\eqref{eq:first:exp} can thus be simplified to
    \begin{align}
        & p_i^\star \,\EE_{z_i \sim \theta_i^+} \Big[ (u_i + z_i) \Big( \prod_{j \neq i} \theta_j^-(z_j < z_i + u_i - u_j) \Big) q_i^\star(z_i) \Big] \notag = \EE_{z_i \sim \theta_i} \left[ (u_i + z_i) q_i^\star(z_i) \right]. \label{eq:first:term}
    \end{align}
    Similarly, for any $z_i \leq F_i^{-1}(1 - p_i^\star)$ and $j\neq i$ we have
    \begin{align*}
        \theta_j^+(z_j < z_i + u_i - u_j)
        \leq \theta_j^+(z_j < F_i^{-1}(1 - p_i^\star) + u_i - u_j)
        = \theta_j^+(z_j < F_j^{-1}(1 - p_j^\star)) = 0, 
    \end{align*}
    where the two equalities follow from~\eqref{eq:KKT} and the observation that $\theta_j^+$ is supported on $(F_j^{-1}(1 - p_j^\star), \infty)$, respectively. As probabilities are non-negative, the above implies that $\theta_j^+(z_j < z_i + u_i - u_j)=0$ for all $z_i \leq F_i^{-1}(1 - p_i^\star)$ and $j\neq i$. Hence, as $\theta_i^-$ is supported on $(-\infty, F_i^{-1}(1 - p_i^\star)]$, the term~\eqref{eq:second:exp} simplifies to
    \begin{align*}
        \sum_{j \neq i} p_j^\star \EE_{z_i \sim \theta_i^-} \Big[ (u_i + z_i) \Big( \prod_{k \neq i, j} \theta_k^-(z_k < z_i + u_i - u_k) \Big) \theta_j^+(z_j < z_i + u_i - u_j) \mathds{1}_{\{ z_i \leq F_i^{-1}(1 - p_i^\star) \}} \Big] = 0.
    \end{align*}
    By combining the simplified reformulations of~\eqref{eq:first:exp} and~\eqref{eq:second:exp}, we finally obtain
    \begin{align*}
        \EE_{\bs z \sim \theta^\star} \Big[ \max\limits_{i \in [N]} u_i + z_{i} \Big] 
	    = \sum_{i=1}^N \EE_{z_i \sim \theta_i} \left[ ( u_i + z_i) q_i^\star(z_i) \right] = \sum_{i=1}^N u_i p_i^\star + \sum_{i=1}^N \int_{1-p_i^\star}^1 F_i^{-1}(t) \diff t,
    \end{align*}
    where the last equality exploits the relations $\EE_{z_i \sim \theta_i} [q_i^\star(z_i)] = p_i^\star$ and $\EE_{z_i \sim \theta_i}[z_i q_i^\star(z_i)] = \int_{1 - p_i^\star}^1 F_i^{-1}(t) \diff t$ derived in the first part of the proof. We have thus shown that the smooth $c$-transform is given by~\eqref{eq:regularized_c_transform}.

    Finally, by the envelope theorem~\citep[Theorem~2.16]{de2000mathematical}, the gradient of $\nabla_{\bs \phi}\overline \psi(\bs \phi, \bs x)$ exists and coincides with the unique maximizer $\bm p^\star$ of the upper bounding problem in~\eqref{eq:regularized_c_transform}.
    \end{proof}
	
	The next theorem reveals that the smooth dual optimal transport problem~\eqref{eq:smooth_ot} with a marginal ambiguity set corresponds to a regularized primal optimal transport problem of the form~\eqref{eq:reg_ot_pri_abstract}.
	
	\begin{theorem}[Fr\'echet regularization]
		\label{theorem:primal_dual}
		Suppose that $\Theta$ is a marginal ambiguity set of the form~\eqref{eq:marginal_ambiguity_set} and that the marginal cumulative distribution functions are defined through
		\begin{equation}
	        \label{eq:marginal_dists}
		    F_i(s) = \min\{1, \max\{0, 1-\eta_i F(-s)\}\}
		\end{equation}
		for some probability vector $\bs \eta \in \Delta^N$ and strictly increasing function $F: \R \to \R$ with $\int_0^1 F^{-1} (t) \diff t = 0$. Then, the smooth dual optimal transport problem~\eqref{eq:smooth_ot} is equivalent to the regularized primal optimal transport problem~\eqref{eq:reg_ot_pri_abstract} with $R_\Theta = D_f(\pi \| \mu \otimes \eta)$, where
	    $f(s) = \int_{0 }^{s} F^{-1}(t) \diff t$ and $\eta = \sum_{i=1}^N \eta_i \delta_{y_i}$.
	\end{theorem}
    
    The function~$f(s)$ introduced in Theorem~\ref{theorem:primal_dual} is smooth and convex because its derivative $ \diff f(s) / \diff s = F^{-1}(s)$ is strictly increasing, and 
    $f(1) = \int_0^1 F^{-1}(t) \diff t=0$ by assumption. Therefore, this function induces a standard $f$-divergence. From now on we will refer to $F$ as the {\em marginal generating function}.

	\begin{proof} [Proof of Theorem~\ref{theorem:primal_dual}]
		By Proposition~\ref{proposition:regularized_ctrans}, the smooth dual optimal transport problem~\eqref{eq:smooth_ot} is equivalent~to
		\begin{align*}
		    \overline{W}_{c}(\mu, \nu)  &= \sup\limits_{ \bs {\phi}\in \R^N} ~ \mathbb{E}_{\bs x \sim \mu}\left[\min\limits_{\bs p\in \Delta^N} \sum\limits_{i=1}^N{\phi_i\nu_i}- \sum\limits_{i=1}^N(\phi_i - c(\bs x, \bs {y_i}))p_i - \sum_{i=1}^N \displaystyle\int_{1-p_i}^1 F_i^{-1}(t)\diff t 
		\right].
		\end{align*}
		As $F$ is strictly increasing, we have $F_i^{-1}(s) = -F^{-1}((1-s) / \eta_i)$ for all $s \in (0, 1)$. Thus, we find
		\begin{align}
		\label{eq:integral_rep_f}
    		f(s) = \int_{0}^{s} F^{-1}(t) \diff t = -\frac{1}{\eta_i} \int_{1}^{1 - s \eta_i} F^{-1} \left( \frac{1 - z}{\eta_i} \right) \diff z= -\frac{1}{ \eta_i} \int_{1 - s \eta_i}^1 F_i^{-1}(z) \diff z,
		\end{align}
		where the second equality follows from the variable substitution $z\leftarrow 1-\eta_i t$. This integral representation of~$f(s)$ then allows us to reformulate the smooth dual optimal transport problem as
		\begin{align*}
		    \overline{W}_{c}(\mu, \nu)= \sup\limits_{ \bs {\phi}\in \R^N} ~ \mathbb{E}_{\bs x \sim \mu}\left[\min\limits_{\bs p\in \Delta^N} \sum\limits_{i=1}^N{\phi_i\nu_i}- \sum\limits_{i=1}^N(\phi_i - c(\bs x, \bs {y_i}))p_i + \sum\limits_{i=1}^N \eta_i \,f\left( \frac{p_i}{\eta_i} \right) \right],
		\end{align*}
		which is manifestly equivalent to problem~\eqref{eq:dual_regularized_ot} thanks to the definition of the discrete $f$-divergence. Lemma~\ref{lem:strong_dual_reg_ot} finally implies that the resulting instance of~\eqref{eq:dual_regularized_ot} is equivalent to the regularized primal optimal transport problem~\eqref{eq:reg_ot_pri_abstract} with regularization term $R_\Theta (\pi ) =  D_{f}(\pi\|\mu \otimes \eta)$.
		Hence, the claim follows.
	\end{proof}
	
	Theorem~\ref{theorem:primal_dual} imposes relatively restrictive conditions on the marginals of~$\bs z$. Indeed, it requires that all marginal distribution functions $F_i$, $i\in[N]$, must be generated by a single marginal generating function~$F$ through the relation~\eqref{eq:marginal_dists}. The following examples showcase, however, that the freedom to select~$F$ offers significant flexibility in designing various (existing as well as new) regularization schemes. Details of the underlying derivations are relegated to Appendix~\ref{appendix:derivations}. {\color{black} Table~\ref{tab:summary_examples} summarizes the marginal generating functions~$F$ studied in these examples and lists the corresponding divergence generators~$f$.} 
	
	\begin{table}[h!]
\scriptsize
    \centering
    {\color{black}
    \begin{tabular}{||l|c|c|c||}
    \hline \hline
         Marginal Distribution &  $F(s)$ & $f(s)$ & Regularization \\ \hline\hline
         Exponential  & $ \exp(s/\lambda - 1)$ & $\lambda s \log(s)$ & Entropic \\ \hline
         Uniform & $s/(2\lambda) + 1/2$ &$\lambda (s^2 - s)$ &$\chi^2$-divergence\\\hline
         Pareto & $(s(q-1)/(\lambda q) + 1/q)^{\frac{1}{q-1}}$ & $\lambda (s^q - s)  / (q-1)$ & Tsallis divergence\\ \hline
         Hyperbolic cosine & ${\sinh(s/\lambda \!-\! k),~k = \sqrt{2} \!-\! 1\!-\! \textrm{arcsinh}(1)} $ & $\lambda(s\, \text{arcsinh}(s) \!-\! \sqrt{s^2 \!+\!1} \!+\! 1\! +\! ks)$ & Hyperbolic divergence\\ \hline
         $t$-distribution &$\frac{N}{2}\left(1 + \frac{s-\sqrt{N\!-\!1}}{\sqrt{\lambda^2 + (s \!-\! \sqrt{N\!-\!1})^2}}\right)$ & \!\!\!\!$\begin{cases}-\lambda \sqrt{s(N\!-\!s)} \! +\! \lambda s \sqrt{N\!-\!1} ~&\text{if}~0\leq \!s\! \leq N\\
         +\infty &\text{if} ~s\!>\!N\end{cases}$& Chebychev\\ 
         \hline\hline
    \end{tabular}
    \caption{\color{black} Marginal generating functions~$F$ with parameter~$\lambda$ and corresponding divergence generators~$f$.}
    \label{tab:summary_examples}}
\end{table}
	
	\begin{example}[Exponential distribution model]
		\label{ex:exp}
		Suppose that $\Theta$ is a marginal ambiguity set with (shifted) exponential marginals of the form~\eqref{eq:marginal_dists} induced by the generating function
		$F(s) = \exp(s / \lambda - 1)$ with $\lambda > 0$.
		Then the smooth dual optimal transport problem~\eqref{eq:smooth_ot} is equivalent to the regularized optimal transport problem~\eqref{eq:reg_ot_pri_abstract} with an entropic regularizer of the form $R_\Theta(\pi) = D_f(\pi \| \mu \otimes \eta)$, where $f(s) =\lambda s \log(s)$, while the smooth $c$-transform~\eqref{eq:smooth_c_transform} reduces to the log-partition function~\eqref{eq:partition:function}. This example shows that entropic regularizers are not only induced by singleton ambiguity sets containing a generalized extreme value distribution (see Section~\ref{sec:gevm}) but also by marginal ambiguity sets with exponential marginals.
	\end{example}

	\begin{example}[Uniform distribution model]
		\label{ex:uniform}
		Suppose that $\Theta$ is a marginal ambiguity set with uniform marginals  of the form~\eqref{eq:marginal_dists} induced by the generating function $F(s) = s/(2\lambda) + 1/2$ with $\lambda > 0$.
		In this case the smooth dual optimal transport problem~\eqref{eq:smooth_ot} is equivalent to the regularized optimal transport problem~\eqref{eq:reg_ot_pri_abstract} with a $\chi^2$-divergence regularizer of the form $R_\Theta(\pi) = D_f(\pi \| \mu \otimes \eta)$, where $f(s) = \lambda (s^2 -s)$. Such regularizers were previously investigated by \citet{blondel2017smooth} and~\citet{seguy2017large} under the additional assumption that $\eta_i$ is independent of $i\in[N]$, yet their intimate relation to noise models with uniform marginals remained undiscovered until now. In addition, the smooth $c$-transform~\eqref{eq:smooth_c_transform} satisfies
		\begin{align*}
		    \overline\psi(\bs \phi, \bs x) = \lambda + \lambda \spmax_{i \in [N]} \;\frac{\phi_i - c(\bs x, \bs {y_i})}{\lambda},
		\end{align*}
		where the sparse maximum operator `$\spmax$'  inspired by \citet{sparsemax} is defined through
		\begin{align}
		    \label{eq:spmax}
		    \spmax_{i \in [N]} \; u_i = \max_{\bs p \in \Delta^N} \; \sum_{i=1}^N u_i p_i - {p_i^2}/{\eta_i} \qquad \forall \bm u\in\R^N.
		\end{align}
		The envelope theorem~\citep[Theorem~2.16]{de2000mathematical} ensures that $\spmax_{i \in[N]} u_i$ is smooth and that its gradient with respect to~$\bm u$ is given by the unique solution~$\bm p^\star$ of the maximization problem on the right hand side of~\eqref{eq:spmax}. We note that $\bm p^\star$ has many zero entries due to the sparsity-inducing nature of the problem's simplicial feasible set. In addition, we have $\lim_{\lambda\downarrow 0} \lambda \spmax_{i \in [N]} u_i/\lambda = \max_{i\in[N]}u_i$. Thus, the sparse maximum can indeed be viewed as a smooth approximation of the ordinary maximum. In marked contrast to the more widely used LogSumExp function, however, the sparse maximum has a sparse gradient. Proposition~\ref{proposition:spmax} in Appendix~\ref{appendix:spmax} shows that $\bm p^\star$ can be computed efficiently by sorting.
	\end{example}
	
	\begin{example}[Pareto distribution model]
		\label{ex:pareto}
		Suppose that $\Theta$ is a marginal ambiguity set with (shifted) Pareto distributed marginals of the form~\eqref{eq:marginal_dists} induced by the generating function
		$F(s) = (s (q-1) / (\lambda q)+1/q)^{1/(q-1)}$ with $\lambda,q>0$.
		Then the smooth dual optimal transport problem~\eqref{eq:smooth_ot} is equivalent to the regularized optimal transport problem~\eqref{eq:reg_ot_pri_abstract} with a Tsallis divergence regularizer of the form $R_\Theta(\pi) = D_f(\pi \| \mu \otimes \eta)$, where $f(s) = \lambda (s^q - s)/(q-1)$. Such regularizers were  investigated by~\citep{muzellec2017tsallis} under the additional assumption that $\eta_i$ is independent of $i\in[N]$. The Pareto distribution model encapsulates the exponential model (in the limit $q\to 1$) and the uniform distribution model (for $q=2$) as special cases. The smooth $c$-transform admits no simple closed-form representation under this model.
	\end{example}
	
	\begin{example}[Hyperbolic cosine distribution model]
		\label{ex:hyperbolic}
	Suppose that $\Theta$ is a marginal ambiguity set with hyperbolic cosine distributed marginals of the form~\eqref{eq:marginal_dists} induced by the generating function $F(s) = \sinh(s/\lambda - k)$ with $k = \sqrt{2} - 1 - \asinh(1)$ and $\lambda > 0$. Then the marginal probability density functions are given by scaled and truncated hyperbolic cosine functions, and the smooth dual optimal transport problem~\eqref{eq:smooth_ot} is equivalent to the regularized optimal transport problem~\eqref{eq:reg_ot_pri_abstract} with a hyperbolic divergence regularizer of the form $R_\Theta(\pi) = D_f(\pi \| \mu \otimes \eta)$, where $f(s) = \lambda(s \hspace{0.1em} \asinh(s) - \sqrt{s^2 + 1} + 1 + ks)$. Hyperbolic divergences were introduced by \citet{ghai2019exponentiated} in order to unify several gradient descent algorithms.
	\end{example} 
	
	\begin{example}[$t$-distribution model]
	    \label{ex:t-distribution}
		Suppose that $\Theta$ is a marginal ambiguity set where the marginals are determined by~\eqref{eq:marginal_dists}, and assume that the generating function is given by
	    \[
	        F(s) = \frac{N}{2}\left(1 + \frac{s - \sqrt{N-1}} {\sqrt{\lambda^2 + (s - \sqrt{N-1})^{2}}}\right)
	    \]
	    for some $\lambda > 0$. In this case one can show that all marginals constitute $t$-distributions with $2$ degrees of freedom. In addition, one can show that the smooth dual optimal transport problem~\eqref{eq:smooth_ot} is equivalent to the Chebyshev regularized optimal transport problem described in Proposition~\ref{prop:chebyshev-regularization}.
	\end{example}
	
	To close this section, we remark that different regularization schemes differ as to how well they approximate the original (unregularized) optimal transport problem. Proposition~\ref{prop:approx_bound} provides simple error bounds that may help in selecting suitable regularizers. For the entropic regularization scheme associated with the exponential distribution model of Example~\ref{ex:exp}, for example, the error bound evaluates to
    $\max_{i\in [N]}\lambda \log(1/\eta_i)$, while for the $\chi^2$-divergence regularization scheme associated with the uniform distribution model of Example~\ref{ex:uniform}, the error bound is given by $\max_{i \in [N]}\lambda (1/\eta_i - 1)$. In both cases, the error is minimized by setting~$\eta_i = 1/N $ for all $i \in [N]$. Thus, the error bound grows logarithmically with $N$ for entropic regularization and linearly with $N$ for $\chi^2$-divergence regularization. Different regularization schemes also differ with regard to their computational properties, which will be discussed in Section~\ref{sec:computation}.
    
    \section{Numerical Solution of Smooth Optimal Transport Problems}
    \label{sec:computation}
	The smooth semi-discrete optimal transport problem~\eqref{eq:smooth_ot} constitutes a stochastic optimization problem and can therefore be addressed with a stochastic gradient descent (SGD) algorithm. In Section~\ref{section:AGD} we first derive new convergence guarantees for an averaged gradient descent algorithm that has only access to a biased stochastic gradient oracle. This algorithm outputs the uniform average of the iterates (instead of the last iterate) as the recommended candidate solution. We prove that if the objective function is Lipschitz continuous, then the suboptimality of this candidate solution is of the order~$\mathcal O(1/\sqrt{T})$, where $T$ stands for the number of iterations. An improvement in the non-leading terms is possible if the objective function is additionally smooth. We further prove that a convergence rate of $\mathcal O(1/{T})$ can be obtained for generalized self-concordant objective functions. In Section~\ref{section:ASGD-OT} we then show that the algorithm of Section~\ref{section:AGD} can be used to efficiently solve the smooth semi-discrete optimal transport problem~\eqref{eq:smooth_ot} corresponding to a marginal ambiguity set of the type~\eqref{eq:marginal_ambiguity_set}. As a byproduct, we prove that the convergence rate of the averaged SGD algorithm for the semi-discrete optimal transport problem with {\em entropic} regularization is of the order~$\mathcal O(1/T)$, which improves the $\mathcal O(1/\sqrt{T})$ guarantee of~\citet{genevay2016stochastic}.

	\subsection{Averaged Gradient Descent Algorithm with Biased Gradient Oracles}
	\label{section:AGD}	
	Consider a general convex minimization problem of the form
	\begin{equation}
	\label{eq:convex:problem}
	\min_{\bm \phi \in \R^n} ~ h(\bm \phi),
	\end{equation}
	where the objective function $h: \R^n \to \R$ is convex and differentiable. We assume that problem~\eqref{eq:convex:problem} admits a minimizer $\bm \phi^\star$. We study the convergence behavior of the inexact gradient descent algorithm
	\begin{equation} \label{eq:gd}
	\bm \phi_{t} = \bm \phi_{t-1} - \gamma \bm g_t(\bm \phi_{t-1}),
	\end{equation}
	where $\gamma > 0$ is a fixed step size, $\bm \phi_0$ is a given deterministic initial point and the function $\bm g_t: \R^n \to \R^n$ is an inexact gradient oracle that returns for every fixed $\bm \phi\in\R^n$ a random estimate of the gradient of~$h$ at~$\bm \phi$. Note that we allow the gradient oracle to depend on the iteration counter~$t$, which allows us to account for increasingly accurate gradient estimates. In contrast to the previous sections, we henceforth model all random objects as measurable functions on an abstract filtered probability space $(\Omega, \mathcal F, (\mathcal F_t)_{t \geq 0}, \mathbb P)$, where $\mathcal{F}_0 = \{ \emptyset,\Omega \}$ represents the trivial $\sigma$-field, while the gradient oracle $\bm g_t(\bm \phi)$ is $\mathcal F_t$-measurable for all $t\in\mathbb N$ and $\bm \phi \in\R^n$. In order to avoid clutter, we use $\mathbb E[\cdot]$ to denote the expectation operator with respect to~$\mathbb P$, and all inequalities and equalities involving random variables are understood to hold $\mathbb P$-almost surely.
	
	In the following we analyze the effect of averaging in inexact gradient descent algorithms. We will show that after $T$ iterations with a constant step size~$\gamma = \mathcal O(1 / \sqrt{T})$, the objective function value of the uniform average of all iterates generated by~\eqref{eq:gd} converges to the optimal value of~\eqref{eq:convex:problem} at a sublinear rate. Specifically, we will prove that the rate of convergence varies between $\mathcal{O}(1 / \sqrt{T})$ and $\mathcal{O}(1/T)$ depending on properties of the objective function.
	Our convergence analysis will rely on several regularity conditions.
	\begin{assumption}[Regularity conditions] Different combinations of the following regularity conditions will enable us to establish different convergence guarantees for the averaged inexact gradient descent algorithm.
		\label{assumption:main}~
		\begin{enumerate}[label=(\roman*)]
			\item
			\textbf{Biased gradient oracle:}
			\label{assumption:main:gradients}
			There exists tolerances $\varepsilon_t>0$, $t\in\mathbb N\cup\{0\}$, such that
			\begin{align*} 
			\left\| \EE \left[ \bm g_t(\bm \phi_{t-1}) \big| \mathcal F_{t-1} \right] - \nabla h(\bm \phi_t) \right\|  \leq \varepsilon_{t-1}\quad \forall t\in\mathbb N.
			\end{align*}
			\item
			\textbf{Bounded gradients:}
			\label{assumption:main:bounded}
			There exists $R > 0$ such that
			$$ \| \nabla h(\bm \phi) \| \leq R\quad \text{and} \quad \| \bm g_t(\bm \phi) \| \leq R \quad \forall \bm \phi \in \R^n,~ \forall t \in \mathbb N.  $$
			\item
			\textbf{Generalized self-concordance:}
			\label{assumption:main:concordance}
			The function~$h$ is $M$-generalized self-concordant for some $M > 0$, that is, $h$ is three times differentiable, and for any $\bm \phi, \bm \phi' \in \R^n$ the function $u(s) = h(\bm \phi + s (\bm \phi' - \bm \phi))$ satisfies the inequality
			$$ \left| \frac{\diff^3 u(s)}{\diff s^3} \right| \leq M \| \bm \phi - \bm \phi' \| \, \frac{\diff^2 u(s)}{\diff s^2} \quad \forall s \in \R.$$
			\item
			\textbf{Lipschitz continuous gradient:}
			\label{assumption:main:smooth}
			The function $h$ is $L$-smooth for some $L > 0$, that is, we have
			$$ \| \nabla h(\bm \phi) - \nabla h(\bm \phi') \| \leq L \| \bm \phi - \bm \phi' \| \quad \forall \bm \phi, \bm \phi' \in \R^n. $$
			\item
			\textbf{Bounded second moments:}
			\label{assumption:main:moment}
			There exists $\sigma > 0$ such that
			\begin{align*}
			\EE \left[\left\| \bm g_t(\bm \phi_{t-1}) - \nabla h(\bm \phi_{t-1}) \right\|^2 | \mathcal F_{t-1} \right]  \leq \sigma^2 \quad \forall t \in \mathbb N.
			\end{align*}
		\end{enumerate}
	\end{assumption}
	The averaged gradient descent algorithm with biased gradient oracles lends itself to solving both deterministic as well as stochastic optimization problems. In deterministic optimization, the gradient oracles $\bm g_t$ are deterministic and output inexact gradients satisfying $\| \bm g_t(\bm \phi) - \nabla h(\bm \phi) \| \leq \varepsilon_t$ for all $\bm \phi\in \R^n$, where the tolerances $\varepsilon_t$ bound the errors associated with the numerical computation of the gradients. A vast body of literature on deterministic optimization focuses on exact gradient oracles for which these tolerances can be set to~$0$. Inexact deterministic gradient oracles with bounded error tolerances are investigated by \citet{nedic2001convergence} and \citet{d2008smooth}. In this case exact convergence to $\bm \phi^\star$ is not possible. If the error bounds decrease to~$0$, however, \citet{luo1993error, schmidt2011convergence} and \citet{friedlander2012hybrid} show that adaptive gradient descent algorithms are guaranteed to converge to $\bm \phi^\star$.
	
	In stochastic optimization, the objective function is representable as $h(\bm \phi) = \EE [H(\bm \phi, \bm x)]$, where the marginal distribution of the random vector $\bm x$ under $\mathbb P$ is given by~$\mu$, while the integrand $H(\bm \phi,\bm x)$ is convex and differentiable in~$\bm \phi$ and $\mu$-integrable in~$\bm x$. In this setting it is convenient to use gradient oracles of the form $\bm g_t(\bm \phi) = \nabla_{\bm \phi} H(\bm \phi, \bm x_t)$ for all $t \in \mathbb N$, where the samples $\bm x_t$ are drawn independently from~$\mu$. As these oracles output unbiased estimates for $\nabla h(\bm \phi)$, all tolerances $\varepsilon_t$ in Assumptions~\ref{assumption:main}\,(i) may be set to~$0$. SGD algorithms with unbiased gradient oracles date back to the seminal paper by~\citet{robbins1951stochastic}. Nowadays, averaged SGD algorithms with Polyak-Ruppert averaging figure among the most popular variants of the SGD algorithm~\citep{ruppert1988efficient, polyak1992acceleration, nemirovski2009robust}. For general convex objective functions the best possible convergence rate of any averaged SGD algorithm run over~$T$ iterations amounts to $\mathcal{O}(1 / \sqrt{T})$, but it improves to $\mathcal{O}(1 / T)$ if the objective function is strongly convex; see for example \citep{nesterov2008confidence, nemirovski2009robust, shalev2009stochastic, duchi2009efficient, xiao2010dual, moulines2011non, shalev2011pegasos, lacoste2012simpler}.
	While smoothness plays a critical role to achieve acceleration in deterministic optimization, it only improves the constants in the convergence rate in stochastic optimization \citep{srebro2010optimistic, dekel2012optimal, lan2012optimal, cohen2018acceleration, kavis2019unixgrad}.
	In fact, \citet{tsybakov2003optimal} demonstrates that smoothness does not provide any acceleration in general, that is, the best possible convergence rate of any averaged SGD algorithm can still not be improved beyond $\mathcal{O}(1 / \sqrt{T})$. Nevertheless, a substantial acceleration is possible when focusing on special problem classes such as linear or logistic regression problems \citep{bach2013adaptivity, bach2013non, hazan2014logistic}. In these special cases, the improvement in the convergence rate is facilitated by a generalized self-concordance property of the objective function~\citep{bach2010self}.
	Self-concordance was originally introduced in the context of Newton-type interior point methods \citep{nesterov1994interior} and later generalized to facilitate the analysis of probabilistic models \citep{bach2010self} and second-order optimization algorithms \citep{sun2019generalized}.
	
	In the following we analyze the convergence properties of the averaged SGD algorithm 
	when we have only access to an {\em inexact} stochastic gradient oracle, in which case the tolerances $\varepsilon_t$ cannot be set to~$0$. To our best knowledge, inexact stochastic gradient oracles have only been considered by~\citet{cohen2018acceleration, hu2020analysis} and \citet{ajalloeian2020analysis}. Specifically, \citet{hu2020analysis} use sequential semidefinite programs to analyze the convergence rate of the averaged SGD algorithm when~$\mu$ has a finite support. In contrast, we do not impose any restrictions on the support of~$\mu$. \citet{cohen2018acceleration} and \citet{ajalloeian2020analysis}, on the other hand, study the convergence behavior of accelerated gradient descent algorithms for smooth stochastic optimization problems under the assumption that~$\bm \phi$ ranges over a compact domain. The proposed algorithms necessitate a projection onto the compact feasible set in each iteration. In contrast, our convergence analysis does not rely on any compactness assumptions. We note that compactness assumptions have been critical for the convergence analysis of the averaged SGD algorithm in the context of convex stochastic optimization \citep{nemirovski2009robust, dekel2012optimal, bubeck2015convex, cohen2018acceleration}. By leveraging a trick due to \citet{bach2013adaptivity}, however, we can relax this assumption provided that the objective function is Lipschitz continuous.
	
	\begin{proposition}
		\label{proposition:moments}
		Consider the inexact gradient descent algorithm~\eqref{eq:gd} with constant step size $\gamma > 0$. If Assumptions~\ref{assumption:main}\,\ref{assumption:main:gradients}--\ref{assumption:main:bounded} hold with $\varepsilon_t \leq {\bar \varepsilon}/{(2\sqrt{1+t})}$ for some $\bar \varepsilon \geq 0$, then we have for all $ p \in \mathbb N$ that
		\begin{align*}
		\EE \left[ \left( h \left(\frac{1}{T} \sum_{t=1}^{T} \bm \phi_{t-1} \right) - h(\bm \phi^\star) \right)^p \right]^{1/p}
		\leq \frac{\| \bm \phi_0 - \bm \phi^\star \|^2}{\gamma T} + 20 \gamma \left( R + \bar \varepsilon \right)^2 p.
		\end{align*}
		If additionally Assumption~\ref{assumption:main}\,\ref{assumption:main:concordance} holds and if $G = \max\{ M, R + \bar \varepsilon \}$, then we have for all $ p \in \mathbb N$ that
		\begin{align*}
		\EE \left[ \left\| \nabla h \left(\frac{1}{T} \sum_{t=1}^{T} \bm \phi_{t-1} \right) \right\|^{2p} \right]^{1/p}
		&\leq \frac{G^{2}}{T} \left( 10 \sqrt{p} + \frac{4p}{\sqrt{T}} + 80 G^2 \gamma \sqrt{T} p + \frac{2 \| \bm \phi_0 - \bm \phi^\star \|^2}{\gamma \sqrt{T}} + \frac{3 \| \bm \phi_0 - \bm \phi^\star \|}{G \gamma \sqrt{T}} \right)^2.
		\end{align*}
	\end{proposition}
	
	The proof of Proposition~\ref{proposition:moments} relies on two lemmas. In order to state these lemmas concisely, we define the $L_p$-norm, of a random variable $\bm z \in \R^n$ for any $p > 0$ through $\| \bm z \|_{L_p} = \left( \EE \left[ \| \bm z \|^p \right] \right)^{1/p}$. For any random variables $\bm z, \bm z' \in \R^n$ and $p \geq 1$, Minkowski's inequality~\citep[\S~2.11]{boucheron2013concentration} then states that
	\begin{equation}
	    \label{eq:minkowski}
	    \| \bm z + \bm z' \|_{L_p} \leq \| \bm z \|_{L_p} + \| \bm z' \|_{L_p}.
	\end{equation}
	Another essential tool for proving Proposition~\ref{proposition:moments} is the Burkholder-Rosenthal-Pinelis (BRP) inequality~\citep[Theorem~4.1]{pinelis1994optimum}, which we restate below without proof to keep this paper self-contained.
	
	\begin{lemma}[BRP inequality]
		\label{lemma:BRP}
		Let $\bm z_t$ be an $\mathcal F_t$-measurable random variable for every $t\in\mathbb N$ and assume that $p \geq 2$. For any $t \in [T]$ with $\EE[\bm z_t | \mathcal F_{t-1}] = 0 $ and $\| \bm z_t \|_{L_p}<\infty$ we then have
		\begin{align*}
		    \left\| \max_{t \in [T]} \left\| \sum_{k=1}^t \bm z_k \right\| \right\|_{L_p}
		    \leq \sqrt{p} \left\| \sum_{t=1}^T \EE[ \| \bm z_t \|^2 | \mathcal F_{t-1}] \right\|_{L_{p/2}}^{1/2} + p \left\| \max_{t \in [T]} \| \bm z_t \| \right\|_{L_p}. 
		\end{align*}
	\end{lemma}
	The following lemma reviews two useful properties of generalized self-concordant functions.
	\begin{lemma} \label{lemma:concordance} [Generalized self-concordance]
		Assume that the objective function
		$h$ of the convex optimization problem~\eqref{eq:convex:problem} is $M$-generalized self-concordant in the sense of Assumption~\ref{assumption:main}\,\ref{assumption:main:concordance} for some $M>0$.
		\begin{enumerate} [label=(\roman*)]
			\item
			\label{lemma:smoothness}
			{\citep[Appendix~D.2]{bach2013adaptivity}} For any sequence $\bm \phi_0, \dots, \bm \phi_{T-1} \in \R^n$, we have
			\begin{align*}
			\left\| \nabla h \left( \frac{1}{T} \sum_{t=1}^T \bm \phi_{t-1} \right) - \frac{1}{T} \sum_{t=1}^T \nabla h(\bm \phi_{t-1}) \right\| \leq 2 M \left( \frac{1}{T} \sum_{t=1}^T h(\bm \phi_{t-1}) - h(\bm \phi^\star) \right).
			\end{align*}
			
			\item
			\label{lemma:stong:convexity}
			{\citep[Lemma~9]{bach2013adaptivity}}
			For any $\bm \phi \in \R^n$ with $ \| \nabla h(\bm \phi) \| \leq 3 \kappa / (4 M) $, where $\kappa$ is the smallest eigenvalue of $\nabla^2 h(\bm \phi^\star)$, and $\bm \phi^\star$ is the optimizer of~\eqref{eq:convex:problem}, we have
			$ h(\bm \phi) - h(\bm \phi^\star) \leq 2 {\| \nabla h(\bm \phi) \|^2}/{\kappa}.$
		\end{enumerate}
	\end{lemma}
	
	Armed with Lemmas~\ref{lemma:BRP} and~\ref{lemma:concordance}, we are now ready to prove Proposition~\ref{proposition:moments}.
	
	\begin{proof}[Proof of Proposition~\ref{proposition:moments}]
		The first claim generalizes Proposition~5 by~\citet{bach2013adaptivity} to inexact gradient oracles. 
		By the assumed convexity and differentiability of the objective function $h$, we have
		\begin{align} \label{eq:lip:update}
		h(\bm \phi_{k-1}) 
		&\leq h(\bm \phi_{\star}) + \nabla h(\bm \phi_{k-1})^\top (\bm \phi_{k-1} - \bm \phi_{\star}) \\
		&= h(\bm \phi_{\star}) + \bm g_k(\bm \phi_{k-1})^\top (\bm \phi_{k-1} - \bm \phi_{\star}) + \left( \nabla h(\bm \phi_{k-1}) - \bm g_k(\bm \phi_{k-1}) \right)^\top (\bm \phi_{k-1} - \bm \phi_{\star}). \notag
		\end{align}
		In addition, elementary algebra  yields the recursion
		\begin{equation*}
		\| \bm \phi_{k} - \bm \phi^\star \|^2 = \| \bm \phi_{k} - \bm \phi_{k-1} \|^2 + \| \bm \phi_{k-1} - \bm \phi^\star \|^2 + 2 (\bm \phi_{k} - \bm \phi_{k-1})^\top (\bm \phi_{k-1} - \bm \phi^\star).
		\end{equation*}
		Thanks to the update rule~\eqref{eq:gd}, this recursion can be re-expressed as
		\begin{equation*}
		\bm g_k(\bm \phi_{k-1})^\top (\bm \phi_{k-1} - \bm \phi^\star) = \frac{1}{2 \gamma} \left( \gamma^2 \| \bm g_k(\bm \phi_{k-1}) \|^2 + \| \bm \phi_{k-1} - \bm \phi^\star \|^2 - \| \bm \phi_{k} - \bm \phi^\star \|^2 \right),
		\end{equation*}
		where $\gamma > 0$ is an arbitrary step size. Combining the above identity with~\eqref{eq:lip:update} then yields
		\begin{align*}
		& ~h(\bm \phi_{k-1}) \\
		\leq & ~h(\bm \phi_{\star}) + \frac{1}{2 \gamma} \left( \gamma^2 \| \bm g_k(\bm \phi_{k-1}) \|^2 + \| \bm \phi_{k-1} - \bm \phi^\star \|^2 - \| \bm \phi_{k} - \bm \phi^\star \|^2 \right) + \left( \nabla h(\bm \phi_{k-1}) - \bm g_k(\bm \phi_{k-1}) \right)^\top \! (\bm \phi_{k-1} - \bm \phi_{\star}) \\
		\leq & ~h(\bm \phi_{\star}) + \frac{1}{2 \gamma} \left( \gamma^2 R^2 + \| \bm \phi_{k-1} - \bm \phi^\star \|^2 - \| \bm \phi_{k} - \bm \phi^\star \|^2 \right) + \left( \nabla h(\bm \phi_{k-1}) - \bm g_k(\bm \phi_{k-1}) \right)^\top (\bm \phi_{k-1} - \bm \phi_{\star}),
		\end{align*}
		where the last inequality follows from Assumption~\ref{assumption:main}\,\ref{assumption:main:bounded}.
		Summing this inequality over $k$ then shows that
		\begin{align}
		\label{eq:bound:A}
		    2 \gamma \sum_{k=1}^t \big( h ( \bm \phi_{k-1}) - h(\bm \phi_{\star}) \big) + \| \bm \phi_{t} - \bm \phi^\star \|^2 \leq A_t,
		\end{align}
		where
		\begin{align*}
		    A_t = t \gamma^2 R^2 + \| \bm \phi_{0} - \bm \phi^\star \|^2 + \sum_{k=1}^t B_k \quad \text{and} \quad B_t = 2 \gamma \left( \nabla h(\bm \phi_{t-1}) - \bm g_t(\bm \phi_{t-1}) \right)^\top (\bm \phi_{t-1} - \bm \phi_{\star})
		\end{align*}
		for all $t \in\mathbb N$. Note that the term on the left hand side of~\eqref{eq:bound:A} is non-negative because $\bm \phi^\star$ is a global minimizer of $h$, which implies that the random variable $A_t$ is also non-negative for all $t\in\mathbb N$. For later use we further define $A_0 = \| \bm \phi_{0} - \bm \phi^\star \|^2$. The estimate~\eqref{eq:bound:A} for $t=T$ then implies via the convexity of $h$ that
		\begin{align}
		\label{eq:bound:A:convexity}
		    h \left( \frac{1}{T} \sum_{t=1}^T \bm \phi_{t-1} \right) - h(\bm \phi_{\star})
		    \leq \frac{A_T}{2 \gamma T },
		\end{align}
		where we dropped the non-negative term $\|\bm \phi_T-\bm \phi^\star\|^2/(2\gamma T)$ without invalidating the inequality. In the following we analyze the $L_p$-norm of $A_T$ in order to obtain the desired bounds from the proposition statement. To do so, we distinguish three different regimes for $p \in \mathbb N$, and we show that the $L_p$-norm of the non-negative random variable $A_T$ is upper bounded by an affine function of~$p$ in each of these regimes.

		\textbf{Case I ($p \geq T / 4$):}
		By using the update rule~\eqref{eq:gd} and Assumption~\ref{assumption:main}\,\ref{assumption:main:bounded}, one readily verifies that
		\begin{align*}
		    \| \bm \phi_k - \bm \phi^\star \| 
		    \leq \| \bm \phi_{k-1} - \bm \phi^\star \| + \| \bm \phi_k - \bm \phi_{k-1} \| 
		    \leq \| \bm \phi_{k-1} - \bm \phi^\star \| + \gamma R.
		\end{align*}
		 Iterating the above recursion $k$ times then yields the conservative estimate $\| \bm \phi_k - \bm \phi^\star \|\leq \| \bm \phi_{0} - \bm \phi^\star \| + k \gamma R$. By definitions of $A_t$ and $B_t$ for $t\in\mathbb N$, we thus have
		\begin{align*}
		    A_t 
		    &\textstyle
		    = t \gamma^2 R^2 + \| \bm \phi_{0} - \bm \phi^\star \|^2 + 2 \gamma \sum_{k=1}^t \left( \nabla h(\bm \phi_{k-1}) - \bm g_k(\bm \phi_{k-1}) \right)^\top (\bm \phi_{k-1} - \bm \phi_{\star}) \\
		    &\textstyle
		    \leq t \gamma^2 R^2 + \| \bm \phi_{0} - \bm \phi^\star \|^2 + 4 \gamma R \sum_{k=1}^t \| \bm \phi_{k-1} - \bm \phi_{\star} \| \\
		    &\textstyle
		    \leq t \gamma^2 R^2 + \| \bm \phi_{0} - \bm \phi^\star \|^2 + 4 \gamma R \sum_{k=1}^t \left( \| \bm \phi_{0} - \bm \phi^\star \| + (k-1) \gamma R \right) \\
		    &\leq t \gamma^2 R^2 + \| \bm \phi_0 - \bm \phi^\star \|^2 + 4 t \gamma R \| \bm \phi_0 - \bm \phi^\star \| + 2 t^2 \gamma^2 R^2 \notag \\
		    &\leq t \gamma^2 R^2 + \| \bm \phi_0 - \bm \phi^\star \|^2 + 4 t^2 \gamma^2 R^2 + \| \bm \phi_0 - \bm \phi^\star \|^2 + 2 t^2 \gamma^2 R^2 \leq 7 t^2 \gamma^2 R^2 + 2 \| \bm \phi_0 - \bm \phi^\star \|^2,
		\end{align*}
		where the first two inequalities follow from Assumption~\ref{assumption:main}\,\ref{assumption:main:bounded} and the conservative estimate derived above, respectively, while the fourth inequality holds because $2 a b \leq a^2 + b^2$ for all $a,b\in\R$. As $A_t \geq 0$, the random variable $A_t$ is bounded and satisfies $| A_t| \leq 2 \| \bm \phi_0 - \bm \phi^\star \|^2 + 7 t^2 \gamma^2 R^2$ for all $t\in\mathbb N$, which implies that
		\begin{align}
		\label{eq:bound:A:T/4}
		\| A_T \|_{L_p} 
		\leq 2 \| \bm \phi_0 - \bm \phi^\star \|^2 + 7 T^2 \gamma^2 R^2
		&\leq 2 \| \bm \phi_0 - \bm \phi^\star \|^2 + 28 T \gamma^2 R^2 p,
		\end{align}
		where the last inequality holds because $p \geq T/4$. Note that the resulting upper bound is affine in~$p$.
		
		\textbf{Case II $({2 \leq p \leq T/4})$:} 
		The subsequent analysis relies on the simple bounds
		\begin{align}
		\label{eq:bound:varepsilon}
    		\textstyle 
    		\max_{t \in [T]} \varepsilon_{t-1} \leq \frac{\bar \varepsilon}{2} \quad \text{and} \quad \sum_{t=1}^T \varepsilon_{t-1} \leq \bar \varepsilon \sqrt{T},
		\end{align}
		which hold because $\varepsilon_t \leq \bar \varepsilon / (2 \sqrt{1+t})$ by assumption and because $\sum_{t=1}^T 1 / \sqrt{t} \leq 2 \sqrt{T}$, which can be proved by induction.
		In addition, it proves useful to introduce the martingale differences $ \bar B_t = B_t - \EE[B_t | \mathcal F_{t-1}]$ for all $t\in\mathbb N$. By the definition of $A_t$ and the subadditivity of the supremum operator, we then have
		\begin{align*}
		\max_{t \in [T+1]} A_{t-1}
		&= \max_{t \in [T+1]} \left\{ (t-1) \gamma^2 R^2 + \| \bm \phi_0 - \bm \phi^\star \|^2 + \sum_{k=1}^{t-1} \EE[B_k | \mathcal F_{k-1}] + \sum_{k=1}^{t-1} \bar B_k \right\} \\
		&\leq T \gamma^2 R^2 + \| \bm \phi_0 - \bm \phi^\star \|^2 + \max_{t \in [T]} \sum_{k=1}^t \EE[B_k | \mathcal F_{k-1}] + \max_{t \in [T]} \sum_{k=1}^t \bar B_k .
		\end{align*}
		As $p \geq 2$,  Minkowski's inequality~\eqref{eq:minkowski} thus implies that
		\begin{align}
		\label{eq:bound:sup:A}
		\left\| \max_{t \in [T+1]} A_{t-1} \right\|_{L_p}
		&\leq T \gamma^2 R^2 + \| \bm \phi_0 - \bm \phi^\star \|^2 + \left\| \max_{t \in [T]} \sum_{k=1}^t \EE[B_k | \mathcal F_{k-1}] \right\|_{L_p} + \left\| \max_{t \in [T]} \sum_{k=1}^t \bar B_k \right\|_{L_p}.
		\end{align}
		In order to bound the penultimate term in~\eqref{eq:bound:sup:A}, we first note that
		\begin{align}
		\left| \EE[B_k | \mathcal F_{k-1}] \right|
		&= 2 \gamma \left| \EE \left[ \left( \nabla h(\bm \phi_{k-1}) - \bm g_t(\bm \phi_{k-1}) \right) | \mathcal F_{k-1} \right]^\top (\bm \phi_{k-1} - \bm \phi_{\star}) \right| \notag \\
		&\leq 2 \gamma \| \EE \left[ \left( \nabla h(\bm \phi_{k-1}) - \bm g_k(\bm \phi_{k-1}) \right) | \mathcal F_{k-1} \right] \| \| \bm \phi_{k-1} - \bm \phi_{\star} \| \notag \\
		&\leq 2 \gamma \varepsilon_{k-1} \| \bm \phi_{k-1} - \bm \phi_{\star} \|
		\leq 2 \gamma \varepsilon_{k-1}\sqrt{ A_{k-1}} \label{eq:bound:B}
		\end{align}
		for all $k\in\mathbb N$, where the second inequality holds due to Assumption~\ref{assumption:main}\,\ref{assumption:main:gradients}, and the last inequality follows from \eqref{eq:bound:A}. This in turn implies that for all $t \in [T]$ we have
		\begin{align*}
		\left| \sum_{k=1}^t \EE[B_k | \mathcal F_{k-1}] \right|
		\leq 2 \gamma \sum_{k=1}^t \varepsilon_{k-1} \sqrt{A_{k-1}}
		\leq 2 \gamma \left( \sum_{k=1}^t \varepsilon_{k-1} \right) \left( \max_{k \in [t]} \sqrt{A_{k-1}} \right)
		\leq 2 \gamma \bar \varepsilon \sqrt{t} \max_{k \in [t]} \sqrt{A_{k-1}},
		\end{align*}
		where the last inequality exploits~\eqref{eq:bound:varepsilon}. Therefore, the penultimate term in~\eqref{eq:bound:sup:A} satisfies
		\begin{align}
		\label{eq:B_k:1}
		\left\| \max_{t \in [T]} \sum_{k=1}^t \EE[B_k | \mathcal F_{k-1}] \right\|_{L_p}
		\leq 2 \gamma \bar \varepsilon \sqrt{T} \left\| \max_{t \in [T+1]} \sqrt{A_{t-1}} \right\|_{L_p}
		= 2 \gamma \bar \varepsilon \sqrt{T} \left\| \max_{t \in [T+1]} A_{t-1} \right\|_{L_{p/2}}^{1/2},
		\end{align}
		where the equality follows from the definition of the $L_p$-norm.
		
		Next, we bound the last term in~\eqref{eq:bound:sup:A} by using the BRP inequality of Lemma~\ref{lemma:BRP}. To this end, note~that
		\begin{align*}
		|\bar B_t | 
		& \leq | B_t | + | \EE[B_t | \mathcal F_{t-1}] | \\
		&\leq 2 \gamma \| \bm \phi_{t-1} - \bm \phi_{\star} \| \| \nabla h(\bm \phi_{t-1}) - \bm g_t(\bm \phi_{t-1}) \| +  2 \gamma \varepsilon_{t-1} \sqrt{A_{t-1}} \\
		&\leq 2 \gamma \sqrt{A_{t-1}} \left( \| \nabla h(\bm \phi_{t-1}) \| + \| \bm g_t(\bm \phi_{t-1}) \| \right) +  2 \gamma \varepsilon_{t-1} \sqrt{A_{t-1}}
		\leq 2 \gamma (2R + \varepsilon_{t-1}) \sqrt{A_{t-1}}
		\end{align*}
		for all $t\in\mathbb N$, where the second inequality exploits the definition of~$B_t$ and~\eqref{eq:bound:B}, the third inequality follows from~\eqref{eq:bound:A}, and the last inequality holds because of Assumption~\ref{assumption:main}\,\ref{assumption:main:bounded}. Hence, we obtain
		\begin{align*}
		\textstyle
		\left\| \max_{t \in [T]} | \bar B_t | \right\|_{L_p} 
		\leq 2 \gamma \left( 2 R + \max_{t \in [T]} \varepsilon_{t-1} \right) \left\| \max_{t \in [T]} \sqrt{A_{t-1}} \right\|_{L_p} 
		\leq ( 4 \gamma R + \gamma \bar \varepsilon) \left\| \max_{t \in [T+1]} A_{t-1} \right\|_{L_{p/2}}^{1/2},
		\end{align*}
		where the second inequality follows from~\eqref{eq:bound:varepsilon} and the definition of the $L_p$-norm. In addition, we have
		\begin{align*}
		\left\| \sum_{t=1}^T \EE[ \bar B_t^2 | \mathcal F_{t-1}] \right\|_{L_{p/2}}^{1/2} 
		= \left\| \sqrt{\sum_{t=1}^T \EE[ \bar B_t^2 | \mathcal F_{t-1}]} \right\|_{L_p}
		&\leq 2 \gamma \left\| \sqrt{ \sum_{t=1}^T (2R + \varepsilon_{t-1})^2 A_{t-1} } \right\|_{L_p} \\
		&\leq 2 \gamma \left( \sum_{t=1}^T (2R + \varepsilon_{t-1})^2 \right)^{1/2}
		\left\| \max_{t \in [T+1]} A_{t-1}^{1/2} \right\|_{L_p} \\
		&\leq 2 \gamma \left( 2 R \sqrt{T} +  \sqrt{\sum_{t=1}^T \varepsilon_{t-1}^2} \right)
		\left\| \max_{t \in [T+1]} A_{t-1}^{1/2} \right\|_{L_p} \\
		&\leq \left( 4 \gamma R \sqrt{T} + \gamma \bar \varepsilon \sqrt{T} \right) \left\| \max_{t \in [T+1]} A_{t-1} \right\|_{L_{p/2}}^{1/2},
		\end{align*}
		where the first inequality exploits the upper bound on $|\bar B_t|$ derived above, which implies that
		$\EE[ \bar B_t ^2 | \mathcal F_{t-1}] \leq 4 \gamma^2 (2R + \varepsilon_{t-1})^2 A_{t-1}$. The last three inequalities follow from the H\"{o}lder inequality, the triangle inequality for the Euclidean norm and the two inequalities in~\eqref{eq:bound:varepsilon}, respectively. Recalling that $p \geq 2$, we may then apply the BRP inequality of Lemma~\ref{lemma:BRP} to the martingale differences $\bar B_t$, $t\in[T]$, and use the bounds derived in the last two display equations in order to conclude that
		\begin{align}
		\label{eq:BRP-application}
		    \left\| \max_{t \in [T]} \left| \sum_{k=1}^t \bar B_k \right| \right\|_{L_p}
		    &\leq \left( 4 \gamma R \sqrt{pT} + \gamma \bar \varepsilon \sqrt{pT} + \gamma \bar \varepsilon p + 4 \gamma R p \right) \left\| \max_{t \in [T+1]} A_{t-1} \right\|_{L_{p/2}}^{1/2}.
		\end{align}
		Substituting~\eqref{eq:B_k:1} and~\eqref{eq:BRP-application} into~\eqref{eq:bound:sup:A}, we thus obtain 
		\begin{align*}
		\left\| \max_{t \in [T+1]} A_{t-1} \right\|_{L_p} 
		&\leq T \gamma^2 R^2 + \| \bm \phi_0 - \bm \phi^\star \|^2 + \left( 4 \gamma R \left( \sqrt{pT} + p \right) + \gamma \bar \varepsilon \left( \sqrt{pT} + p +2 \sqrt{T} \right) \right) \left\| \max_{t \in [T+1]} A_{t-1} \right\|_{L_{p/2}}^{1/2} \\
		&\leq T \gamma^2 R^2 + \| \bm \phi_0 - \bm \phi^\star \|^2 + 6 \gamma \left( R + \bar \varepsilon \right) \sqrt{pT} \left\| \max_{t \in [T+1]} A_{t-1} \right\|_{L_{p/2}}^{1/2},
		\end{align*}
		where the second inequality holds because $p \leq T/4$ by assumption, which implies that $\sqrt{pT} + p \leq 1.5 \sqrt{pT} $ and $ \sqrt{pT} + p + 2 \sqrt{T} \leq 6 \sqrt{pT}$. As Jensen's inequality ensures that $\| \bm z \|_{L_{p/2}} \leq \| \bm z \|_{L_p}$ for any random variable~$\bm z$ and $p > 0$, the following inequality holds for all $2 \leq p \leq T/4$.
		\begin{align*}
		\left\| \max_{t \in [T+1]} A_{t-1} \right\|_{L_p} 
		&\leq T \gamma^2 R^2 + \| \bm \phi_0 - \bm \phi^\star \|^2 + 6 \gamma \left( R + \bar \varepsilon \right) \sqrt{pT} \left\| \max_{t \in [T+1]} A_{t-1} \right\|_{L_p}^{1/2}
		\end{align*}
		To complete the proof of Case~II, we note that for any numbers $a, b, c \geq 0$ the inequality $c \leq a + 2b \sqrt{c} $ is equivalent to $\sqrt{c} \leq b + \sqrt{b^2+a}$ and therefore also to $c \leq (b + \sqrt{b^2+a})^2 \leq 4b^2 + 2a$. Identifying $a$ with $T \gamma^2 R^2 + \| \bm \phi_0 - \bm \phi^\star \|^2$, $b$ with $3\gamma \left( R + \bar \varepsilon \right) \sqrt{pT}$ and $c$ with $\| \max_{t \in [T+1]} A_{t-1}\|_{L_p}$ then allows us to translate the inequality in the last display equation to
		\begin{align}
		\label{eq:bound:Lp:A}
		\left\| A_{T} \right\|_{L_p}
		\leq \left\| \max_{t \in [T+1]} A_{t-1} \right\|_{L_p} 
		&\leq 2 T \gamma^2 R^2 + 2 \| \bm \phi_0 - \bm \phi^\star \|^2 + 36 \gamma^2 \left( R + \bar \varepsilon \right)^2 p T.
		\end{align}
		Thus, for any $2 \leq p \leq T/4$, we have again found an upper bound on $\| A_{T}\|_{L_p}$ that is affine in $p$.
		
		\textbf{Case III $({p = 1})$:} Recalling the definition of $A_T\ge 0$, we find that
		\begin{align*}
		    \| A_T \|_{L_{1}} = \EE [A_T] 
		    &= T \gamma^2 R^2 + \| \bm \phi_0 - \bm \phi^\star \|^2 + \EE \left[ \, \sum_{t=1}^T \EE [B_t | \mathcal F_{t-1}] \right] \\
		    &\leq T \gamma^2 R^2 + \| \bm \phi_0 - \bm \phi^\star \|^2 + \left\| \max_{t \in [T]} \sum_{k=1}^t \EE[B_k | \mathcal F_{k-1}] \right\|_{L_1} \\
		    &\leq T \gamma^2 R^2 + \| \bm \phi_0 - \bm \phi^\star \|^2 + 2 \gamma \bar \varepsilon \sqrt{T} \left\| \max_{t \in [T+1]} A_{t-1} \right\|^{1/2}_{L_{1/2}} \\
		    &\leq T \gamma^2 R^2 + \| \bm \phi_0 - \bm \phi^\star \|^2 + 2 \gamma \bar \varepsilon \sqrt{T} \left\| \max_{t \in [T+1]} A_{t-1} \right\|^{1/2}_{L_{2}},
		\end{align*}
		where the second inequality follows from the estimate~\eqref{eq:B_k:1}, which holds indeed for all~$p\in\mathbb N$, while the last inequality follows from Jensen's inequality. By the second inequality in~\eqref{eq:bound:Lp:A} for $p=2$, we thus find
		\begin{subequations}
		\label{eq:bound:E:A}
		\begin{align}
		    \label{eq:bound:E:A1}
		    \| A_T \|_{L_{1}} 
		    &\leq T \gamma^2 R^2 + \| \bm \phi_0 - \bm \phi^\star \|^2 + 2 \bar \varepsilon \gamma \sqrt{T} \cdot \sqrt{2 T \gamma^2 R^2 + 2 \| \bm \phi_0 - \bm \phi^\star \|^2 + 72 \gamma^2 (R + \bar \varepsilon)^2 T} \\
		    &\leq 2 T \gamma^2 R^2 + 2 \| \bm \phi_0 - \bm \phi^\star \|^2 + 36 \gamma^2 (R + \bar \varepsilon)^2 T + 2 \bar \varepsilon^2 \gamma^2 T ,
		    \label{eq:bound:E:A2}
		\end{align}
		\end{subequations}
		where the last inequality holds because $2ab \leq 2a^2 + b^2/ 2$ for all $a,b\in\R$.
		
		We now combine the bounds derived in Cases~I, II and~III to obtain a universal bound on $\left\| A_{T} \right\|_{L_p}$ that holds for all $p\in\mathbb N$. Specifically, one readily verifies that the bound
		\begin{align}
		\label{eq:universal-bound}
		\left\| A_{T} \right\|_{L_p}
		&\leq 2 \| \bm \phi_0 - \bm \phi^\star \|^2 + 40 \gamma^2 \left( R + \bar \varepsilon \right)^2 p T,
		\end{align}
		is more conservative than each of the bounds~\eqref{eq:bound:A:T/4}, \eqref{eq:bound:Lp:A} and \eqref{eq:bound:E:A}, and thus it holds indeed for any $p \in \mathbb N$. Combining this universal bound with~\eqref{eq:bound:A:convexity} proves the first inequality from the proposition statement.
		
		In order to prove the second inequality, we need to extend \citep[Proposition~7]{bach2013adaptivity} to biased gradient oracles. To this end, we first note that
		\begin{align*}
		\left\| \nabla h \left(\frac{1}{T} \sum_{t=1}^{T} \bm \phi_{t-1} \right) \right\|
		&\leq \left\| \nabla h \left(\frac{1}{T} \sum_{t=1}^{T} \bm \phi_{t-1} \right) - \frac{1}{T} \sum_{t=1}^T \nabla h(\bm \phi_{t-1}) \right\| 
		+ \left\| \frac{1}{T} \sum_{t=1}^T \nabla h(\bm \phi_{t-1}) \right\| \\
		&\leq  2 M \left( \frac{1}{T} \sum_{t=1}^T h(\bm \phi_{t-1}) - h(\bm \phi^\star) \right) + \left\| \frac{1}{T} \sum_{t=1}^T \nabla h(\bm \phi_{t-1}) \right\| \\
		&\leq \frac{M}{T \gamma} A_T + \left\| \frac{1}{T} \sum_{t=1}^T \nabla h(\bm \phi_{t-1}) \right\|,
		\end{align*}
		where the second inequality follows from Lemma~\ref{lemma:concordance}\,\ref{lemma:smoothness}, and the third inequality holds due to~\eqref{eq:bound:A}. By Minkowski's inequality \eqref{eq:minkowski}, we thus have for any $p \geq 1$ that
		\begin{align*}
		\left\| \nabla h \left(\frac{1}{T} \sum_{t=1}^{T} \bm \phi_{t-1} \right) \right\|_{L_{2p}}
		&\leq \frac{M}{T \gamma} \| A_T \|_{L_{2p}} + \left\| \frac{1}{T} \sum_{t=1}^T \nabla h(\bm \phi_{t-1}) \right\|_{L_{2p}} \\
		&\leq \frac{2 M}{T \gamma} \| \bm \phi_0 - \bm \phi^\star \|^2 + 80 M \gamma \left( R + \bar \varepsilon \right)^2 p + \left\| \frac{1}{T} \sum_{t=1}^T \nabla h(\bm \phi_{t-1}) \right\|_{L_{2p}},
		\end{align*}
		where the last inequality follows from the universal bound~\eqref{eq:universal-bound}. In order to estimate the last term in the above expression, we recall that the update rule~\eqref{eq:gd} is equivalent to $\bm g_t(\bm \phi_{t-1}) = \left( \bm \phi_{t-1} - \bm \phi_{t} \right) / \gamma ,$ which in turn implies that $\sum_{t=1}^T \bm g_t(\bm \phi_{t-1}) = \left( \bm \phi_0 - \bm \phi_T \right) / \gamma.$ Hence, for any $p \geq 1$, we have
		\begin{align*}
		\left\| \frac{1}{T} \sum_{t=1}^T \nabla h(\bm \phi_{t-1}) \right\|_{L_{2p}}
		&= \left\| \frac{1}{T} \sum_{t=1}^T \Big( \nabla h(\bm \phi_{t-1}) - \bm g_t(\bm \phi_{t-1}) \Big) + \frac{\bm \phi_0 - \bm \phi^\star}{T \gamma} + \frac{\bm \phi^\star - \bm \phi_T}{T \gamma} \right\|_{L_{2p}} \\
		&\leq \left\| \frac{1}{T} \sum_{t=1}^T \nabla h(\bm \phi_{t-1}) - \bm g_t(\bm \phi_{t-1}) \right\|_{L_{2p}} + \frac{1}{T \gamma} \left\| \bm \phi_0 - \bm \phi^\star \right\| + \frac{1}{T \gamma} \left\| \bm \phi^\star - \bm \phi_T \right\|_{L_{2p}} \\
		&\leq \left\| \frac{1}{T} \sum_{t=1}^T \nabla h(\bm \phi_{t-1}) - \bm g_t(\bm \phi_{t-1}) \right\|_{L_{2p}} + \frac{1}{T \gamma} \left\| \bm \phi_0 - \bm \phi^\star \right\| + \frac{1}{T \gamma} \left\| A_T \right\|_{L_{p}}^{1/2} \\
		&\leq \left\| \frac{1}{T} \sum_{t=1}^T \nabla h(\bm \phi_{t-1}) - \bm g_t(\bm \phi_{t-1}) \right\|_{L_{2p}} + \frac{1 + \sqrt{2}}{T \gamma} \left\| \bm \phi_0 - \bm \phi^\star \right\| + \frac{2 \sqrt{10} \left( R + \bar \varepsilon \right) \sqrt{p}}{\sqrt{T}},
		\end{align*}
		where the first inequality exploits Minkowski's inequality~\eqref{eq:minkowski}, the second inequality follows from~\eqref{eq:bound:A}, which implies that $\| \bm \phi^\star - \bm \phi_T \| \leq \sqrt{A_T}$, and the definition of the $L_p$-norm. The last inequality in the above expression is a direct consequence of the universal bound~\eqref{eq:universal-bound} and the inequality $ \sqrt{a+b} \leq \sqrt{a} + \sqrt{b}$ for all $a,b\ge 0$. Next, define for any $t\in\mathbb N$ a martingale difference of the form
		$$\bm C_t = \frac{1}{T} \Big( \nabla h(\bm \phi_{t-1}) - \bm g_t(\bm \phi_{t-1}) - \EE[\nabla h(\bm \phi_{t-1}) - \bm g_t(\bm \phi_{t-1}) | \mathcal F_{t-1}] \Big).$$
		Note that these martingale differences are bounded because
		\begin{align*}
		\| \bm C_t \| 
		&\leq \frac{1}{T} \Big( \| \nabla h(\bm \phi_{t-1}) \| + \| \bm g_t(\bm \phi_{t-1}) \| + \| \EE[\nabla h(\bm \phi_{t-1}) - \bm g_t(\bm \phi_{t-1}) | \mathcal F_{t-1}] \| \Big) \leq \frac{2R + \varepsilon_{t-1}}{T} \leq \frac{2R + \bar \varepsilon}{T},
		\end{align*}
		and thus the BRP inequality of Lemma~\ref{lemma:BRP} implies that
		\begin{align*}
		    \left\| \sum_{t=1}^T \bm C_t \right\|_{L_{2p}} \leq \sqrt{2p} \, \frac{2R + \bar \varepsilon}{\sqrt{T}} + 2p \, \frac{2R + \bar \varepsilon}{T}.
		\end{align*}
		Recalling the definition of the martingale differences $\bm C_t$, $t\in\mathbb N$, this bound allows us to conclude that
		\begin{align*}
		\frac{1}{T} \left\| \sum_{t=1}^T \nabla h(\bm \phi_{t-1}) - \bm g_t(\bm \phi_{t-1}) \right\|_{L_{2p}} 
		&\leq \left\| \sum_{t=1}^T \bm C_t \right\|_{L_{2p}} + \frac{1}{T} \left\| \sum_{t=1}^T \EE[\nabla h(\bm \phi_{t-1}) - \bm g_t(\bm \phi_{t-1}) | \mathcal F_{t-1}] \right\|_{L_{2p}} \\
		&\leq \sqrt{2p} \, \frac{2R + \bar \varepsilon}{\sqrt{T}} + 2p \, \frac{2R + \bar \varepsilon}{T} + \frac{\bar \varepsilon}{\sqrt{T}}
		\leq 2 \sqrt{2p} \, \frac{R + \bar \varepsilon}{\sqrt{T}} + 4p \, \frac{R + \bar \varepsilon}{T},
		\end{align*}
		where the second inequality exploits Assumption~\ref{assumption:main}\,\ref{assumption:main:gradients} as well as the second inequality in~\eqref{eq:bound:varepsilon}.
		Combining all inequalities derived above and observing that $2\sqrt{2} + 2 \sqrt{10} < 10 $ finally yields
		\begin{align*}
		\left\| \nabla h \left(\frac{1}{T} \sum_{t=1}^{T} \bm \phi_{t-1} \right) \right\|_{L_{2p}}
		&\leq \frac{2 M}{T \gamma} \| \bm \phi_0 - \bm \phi^\star \|^2 + 80 M \gamma \left( R + \bar \varepsilon \right)^2 p +  2 \sqrt{2p} \, \frac{R + \bar \varepsilon}{\sqrt{T}} + 4p \, \frac{R + \bar \varepsilon}{T} \\
		&\qquad + \frac{1 + \sqrt{2}}{T \gamma} \left\| \bm \phi_0 - \bm \phi^\star \right\| + \frac{2 \sqrt{10} \left( R + \bar \varepsilon \right) \sqrt{p}}{\sqrt{T}} \\
		&\leq \frac{G}{\sqrt{T}} \left( 10 \sqrt{p} + \frac{4p}{\sqrt{T}} + 80 G^2 \gamma \sqrt{T} p + \frac{2}{\gamma \sqrt{T}} \| \bm \phi_0 - \bm \phi^\star \|^2 + \frac{3}{G \gamma \sqrt{T}} \| \bm \phi_0 - \bm \phi^\star \| \right), 
		\end{align*}
		where $G = \max\{ M, R + \bar \varepsilon \}$. This proves the second inequality from the proposition statement.
	\end{proof}
	
	The following corollary follows immediately from the proof of Proposition~\ref{proposition:moments}.
	
	\begin{corollary}
	\label{corollary:auxiliary}
	Consider the inexact gradient descent algorithm~\eqref{eq:gd} with constant step size $\gamma > 0$. If Assumptions~\ref{assumption:main}\,\ref{assumption:main:gradients}--\ref{assumption:main:bounded} hold with $\varepsilon_t \leq {\bar \varepsilon}/{(2\sqrt{1+t})}$ for some $\bar \varepsilon \geq 0$, then we have
	\begin{align*}
		\frac{1}{T} \sum_{t=1}^T \mathbb E \left[ \left( \nabla h(\bm \phi_{t}) - \bm g_t(\bm \phi_{t}) \right)^\top (\bm \phi_{t} - \bm \phi_{\star}) \right]
		\leq \frac{\bar \varepsilon}{\sqrt{T}} \sqrt{2 \| \bm \phi_0 - \bm \phi^\star \|^2 + 74 \gamma^2 (R + \bar \varepsilon)^2 T}.
	\end{align*}
	\end{corollary}
	\begin{proof}[Proof of Corollary~\ref{corollary:auxiliary}]
	Defining $B_t$ as in the proof of Proposition~\ref{proposition:moments}, we find
	\begin{align*}
	    \frac{1}{T} \sum_{t=1}^T \mathbb E \left[ \left( \nabla h(\bm \phi_{t}) - \bm g_t(\bm \phi_{t}) \right)^\top (\bm \phi_{t} - \bm \phi_{\star}) \right]
	    &=
	    \frac{1}{2 \gamma T} \EE \left[  \sum_{t=1}^T \EE [B_t | \mathcal F_{t-1}] \right] \\
		&\leq \frac{\bar \varepsilon}{\sqrt{T}} \sqrt{2 T \gamma^2 R^2 + 2 \| \bm \phi_0 - \bm \phi^\star \|^2 + 72 \gamma^2 (R + \bar \varepsilon)^2 T},
	\end{align*}
	where the inequality is an immediate consequence of the reasoning in Case~(III) in the proof of Proposition~\ref{proposition:moments}. The claim then follows from the trivial inequality $R+ \bar \varepsilon \geq R$.
	\end{proof}

	Armed with Proposition~\ref{proposition:moments} and Corollary~\ref{corollary:auxiliary}, we are now ready to prove the main convergence result.
	
	\begin{theorem}	\label{theorem:convergence}
		Consider the inexact gradient descent algorithm~\eqref{eq:gd} with constant step size $\gamma > 0$. If Assumptions~\ref{assumption:main}\,\ref{assumption:main:gradients}--\ref{assumption:main:bounded} hold with $\varepsilon_t \leq {\bar \varepsilon}/{(2\sqrt{1+t})}$ for some $\bar \varepsilon \geq 0$, then the following statements hold.
		\begin{enumerate} [label=(\roman*)]
			\item \label{theorem:convergence:Lipschitz} 
			If $\gamma = 1 / (2 (R + \bar \varepsilon)^2 \sqrt{T})$, then we have
			\begin{align*}
			\EE \left[ h \left(\frac{1}{T} \sum_{t=1}^{T} \bm \phi_{t-1} \right) \right] - h(\bm \phi^\star)
			&\leq \frac{(R + \bar \varepsilon)^2}{\sqrt{T}} \| \bm \phi_0 - \bm \phi^\star \|^2 + \frac{1}{4\sqrt{T}} + \frac{\bar \varepsilon}{\sqrt{T}} \sqrt{2 \| \bm \phi_0 - \bm \phi^\star \|^2 + \frac{37}{2(R + \bar \varepsilon)^2}} .
			\end{align*}
			\item \label{theorem:convergence:smooth} 
			If $\gamma = 1 / (2 (R + \bar \varepsilon)^2 \sqrt{T} + L)$ and the Assumptions~\ref{assumption:main}\,\ref{assumption:main:smooth}--\ref{assumption:main:moment} hold in addition to the blanket assumptions mentioned above, then we have
			\begin{align*}
			\EE \left[ h \left(\frac{1}{T} \sum_{t=1}^{T} \bm \phi_{t} \right) \right] - h(\bm \phi^\star)
			&\leq \frac{L}{2T}\| \bm \phi_0 - \bm \phi^\star \|^2 + 
			\frac{(R + \bar \varepsilon)^2}{\sqrt{T}} \| \bm \phi_0 - \bm \phi^\star \|^2 + \frac{\sigma^2}{4 (R+\bar \varepsilon)^2\sqrt{T}} \\
			&\qquad+ \frac{\bar \varepsilon}{\sqrt{T}} \sqrt{2 \| \bm \phi_0 - \bm \phi^\star \|^2 + \frac{37}{2(R + \bar \varepsilon)^2}}.
			\end{align*}
			\item \label{theorem:convergence:concordance} 
			If $\gamma = 1 / (2 G^2 \sqrt{T})$ with $G = \max \{M, R + \bar \varepsilon \}$, the smallest eigenvalue $\kappa$ of $\nabla^2 h(\bm \phi^\star)$ is strictly positive and Assumption~\ref{assumption:main}\,\ref{assumption:main:concordance} holds in addition to the blanket assumptions mentioned above, then we have
			\begin{align*}
			\mathbb E \left[ h\left(\frac{1}{T} \sum_{t=1}^{T} \bm \phi_{t-1}\right) \right] - h(\bm \phi^\star)
			&\leq \frac{G^2}{\kappa T} \left( 4 G \| \bm \phi_0 - \bm \phi^\star \| + 20 \right)^4.
			\end{align*}
		\end{enumerate}
	\end{theorem}
	The proof of Theorem~\ref{theorem:convergence} relies on the following concentration inequalities due to \citet{bach2013adaptivity}.
	
	\begin{lemma} \label{lemma:probability101} [Concentration inequalities]~
		\begin{enumerate} [label=(\roman*)]
			\item 
			\label{lemma:sub-exponential1}
			{\citep[Lemma~11]{bach2013adaptivity}:}
			If there exist $a,b>0$ and a random variable $\bm z \in \R^n$  with $ \| \bm z \|_{L_p} \leq a + b p $ for all $p \in \mathbb N$, then we have
			$$ \mathbb P \left[ \| \bm z \| \geq 3 b s + 2 a \right] \leq 2 \exp(-s)\quad \forall s \geq 0. $$
			
			\item
			\label{lemma:sub-exponential2}
			{\citep[Lemma~12]{bach2013adaptivity}:}
			If there exist $a,b,c>0$ and a random variable $\bm z \in \R^n$  with $ \| \bm z \|_{L_p} \leq (a \sqrt{p} + b p + c)^2 $ for all $p \in [T]$, then we have
			$$ \mathbb P \left[ \| \bm z \| \geq (2 a \sqrt{s} + 2 b s + 2 c)^2 \right] \leq 4 \exp(-s)\quad \forall s \leq T. $$
		\end{enumerate}
	\end{lemma}
	
	\begin{proof} [Proof of Theorem~\ref{theorem:convergence}]
	    Define $A_t$ as in the proof of Proposition~\ref{proposition:moments}. Then, we have
	    \begin{align}
	        \EE \left[ h \left(\frac{1}{T} \sum_{t=0}^{T-1} \bm \phi_{t} \right) - h(\bm \phi^\star) \right]
	        &\leq \frac{\mathbb E[A_T]}{2 \gamma T}
	        = \frac{\| \bm \phi_0 - \bm \phi^\star \|^2}{2 \gamma T} + \frac{\gamma R^2}{2} + \frac{1}{T} \sum_{t=1}^T \mathbb E \left[ \left( \nabla h(\bm \phi_{t}) - \bm g_t(\bm \phi_{t}) \right)^\top (\bm \phi_{t} - \bm \phi_{\star}) \right] \notag \\
	        &\leq \frac{\| \bm \phi_0 - \bm \phi^\star \|^2}{2 \gamma T} + \frac{\gamma R^2}{2} + \frac{\bar \varepsilon}{\sqrt{T}} \sqrt{2 \| \bm \phi_0 - \bm \phi^\star \|^2 + 74 \gamma^2 (R + \bar \varepsilon)^2 T}, \label{eq:main:convergence}
	    \end{align}
	    where the two inequalities follow from~\eqref{eq:bound:A:convexity} and from Corollary~\ref{corollary:auxiliary}, respectively.
		Setting the step size to $\gamma = 1 / ( 2 (R+ \bar \varepsilon)^2 \sqrt{T} )$ then completes the proof of assertion~\ref{theorem:convergence:Lipschitz}.
		
		Assertion~\ref{theorem:convergence:smooth} generalizes \citep[Theorem~1]{dekel2012optimal}.
		By the $L$-smoothness of $h(\bm \phi)$, we have
		\begin{align}
		h(\bm \phi_{t}) 
		&\leq h(\bm \phi_{t-1}) + \nabla h(\bm \phi_{t-1})^\top (\bm \phi_{t} - \bm \phi_{t-1}) + \frac{L}{2}\|\bm \phi_{t} - \bm \phi_{t-1}\|^2  \notag \\
		&= h(\bm \phi_{t-1}) + \bm g_t(\bm \phi_{t-1})^\top (\bm \phi_{t} - \bm \phi_{t-1}) + \bm \left( \nabla h(\bm \phi_{t-1}) - \bm g_t(\bm \phi_{t-1}) \right)^\top (\bm \phi_{t} - \bm \phi_{t-1}) + \frac{L}{2}\|\bm \phi_{t} - \bm \phi_{t-1}\|^2 \notag \\
		&\leq h(\bm \phi_{t-1}) + \bm g_t(\bm \phi_{t-1})^\top (\bm \phi_{t} - \bm \phi_{t-1}) + \frac{\zeta}{2}\| \nabla h(\bm \phi_{t-1}) - \bm g_t(\bm \phi_{t-1}) \|^2 + \frac{L + 1/\zeta}{2}\|\bm \phi_{t} - \bm \phi_{t-1}\|^2, \label{eq:smooth:update}
		\end{align}
		where the last inequality exploits the Cauchy-Schwarz inequality together with the elementary inequality $2ab \leq \zeta a^2 + b^2 / \zeta$, which holds for all $a,b\in\R$ and $\zeta > 0$.
		Next, note that the iterates satisfy the recursion
		\begin{equation*}
		\| \bm \phi_{t-1} - \bm \phi^\star \|^2 = \| \bm \phi_{t-1} - \bm \phi_{t} \|^2 + \| \bm \phi_{t} - \bm \phi^\star \|^2 + 2 (\bm \phi_{t-1} - \bm \phi_{t})^\top (\bm \phi_{t} - \bm \phi^\star),
		\end{equation*}
		which can be re-expressed as 
		\begin{equation*}
		\bm g_t(\bm \phi_{t-1})^\top (\bm \phi_{t} - \bm \phi^\star) = \frac{1}{2 \gamma} \left( \| \bm \phi_{t-1} - \bm \phi^\star \|^2 - \| \bm \phi_{t-1} - \bm \phi_{t} \|^2 - \| \bm \phi_{t} - \bm \phi^\star \|^2 \right)
		\end{equation*}
		by using the update rule~\eqref{eq:gd}. In the remainder of the proof we assume that $0 < \gamma < 1 / L$.  Substituting the above equality into~\eqref{eq:smooth:update} and setting~$\zeta = \gamma / (1 - \gamma L)$ then~yields
		\begin{align*}
		h(\bm \phi_{t}) 
		&\leq h(\bm \phi_{t-1}) + \bm g_t(\bm \phi_{t-1})^\top (\bm \phi^\star - \bm \phi_{t-1}) + \frac{\gamma}{2(1 - \gamma L)} \| \nabla h(\bm \phi_{t-1}) - \bm g_t(\bm \phi_{t-1}) \|^2 \\
		& \qquad + \frac{1}{2 \gamma} \left( \| \bm \phi_{t-1} - \bm \phi^\star \|^2 - \| \bm \phi_{t} - \bm \phi^\star \|^2 \right).
		\end{align*}
		By the convexity of $h$, we have $h(\bm \phi^\star) \geq h(\bm \phi_{t-1}) + \nabla h(\bm \phi_{t-1})^\top (\bm \phi^\star - \bm \phi_{t-1})$, which finally implies that
		\begin{align*}
		h(\bm \phi_{t})
		&\leq h(\bm \phi^\star) + \left( \nabla h(\bm \phi_{t-1}) - \bm g_t(\bm \phi_{t-1}) \right)^\top ( \bm \phi_{t-1} - \bm \phi^\star) + \frac{\gamma}{2(1 - \gamma L)} \| \nabla h(\bm \phi_{t-1}) - \bm g_t(\bm \phi_{t-1}) \|^2 \\
		& \qquad + \frac{1}{2\gamma} \left( \| \bm \phi_{t-1} - \bm \phi^\star \|^2 - \| \bm \phi_{t} - \bm \phi^\star \|^2 \right). 
		\end{align*}
		Averaging the above inequality over $t$ and taking expectations then yields the estimate
		\begin{align*}
		\mathbb E \left[ \frac{1}{T} \sum_{t=1}^T h(\bm \phi_{t}) \right] - h(\bm \phi^\star)
		&\leq \frac{\| \bm \phi_0 - \bm \phi^\star \|^2}{2\gamma T} + \frac{\gamma}{2 (1 - \gamma L)} \mathbb E \left[ \frac{1}{T} \sum_{t=1}^T \| \nabla h(\bm \phi_{t-1}) - \bm g_t(\bm \phi_{t-1}) \|^2 \right] \\
		&\qquad + \mathbb E \left[ \frac{1}{T} \sum_{t=1}^T \left( \nabla h(\bm \phi_{t-1}) - \bm g_t(\bm \phi_{t-1}) \right)^\top (\bm \phi_{t-1} - \bm \phi_{\star}) \right] \\
		&\leq \frac{\| \bm \phi_0 - \bm \phi^\star \|^2}{2\gamma T} + \frac{\gamma \sigma^2}{2 (1 - \gamma L)} + \frac{\bar \varepsilon}{\sqrt{T}} \sqrt{2 \| \bm \phi_0 - \bm \phi^\star \|^2 + 74 \gamma^2 (R + \bar \varepsilon)^2 T},
		\end{align*}
		where the second inequality exloits Assumption~\ref{assumption:main}\,\ref{assumption:main:moment} and Corollary~\ref{corollary:auxiliary}. Using Jensen's inequality to move the average over $t$ inside~$h$, assertion~\ref{theorem:convergence:smooth} then follows by setting $\gamma = 1 / (2 (R + \bar \varepsilon)^2 \sqrt{T} + L)$ and observing that $\gamma / ( 1 - \gamma L) = 1 / ( 2(R+\bar \varepsilon)^2 \sqrt{T} )$.

		To prove assertion~\ref{theorem:convergence:concordance}, we distinguish two different cases.
		
		\textbf{Case~I:} Assume first that $4 G^2 \| \bm \phi_0 - \bm \phi^\star \|^2 + 6 G \| \bm \phi_0 - \bm \phi^\star \| \leq {\kappa \sqrt{T}}/{(8 G^2)}$, where~$G = \max \{M, R + \bar \varepsilon \}$ and $\kappa$ denotes the smallest eigenvalue of $\nabla^2 h(\bm \phi^\star)$. By a standard formula for the expected value of a non-negative random variable, we find
		\begin{align}
		    \EE \left[ h \left(\frac{1}{T} \sum_{t=1}^{T} \bm \phi_{t-1} \right) - h(\bm \phi^\star) \right] \nonumber
		    &= \phantom{+} \int_{0}^\infty \mathbb P \left[ h \left(\frac{1}{T} \sum_{t=1}^{T} \bm \phi_{t-1} \right) - h(\bm \phi^\star) \geq u \right] \diff u \\
		    &= \phantom{+} \int_{0}^{u_1} \mathbb P \left[ h \left(\frac{1}{T} \sum_{t=1}^{T} \bm \phi_{t-1} \right) - h(\bm \phi^\star) \geq u \right] \diff u \nonumber\\
		    &\quad + \int_{u_1}^{u_2} \mathbb P \left[ h \left(\frac{1}{T} \sum_{t=1}^{T} \bm \phi_{t-1} \right) - h(\bm \phi^\star) \geq u \right] \diff u \label{eq:three-integrals}\\
		    &\quad + \int_{u_2}^{\infty} \mathbb P \left[ h \left(\frac{1}{T} \sum_{t=1}^{T} \bm \phi_{t-1} \right) - h(\bm \phi^\star) \geq u \right] \diff u, \nonumber
		\end{align}
		where $u_1 = \frac{8 G^2}{\kappa T}(4 G^2 \| \bm \phi_0 - \bm \phi^\star \|^2 + 6 G \| \bm \phi_0 - \bm \phi^\star \|)^2$ and $u_2 = \frac{8 G^2}{\kappa T}(\frac{\kappa \sqrt{T}}{4 G^2} + 4 G^2 \| \bm \phi_0 - \bm \phi^\star \|^2 + 6 G \| \bm \phi_0 - \bm \phi^\star \|)^2$.
		The first of the three integrals in~\eqref{eq:three-integrals} is trivially upper bounded by~$u_1$.
		Next, we investigate the third integral in~\eqref{eq:three-integrals}, which is easier to bound from above than the second one. By combining the first inequality in Proposition~\ref{proposition:moments} for $\gamma = 1 / (2 G^2 \sqrt{T})$ with the trivial inequality $G \geq R + \bar \varepsilon$, we find
		\[
		    \left\| h\left(\frac{1}{T} \sum_{t=1}^{T} \bm \phi_{t-1} \right) - h(\bm \phi^\star) \right\|_{L_p} \leq \frac{2G^2}{\sqrt{T}}\,\|\bm \phi_0-\bm \phi^\star\|^2 + \frac{10}{\sqrt{T}} \,p\quad \forall p\in\mathbb N.
		\]
		Lemma~\ref{lemma:probability101}\,\ref{lemma:sub-exponential1} with $a = 2 G^2 \| \bm \phi_0 -\bm \phi^\star \|^2 / \sqrt{T}$ and $b = 10 / \sqrt{T}$ thus implies that
		\begin{align}
		    \label{eq:large:deviation}
		    \mathbb P \left[ h \left(\frac{1}{T} \sum_{t=1}^{T} \bm \phi_{t-1} \right) - h(\bm \phi^\star) \geq \frac{30}{\sqrt{T}} s + \frac{4 G^2}{\sqrt{T}} \| \bm \phi_0 - \bm \phi^\star \|^2 \right] \leq 2 \exp(-s) \quad \forall s \geq 0.
		\end{align}
		We also have
		\begin{align}
		    \label{eq:upper:bound:for:u_2}
		    u_2 - \frac{4 G^2}{\sqrt{T}} \| \bm \phi_0 - \bm \phi^\star \|^2 
		    \geq u_2 - \frac{\kappa}{8 G^2} \geq \frac{8 G^2}{\kappa T} \left( \frac{\kappa \sqrt{T}}{4 G^2} \right)^2 - \frac{\kappa}{8 G^2} = \frac{3 \kappa}{8 G^2} \geq 0,
		\end{align}
		where the first inequality follows from the basic assumption underlying Case~I, while the second inequality holds due to the definition of $u_2$. 
		By~\eqref{eq:large:deviation} and~\eqref{eq:upper:bound:for:u_2}, the third integral in~\eqref{eq:three-integrals} satisfies
		\begin{align*}
		    &\int_{u_2}^{\infty} \mathbb P \left[ h \left(\frac{1}{T} \sum_{t=1}^{T} \bm \phi_{t-1} \right) - h(\bm \phi^\star) \geq u \right] \diff u \\
		    =\;& \int_{u_2 - \frac{4 G^2}{\sqrt{T}} \| \bm \phi_0 - \bm \phi^\star \|^2}^{\infty} \mathbb P \left[ h \left(\frac{1}{T} \sum_{t=1}^{T} \bm \phi_{t-1} \right) - h(\bm \phi^\star) \geq u + \frac{4 G^2}{\sqrt{T}} \| \bm \phi_0 - \bm \phi^\star \|^2 \right] \diff u \\
		    \leq\; & 2 \int_{u_2 - \frac{4 G^2}{\sqrt{T}} \| \bm \phi_0 - \bm \phi^\star \|^2}^\infty \exp \left( -\frac{\sqrt{T} u}{30} \right) \diff u= \frac{60}{\sqrt{T}} \exp \left( -\frac{\sqrt{T}}{30} \left( u_2 - \frac{4 G^2}{\sqrt{T}} \| \bm \phi_0 - \bm \phi^\star \|^2 \right) \right) \\
		    \leq\; & \frac{60}{\sqrt{T}} \exp \left( -\frac{\kappa \sqrt{T}}{80 G^2} \right) \leq \frac{2400 G^2}{\kappa T},
		\end{align*}
		where the first inequality follows from the concentration inequality~\eqref{eq:large:deviation} and the insight from~\eqref{eq:upper:bound:for:u_2} that $u_2 - \frac{4 G^2}{\sqrt{T}}\| \bm \phi_0 - \bm \phi^\star \|^2 \geq 0$. The second inequality exploits again~\eqref{eq:upper:bound:for:u_2}, and the last inequality holds because $\exp(-x) \leq 1 / (2x)$ for all $ x > 0$. We have thus found a simple upper bound on the third integral in~\eqref{eq:three-integrals}. It remains to derive an upper bound on the second integral in~\eqref{eq:three-integrals}. To this end, we first observe that the second inequality in Proposition~\ref{proposition:moments} for $\gamma = 1 / (2 G^2 \sqrt{T})$ translates to
		\[
		    \left\| \left\| \nabla h \left(\frac{1}{T} \sum_{t=1}^{T} \bm \phi_{t-1} \right) \right\|^2 \right\|_{L_p} \leq \frac{G^{2}}{T} \left( 10 \sqrt{p} + \frac{4p}{\sqrt{T}} + 40 p + 4G^2 \| \bm \phi_0 - \bm \phi^\star \|^2 + 6G \| \bm \phi_0 - \bm \phi^\star \| \right)^2 \quad \forall p\in\mathbb N.
		\]
		Lemma~\ref{lemma:probability101}\,\ref{lemma:sub-exponential2} with $a = 10 G / \sqrt{T}$, $b = 4 G / T + 40 G / \sqrt{T}$ and $c = 4 G^3 \| \bm \phi_0 - \bm \phi^\star \|^2 / \sqrt{T} + 6 G^2 \| \bm \phi_0 - \bm \phi^\star \|/\sqrt{T}$ thus gives rise to the concentration inequality
		\begin{align*}
		    \mathbb P \left[ \;\left\| \nabla h \left(\frac{1}{T} \sum_{t=1}^{T} \bm \phi_{t-1} \!\right) \right\|^2 
		    \!\!\!\geq\! \frac{4G^2}{T} \left( 10 \sqrt{s} + \frac{4s}{\sqrt{T}} + 40 s + 4 G^2 \| \bm \phi_0 - \bm \phi^\star \|^2 + 6 G \| \bm \phi_0 - \bm \phi^\star \| \right)^2 \right] \leq 4 \exp(-s),
		\end{align*}
		which holds only for small deviations~$s\leq T$. However, this concentration inequality can be simplified to
		\begin{align*}
		    \mathbb P \left[ \;\left\| \nabla h \left(\frac{1}{T} \sum_{t=1}^{T} \bm \phi_{t-1} \right) \right\| 
		    \geq \frac{2G}{\sqrt{T}} \left( 12 \sqrt{s} + 40 s + 4 G^2 \| \bm \phi_0 - \bm \phi^\star \|^2 + 6 G \| \bm \phi_0 - \bm \phi^\star \| \right) \right] \leq 4 \exp(-s),
		\end{align*}
		which remains valid for all deviations~$ s\ge 0$. To see this, note that if $ s \leq T/4 $, then the simplified concentration inequality holds because $ 4 s / T \leq 2 \sqrt{s / T}$. Otherwise, if $ s > T/4 $, then the simplified concentration inequality holds trivially because the probability on the left hand vanishes. Indeed, this is an immediate consequence of Assumption~\ref{assumption:main}\,\ref{assumption:main:bounded}, which stipulates that the norm of the gradient of $h$ is bounded by $R$, and of the elementary estimate~$24 G \sqrt{s / T} > G\geq R$, which holds for all $s > T / 4$.

	    In the following, we restrict attention to those deviations $s\ge 0$ that are small in the sense that
		\begin{align}
		    \label{eq:condition:s}
		    \DS 12 \sqrt{s} + 40 s \leq \frac{ \kappa \sqrt{T}}{4G^2}.
		\end{align}
		Assume now for the sake of argument that the event inside the probability in the simplified concentration inequality does {\em not} occur, that is, assume that
		\begin{align}
		    \label{eq:inverse-event}
		    \left\| \nabla h \left(\frac{1}{T} \sum_{t=1}^{T} \bm \phi_{t-1} \right) \right\| 
    	    < \frac{2G}{\sqrt{T}} \left( 12 \sqrt{s} + 40 s + 4 G^2 \| \bm \phi_0 - \bm \phi^\star \|^2 + 6 G \| \bm \phi_0 - \bm \phi^\star \| \right).
		\end{align}
		By~\eqref{eq:condition:s} and the assumption of Case~I, \eqref{eq:inverse-event} implies that $\| \nabla h ( \frac{1}{T}\sum_{t=1}^T \bm \phi_{t-1} ) \| < 3 \kappa / (4G) < 3 \kappa / (4M)$. Hence, we may apply Lemma~\ref{lemma:concordance}\,\ref{lemma:stong:convexity} to conclude that $h ( \frac{1}{T}\sum_{t=1}^T \bm \phi_{t-1} ) - h(\bm \phi^\star) \leq \frac{2}{\kappa} \| \nabla h ( \frac{1}{T} \sum_{t=1}^T \bm \phi_{t-1} ) \|^2$. Combining this inequality with~\eqref{eq:inverse-event} then yields
		\begin{align}
		\label{eq:inverse-event-implication}
		    h \left( \frac{1}{T}\sum_{t=1}^T \bm \phi_{t-1} \right) - h(\bm \phi^\star) < \frac{8G^2}{\kappa T} \left( 12 \sqrt{s} + 40 s + 4 G^2 \| \bm \phi_0 - \bm \phi^\star \|^2 + 6 G \| \bm \phi_0 - \bm \phi^\star \| \right)^2.
		\end{align}
		By the simplified concentration inequality derived above, we may thus conclude that
		\begin{align}
		\nonumber
		    4 \exp(-s) \geq \; &\mathbb P \left[ \; \left\| \nabla h \left(\frac{1}{T} \sum_{t=1}^{T} \bm \phi_{t-1} \right) \right\| 
    	    \geq \frac{2G}{\sqrt{T}} \left( 12 \sqrt{s} + 40 s + 4 G^2 \| \bm \phi_0 - \bm \phi^\star \|^2 + 6 G \| \bm \phi_0 - \bm \phi^\star \| \right) \right] \\
    	    \geq \; & \mathbb P \left[ \; h \left(\frac{1}{T} \sum_{t=1}^{T} \bm \phi_{t-1} \right) - h(\bm \phi^\star) \geq \frac{8G^2}{\kappa T} \left( 12 \sqrt{s} + 40 s + 4 G^2 \| \bm \phi_0 - \bm \phi^\star \|^2 + 6 G \| \bm \phi_0 - \bm \phi^\star \| \right)^2 \right]
    	    \label{eq:small:deviation}
		\end{align}
		for any $s\ge 0$ that satisfies~\eqref{eq:condition:s}, where the second inequality holds because~\eqref{eq:inverse-event} implies~\eqref{eq:inverse-event-implication} or, equivalently, because the negation of~\eqref{eq:inverse-event-implication} implies the negation of~\eqref{eq:inverse-event}. The resulting concentration inequality~\eqref{eq:small:deviation} now enables us to construct an upper bound on the second integral in~\eqref{eq:three-integrals}. To this end, we define the function
		$$ \ell(s) = \frac{8 G^2}{\kappa T} \left(12 \sqrt{s} + 40 s + 4 G^2 \| \bm \phi_0 - \bm \phi^\star \|^2 + 6 G \| \bm \phi_0 - \bm \phi^\star \|\right)^2 $$
		for all $s\ge 0$, and set $\bar s = ((9/400 + \kappa \sqrt{T} / (160 G^2))^{\frac{1}{2}} - 3 / 20)^{2}$. Note that $s\ge 0$ satisfies the inequality~\eqref{eq:condition:s} if and only if $s\le\bar s$ and that $\ell(0) = u_1$ as well as $\ell(\bar s) = u_2$. By substituting $u$ with $ \ell(s)$ and using the concentration inequality~\eqref{eq:small:deviation} to bound the integrand, we find that the second integral in~\eqref{eq:three-integrals} satisfies
		\begin{align*}
		    \int_{u_1}^{u_2} \mathbb P \left[ h \left(\frac{1}{T} \sum_{t=1}^{T} \bm \phi_{t-1} \right) - h(\bm \phi^\star) \geq u \right] \diff u 
		    &= \int_{0}^{\bar s} \mathbb P \left[ h \left(\frac{1}{T} \sum_{t=1}^{T} \bm \phi_{t-1} \right) - h(\bm \phi^\star) \geq \ell(s) \right] \frac{\diff \ell(s)}{\diff s} \diff s \\
		    &\leq \int_{0}^{\bar s} 4 \mathrm{e}^{-s} \; \frac{\diff}{\diff s} \! \left( \frac{8 G^2}{\kappa T} \left(12 \sqrt{s} + 40 s + \tau \right)^2 \right) \diff s \\
		    &\leq \frac{32 G^2}{\kappa T} \int_{0}^{\infty} \mathrm{e}^{-s} \left( 144 + 3200 s + 1440 s^{1/2} + 80 \tau + 12 \tau s^{-1/2} \right) \diff s \\
		    &= \frac{32 G^2}{\kappa T} \big( 144 + 3200 \Gamma(2) + 1440 \Gamma(3/2) + 80 \tau + 12 \tau \Gamma(1/2) \big) \\
		    &\leq \frac{32 G^2}{\kappa T} ( 4621 + 102 \tau ),
		\end{align*}
		where $\tau$ is is a shorthand for $4 G^2 \| \bm \phi_0 - \bm \phi^\star \|^2 + 6 G \| \bm \phi_0 - \bm \phi^\star \| $, and $\Gamma$ denotes the Gamma function with $\Gamma(2) = 1$, $\Gamma(1/2) = \sqrt{\pi}$ and $\Gamma(3/2) = \sqrt{\pi}/2$; see for example~\citep[Chapter~8]{rudin1964principles}. The last inequality is obtained by rounding all fractional numbers up to the next higher integer. Combining the upper bounds for the three integrals in~\eqref{eq:three-integrals} finally yields
		\begin{align*}
		    \EE \left[ h \left(\frac{1}{T} \sum_{t=1}^{T} \bm \phi_{t-1} \right) - h(\bm \phi^\star) \right] 
		    &\leq \frac{8 G^2}{\kappa T} \left( \tau^2 + 18484 + 408 \tau + 300 \right) \\
		    &= \frac{8 G^2}{\kappa T} \Big( 16 G^4 \| \bm \phi_0 - \bm \phi^\star \|^4 + 48 G^3 \| \bm \phi_0 - \bm \phi^\star \|^3 + 1668 G^2 \| \bm \phi_0 - \bm \phi^\star \|^2 \\
		    &\hspace{4em} + 2448 G \| \bm \phi_0 - \bm \phi^\star \| + 18784 \Big) \\
		    &\leq \frac{G^2}{\kappa T} (4 G \| \bm \phi_0 - \bm \phi^\star \| + 20)^4.
		\end{align*}
		This complete the proof of assertion~\ref{theorem:convergence:concordance} in Case~I.
		
		\textbf{Case II:} Assume now that $4 G^2 \| \bm \phi_0 - \bm \phi^\star \|^2 + 6 G \| \bm \phi_0 - \bm \phi^\star \| > {\kappa \sqrt{T}}/{(8 G^2)}$, where $G$ is defined as before.
		Since $h$ has bounded gradients, the inequality~\eqref{eq:main:convergence} remains valid. Setting the step size to $\gamma = 1 / (2 G^2 \sqrt{T})$ and using the trivial inequalities $G \geq R + \bar \varepsilon \geq R$, we thus obtain
		\begin{align*}
			\EE \left[ h \left(\frac{1}{T} \sum_{t=1}^{T} \bm \phi_{t-1} \right) \right] - h(\bm \phi^\star)
			&\leq \frac{G^2}{\sqrt{T}} \| \bm \phi_0 - \bm \phi^\star \|^2 + \frac{1}{4\sqrt{T}} + \frac{\bar \varepsilon}{\sqrt{T}} \sqrt{2 \| \bm \phi_0 - \bm \phi^\star \|^2 + \frac{37}{2G^2}} \\
			&\leq \frac{G^2}{\sqrt{T}} \| \bm \phi_0 - \bm \phi^\star \|^2 + \frac{2G}{\sqrt{T}} \| \bm \phi_0 - \bm \phi^\star \| + \frac{5}{\sqrt{T}} ,
		\end{align*}
		where the second inequality holds because $G \geq \bar \varepsilon$ and $\sqrt{a + b} \leq \sqrt{a} + \sqrt{b}$ for all $a,b\ge 0$.
		Multiplying the right hand side of the last inequality by $G^2 (32 G^2 \| \bm \phi_0^\star - \bm \phi^\star \|^2 + 48 G \| \bm \phi_0^\star - \bm \phi^\star \|) / (\kappa \sqrt{T})$, which is strictly larger than~$1$ by the basic assumption underlying Case~II, we then find
		\begin{align*}
		    & \phantom{\leq} \EE \left[ h \left(\frac{1}{T} \sum_{t=1}^{T} \bm \phi_{t-1} \right) \right] - h(\bm \phi^\star) \\
		    &\leq \frac{G^2}{\kappa T}
		    \left( G^2 \| \bm \phi_0 - \bm \phi^\star \|^2 + 2 G \| \bm \phi_0 - \bm \phi^\star \| + 5 \right) \left( 32 G^2 \| \bm \phi_0^\star - \bm \phi^\star \|^2 +
		    48 G \| \bm \phi_0^\star - \bm \phi^\star \| \right) \\
		    &\leq \frac{G^2}{\kappa T} (4 G \| \bm \phi_0 - \bm \phi^\star \| + 20)^4.
		\end{align*}
		This observation completes the proof.
	\end{proof}

	\subsection{Smooth Optimal Transport Problems with Marginal Ambiguity Sets}
	\label{section:ASGD-OT}
	The smooth optimal transport problem \eqref{eq:smooth_ot} can be viewed as an instance of a stochastic optimization problem, that is, a convex maximization problem akin to~\eqref{eq:convex:problem}, where the objective function is representable as $h(\bm \phi) = \mathbb{E}_{\bs x \sim \mu} [ \bs \nu^\top \bs \phi - \overline\psi_c(\bs \phi, \bs x)]$. Throughout this section we assume that the smooth (discrete) $c$-transform~$\overline\psi_c(\bs \phi, \bs x)$ defined in~\eqref{eq:smooth_c_transform} is induced by a marginal ambiguity set of the form~\eqref{eq:marginal_ambiguity_set} with continuous marginal distribution functions. By Proposition~\ref{proposition:regularized_ctrans}, the integrand $\bs \nu^\top \bs \phi - \overline\psi_c(\bs \phi, \bs x)$ is therefore concave and differentiable in~$\bm \phi$. We also assume that $\overline\psi_c(\bs \phi, \bs x)$ is $\mu$-integrable in~$\bm x$, that we have access to an oracle that generates independent samples from~$\mu$ and that problem~\eqref{eq:smooth_ot} is solvable. 
	
	The following proposition establishes several useful properties of the smooth $c$-transform.
	
	\begin{proposition}[Properties of the smooth $c$-transform]
		\label{proposition:structural}
		If $\Theta$ is a marginal ambiguity set of the form~\eqref{eq:marginal_ambiguity_set} with cumulative distribution functions $F_i$, $i\in[N]$, then $\overline\psi_c(\bs \phi, \bs x)$ has the following properties for all $\bs x \in \mathcal X$.
		\begin{enumerate} [label=(\roman*)]
			\item \textbf{Bounded gradient:} \label{proposition:gradient} If $F_i$, $i\in[N]$, are continuous, then we have $ \| \nabla_{\bs \phi} \overline\psi_c(\bs \phi, \bs x) \| \leq 1 $ for all $\bs \phi\in\R^N$.
			\item \textbf{Lipschitz continuous gradient:} \label{proposition:smooth} If $F_i$, $i\in[N]$, are Lipschitz continuous with Lipschitz constant~$L>0$, then $\overline\psi_c(\bs \phi, \bs x)$ is $L$-smooth with respect to $\bs \phi$ in the sense of Assumption~\ref{assumption:main}\,\ref{assumption:main:smooth}.
			\item \textbf{Generalized self-concordance:} \label{proposition:concordance} If $F_i$, $i\in[N]$, are twice differentiable on the interiors of their respective supports and if there is $M > 0$ with
			\begin{equation}
			    \label{eq:cumulative:concoradance}
			    \sup_{s \in F_i^{-1}(0,1)} ~ \frac{|\diff^2F_i(s) / \diff s^2|}{\diff F_i(s) / \diff s} \leq M,
			\end{equation}
			then $\overline\psi_c(\bs \phi, \bs x)$ is $M$-generalized self-concordant with respect to $\bs \phi$ in the sense of Assumption~\ref{assumption:main}\,\ref{assumption:main:concordance}.
		\end{enumerate}
	\end{proposition}
	\begin{proof}
		As for~\ref{proposition:gradient},  Proposition~\ref{proposition:regularized_ctrans} implies that $\nabla_{\bs \phi} \overline\psi_c(\bs \phi, \bs x) \in \Delta^N$, and thus we have $\| \nabla_{\bs \phi} \overline\psi_c(\bs \phi, \bs x) \| \leq 1$. As for~\ref{proposition:smooth}, note that the convex conjugate of the smooth $c$-transform with respect to $\bs \phi$ is given by
		\begin{align*}
		\overline\psi{}_c^*(\bs p, \bs x) &= \sup_{\bs \phi \in \R^N} \bs p^\top \bs \phi - \overline\psi(\bs \phi, \bs x) = \sup_{\bs \phi \in \R^N} \inf_{\bs q \in \Delta^N} ~ \sum_{i=1}^N p_i \phi_i - (\phi_i - c(\bs x, \bs {y_i})) q_i - \int_{1-q_i}^1 F_i^{-1}(t)\diff t \\
		&= \inf_{\bs q \in \Delta^N} \sup_{\bs \phi \in \R^N} ~ \sum_{i=1}^N p_i \phi_i - (\phi_i - c(\bs x, \bs {y_i})) q_i - \int_{1-q_i}^1 F_i^{-1}(t)\diff t \\
		&= \begin{cases}
		\;\DS \sum\limits_{i=1}^N c(\bs x, \bs {y_i}) p_i - \int_{1-p_i}^1 F_i^{-1}(t)\diff t &  \text{if } \bs p \in \Delta^N \\
		\;+\infty &  \text{otherwise,}
		\end{cases}
		\end{align*}
		where the second equality follows again from Proposition~\ref{proposition:regularized_ctrans}, and the interchange of the infimum and the supremum is allowed by Sion's classical minimax theorem. 
		In the following we first prove that $\overline\psi{}_c^*(\bs p, \bs x)$ is $1/L$-strongly convex in $\bs p$, that is, the function $\overline\psi{}_c^*(\bs p, \bs x) -  \| \bs p\|^2/ (2L)$ is convex in $\bs p$ for any fixed $\bs x \in \mc X$.
		To this end, recall that $F_i$ is assumed to be Lipschitz continuous with Lipschitz constant $L$. Thus, we have
		\begin{align*}
		    L\ge  \sup_{\substack{s_1,s_2 \in \R\\ s_1 \neq s_2}}\frac{\left| F_i (s_1) - F_i(s_2)\right|}{|s_1 - s_2|} 
	    	= \sup_{\substack{s_1, s_2 \in \R\\ s_1 > s_2}}\frac{ F_i (s_1) - F_i(s_2)}{s_1 - s_2}
		    \geq \sup_{\substack{p_i, q_i \in (0,1)\\ p_i > q_i}} \frac{p_i - q_i}{F_i^{-1}(p_i) - F_i^{-1}(q_i)},
		\end{align*}
		where the second inequality follows from restricting $s_1$ and $s_2$ to the preimage of $(0,1)$ with respect to~$F_i$. 
		Rearranging terms in the above inequality then yields
		\begin{align*} 
		    -F_i^{-1}(1 - q_i) - q_i/L &\leq  -F_i^{-1}(1-p_i)-p_i/L
		\end{align*}
		for all $p_i, q_i \in (0, 1)$ with $q_i < p_i$. Consequently, the function $- F_i^{-1}(1-p_i) - {p_i}/L$ is non-decreasing and its primitive $- \int_{1-p_i}^1 F_i^{-1}(t)\diff t - p_i^2/(2 L)$ is convex in $p_i$ on the interval $(0,1)$. This implies that
		$$ \overline\psi{}_c^*(\bs p, \bs x) -  \frac{\| \bs p\|_2^2}{2 L} = \sum_{i=1}^N c(\bs x, \bs {y_i}) p_i - \int_{1-p_i}^1 F_i^{-1}(t)\diff t - \frac{p_i^2}{2 L}$$
		constitutes a sum of convex univariate functions for every fixed $\bs x\in \x$. Thus, $\overline\psi{}_c^*(\bs p, \bs x)$ is $1/L$-strongly convex in $\bs p$. By \citep[Theorem~6]{kakade2009duality}, however, any convex function whose conjugate is $1/L$-strongly convex is guaranteed to be $L$-smooth. This observation completes the proof of assertion~\ref{proposition:smooth}. As for assertion~\ref{proposition:concordance}, choose any $\bs \phi, \bs \varphi \in \R^N$ and $\bs x \in \mathcal X$, and introduce the auxiliary function
		\begin{equation}
		\label{eq:u:function}
		    u(s) = \overline \psi_c \left(\bs \phi + s (\bs \varphi - \bs \phi), \bs x \right) = \max_{ \bs p \in \Delta^N} \displaystyle \sum\limits_{i=1}^N ~ (\phi_i + s (\varphi_i - \phi_i) - c(\bs x, \bs {y_i}))p_i + \int_{1-p_i}^1 F_i^{-1}(t) \diff t.
		\end{equation}
		For ease of exposition, in the remainder of the proof we use prime symbols to designate derivatives of univariate functions. A direct calculation then yields
		\begin{align*}
		    u'(s) = \left( \bs \varphi - \bs \phi \right)^\top \nabla_{\bs \phi} \overline \psi \left(\bs \phi + s (\bs \varphi - \bs \phi), \bs x \right) \quad \text{and} \quad u''(s) = \left( \bs \varphi - \bs \phi \right)^\top \nabla_{\bs \phi}^2 \overline \psi \left(\bs \phi + s (\bs \varphi - \bs \phi), \bs x \right) \left( \bs \varphi - \bs \phi \right).
		\end{align*}
		By Proposition~\ref{proposition:regularized_ctrans}, $\bs p^\star(s)=\nabla_{\bs \phi} \overline \psi_c \left(\bs \phi + s (\bs \varphi - \bs \phi), \bs x \right)$ represents the unique solution of the maximization problem in~\eqref{eq:u:function}. In addition, by~\citep[Proposition~6]{sun2019generalized}, the Hessian of the smooth $c$-transform with respect to $\bs \phi$ can be computed from the Hessian of its convex conjugate as follows.
		$$ \nabla_{\bs \phi}^2 \overline \psi_c \left(\bs \phi + s (\bs \varphi - \bs \phi), \bs x \right) = \left( \nabla^2_{\bs p} \overline \psi{}_c^*(\bs p^\star(s), \bs x) \right)^{-1} = \mathrm{diag} \left( [F_1'(F_1^{-1}(1 - p_1^\star(s))), \dots, F_N'(F_N^{-1}(1 - p_N^\star(s))) ] \right)$$
		Hence, the first two derivatives of the auxiliary function $u(s)$ simplify to
		$$ u'(s) = \sum_{i=1}^N (\varphi_i- \phi_i) p^\star_i(s) \quad \text{and} \quad u''(s) = \sum_{i=1}^N (\varphi_i- \phi_i)^2 F_i'(F_i^{-1}(1 - p_i^\star(s))).$$
		Similarly, the above formula for the Hessian of the smooth $c$-transform can be used to show that $(p_i^\star)'(s)  = (\varphi_i- \phi_i) F_i'(F_i^{-1}(1 - p_i^\star(s)))$ for all $i \in [N]$. The third derivative of $u(s)$ therefore simplifies to
		\begin{align*}
		    u'''(s) = - \sum_{i=1}^N (\varphi_i- \phi_i)^2 \,\frac{ F_i''(F_i^{-1}(1 - p_i^\star(s)))}{F_i'(F_i^{-1}(1 - p_i^\star(s)))}\, (p_i^\star)'(s) = - \sum_{i=1}^N (\varphi_i- \phi_i)^3 F_i''(F_i^{-1}(1 - p_i^\star(s))).
		\end{align*}
		This implies via H\"{o}lder's inequality that
		\begin{align*}
		| u'''(s) | 
		&= \left| \sum_{i=1}^N  (\varphi_i- \phi_i)^2\, F_i'(F_i^{-1}(1 - p_i^\star(s))) \, \frac{F_i''(F_i^{-1}(1 - p_i^\star(s)))}{F_i'(F_i^{-1}(1 - p_i^\star(s)))} \, (\varphi_i- \phi_i) \right| \\
		&\leq \left( \sum_{i=1}^N (\varphi_i- \phi_i)^2\, F_i'(F_i^{-1}(1 - p_i^\star(s))) \right) \left( \max_{i \in [N]} \left| \frac{F_i''(F_i^{-1}(1 - p_i^\star(s)))}{F_i'(F_i^{-1}(1 - p_i^\star(s)))} \, (\varphi_i- \phi_i) \right| \right).
		\end{align*}
		Notice that the first term in the above expression coincides with $u''(s)$, and the second term satisfies
		\begin{align*}
		\max_{i \in [N]} \left| \frac{F_i''(F_i^{-1}(1 - p_i^\star(s)))}{F_i'(F_i^{-1}(1 - p_i^\star(s)))} \, (\varphi_i- \phi_i) \right| 
		\leq \max_{i \in [N]} \left| \frac{F_i''(F_i^{-1}(1 - p_i^\star(s)))}{F_i'(F_i^{-1}(1 - p_i^\star(s)))} \right| \, \| \bs \varphi - \bs \phi \|_\infty \leq M \| \bs \varphi \bs - \bs \phi \|,
		\end{align*}
		where the first inequality holds because $\max_{i \in [N]} |a_i b_i| \leq \| \bs a \|_{\infty} \| \bs b \|_\infty $ for all $\bs a, \bs b \in \mathbb R^N$, and the second inequality follows from the definition of $M$ and the fact that the 2-norm provides an upper bound on the $\infty$-norm. Combining the above results shows that $|u'''(s)|\leq M  \| \bs \varphi \bs - \bs \phi \| u''(s)$ for all $s\in\R$. The claim now follows because $\bs \phi, \bs \varphi \in \R^N$ and $\bs x \in \mathcal X$ were chosen arbitrarily. 
	\end{proof}
	
	\begin{table}[!b]
		\centering
		\begin{minipage}{0.375\textwidth}
			\vspace{-1em}
			\begin{algorithm}[H]
				\caption{\label{algorithm:asgd}Averaged SGD \protect\phantom{$\nabla_{\bs \phi} \overline\psi$}}
				\begin{algorithmic}[1]
					\Require $\gamma, T, \bar \varepsilon$
					\vspace{0.05em}
					\State Set $\bs \phi_0 \gets \bs 0$
					\For{$t = 1, 2,\dots, T$}
					\State \hspace{-1ex}Sample $\bs {x_t}$ from $\mu$
					\State \hspace{-1ex}Choose $\varepsilon_{t-1} \in (0, \bar \varepsilon / (2 \sqrt{t})]$
					\State \hspace{-1ex}Set $ \bs p \gets \text{Bisection}(\bs {x_t}, \bs \phi_{t-1}, \varepsilon_{t-1})$
					\State \hspace{-1ex}Set $\bs \phi_t \leftarrow \bs \phi_{t-1} + \gamma (\bs \nu - \bs p)$
					\EndFor
					\vspace{0.05em}
					\Ensure $\ubar{\bs \phi}_T = \frac{1}{T}\sum_{t=1}^T \bs \phi_{t-1}$~ and \phantom{abcd} \phantom{abcd}\hspace{0.85ex} $\bar{\bs \phi}_T = \frac{1}{T}\sum_{t=1}^T \bs \phi_{t}$
				\end{algorithmic}
			\end{algorithm}
			\vspace{-1em}
		\end{minipage} 
		\begin{minipage}{0.6175\textwidth}
		    \vspace{-1em}
			\begin{algorithm}[H]				\caption{\label{algorithm:bisection}Bisection method to approximate $\nabla_{\bs \phi} \overline\psi_c(\bs \phi, \bs x)$}
				\begin{algorithmic}[1]
					\Require $\bs x, \bs \phi, \varepsilon$ \vspace{0.1em}
					\State Set $\overline{\tau} \gets \max_{i \in [N]} ~ \{c(\bs x,\bs {y_i}) - \phi_i - F_i^{-1}(1-1/N) \}$
					\State Set $\underline{\tau} \gets \min_{i \in [N]} ~ \{c(\bs x,\bs {y_i}) - \phi_i - F_i^{-1}(1-1/N) \}$
					\State Evaluate $\delta(\varepsilon)$ as defined in~\eqref{eq:delta(eps)}
					\For{$k = 1, 2,\dots, \lceil \log_2 ((\overline{\tau} - \underline{\tau}) / \delta(\varepsilon)) \rceil$}
					\State \hspace{-1ex}Set $\tau \gets {(\overline{\tau} + \underline{\tau})}/{2}$
					\State \hspace{-1ex}Set $p_i \gets 1-F_i(c(\bs x,\bs {y_i}) -\phi_i -\tau)$ for $i \in [N]$
					\State \hspace{-1ex}\algorithmicif~~$\sum_{i \in [N]} p_i > 1$~~\algorithmicthen~~$\overline{\tau} \gets \tau$~~\algorithmicelse~~$\underline{\tau} \gets \tau$
					\EndFor	\vspace{0.1em}
					\Ensure $\bs p $ with $p_i = 1-F_i(c(\bs x,\bs {y_i}) -\phi_i -\underline \tau)$, $i\in[N]$
				\end{algorithmic}
			\end{algorithm}
			\vspace{-1em}
		\end{minipage}
	\end{table}
	
	In the following we use the averaged SGD algorithm of Section~\ref{section:AGD} to solve the smooth optimal transport problem~\eqref{eq:smooth_ot}. A detailed description of this algorithm in pseudocode is provided in Algorithm~\ref{algorithm:asgd}. This algorithm repeatedly calls a sub-routine for estimating the gradient of $\overline\psi_c(\bs \phi, \bs x)$ with respect to $\bm \phi$. By Proposition~\ref{proposition:regularized_ctrans}, this gradient coincides with the unique solution~$\bs p^\star$ of the convex maximization problem~\eqref{eq:regularized_c_transform}. In addition, from the proof of Proposition~\ref{proposition:regularized_ctrans} it is clear that its components are given by 
	\begin{align*}
	    p^\star_i = \theta^\star \left[ i = \min \argmax_{j \in [N]} \phi_j - c(\bs x, \bs y_j) + z_j \right] \quad \forall i \in [N],
	\end{align*}
	where $\theta^\star$ represents an optimizer of the semi-parametric discrete choice problem~\eqref{eq:smooth_c_transform}. Therefore, $\bs p^\star$ can be interpreted as a vector of choice probabilities under the best-case probability measure~$\theta^\star$. Sometimes these choice probabilities are available in closed form. This is the case, for instance, in the exponential distribution model of Example~\ref{ex:exp}, which is equivalent to the generalized extreme value distribution model of Section~\ref{sec:gevm}. Indeed, in this case $\bs p^\star$ is given by a softmax of the utility values $\phi_i - c(\bs x, \bs {y_i})$, $i\in[N]$, {\em i.e.},
	\begin{equation}\label{eq:softmax}
	    p_i^\star= \frac{\eta_i \exp\left(({\phi_i - c(\bs x, \bs {y_i}) )}/{\lambda}\right)}{\sum_{j=1}^N \eta_j \exp\left(({\phi_j - c(\bs x,\bs {y_j}) })/{\lambda}  \right)} \quad \forall i \in [N].
	\end{equation}
	Note that these particular choice probabilities are routinely studied in the celebrated multinomial logit choice model~\citep[\S~5.1]{ben1985discrete}.
	The choice probabilities are also available in closed form in the uniform distribution model of Example~\ref{ex:uniform}. As the derivation of $\bs p^\star$ is somewhat cumbersome in this case, we relegate it to Appendix~\ref{appendix:spmax}.
	For general marginal ambiguity sets with continuous marginal distribution functions, we propose a bisection method to compute the gradient of the smooth $c$-transform numerically up to any prescribed accuracy; see Algorithm~\ref{algorithm:bisection}.

	\begin{theorem}[Biased gradient oracle]
		\label{theorem:bisection}
		If $\Theta$ is a marginal ambiguity set of the form~\eqref{eq:marginal_ambiguity_set} and the cumulative distribution function $F_i$ is continuous for every $i\in[N]$, then, for any $\bs x \in \mathcal X$, $\bs \phi \in \R^N$ and $\varepsilon > 0$, Algorithm~\ref{algorithm:bisection} outputs $\bs p \in \R^N $ with $\| \bs p \| \leq 1$ and $\| \nabla_{\bs \phi} \overline\psi(\bs \phi, \bs x) - {\bs p} \| \leq \varepsilon$.
	\end{theorem}
	\begin{proof}
		Thanks to Proposition~\ref{proposition:regularized_ctrans}, we can recast the smooth $c$-transform in dual form as
		\begin{equation*}
		    \overline \psi_c(\bs \phi, \bs x) = 
		    \min_{\substack{\bs \zeta \in \R_+^N \\ \tau \in \R}}\;\sup_{\bs p \in \R^N} ~ \sum\limits_{i=1}^N (\phi_i - c(\bs {x},\bs {y_i}))p_i +\sum\limits_{i=1}^N \int^1_{1- p_i} F_i^{-1}(t)\diff t + \tau \left(\sum\limits_{i=1}^N p_i - 1 \right)+ \sum\limits_{i=1}^N \zeta_i p_i.
		\end{equation*}
		Strong duality and dual solvability hold because we may construct a Slater point for the primal problem by setting $p_i=1/N$, $i\in[N]$. By the Karush-Kuhn-Tucker optimality conditions, $\bs p^\star$ and $(\tau^\star,\bs \zeta^\star)$ are therefore optimal in the primal and dual problems, respectively, if and only if we have
		\begin{align*}
		    \begin{array}{lll}
		    \sum_{i=1}^N p^\star_i =1, ~p^\star_i \geq 0 & \forall i \in [N] & \text{(primal feasibility)}\\
		    \zeta^\star_i\geq 0 & \forall i \in [N] & \text{(dual feasibility)}\\ \zeta_i^\star p_i^\star=0 & \forall i \in [N] & \text{(complementary slackness)} \\
		    \phi_i-c(\bs {x},\bs {y_i}) + F_i^{-1}(1-p^\star_i) + \tau^\star + \zeta^\star_i = 0 & \forall i \in [N] & \text{(stationarity)}.
		    \end{array}
		\end{align*}
		If $p_i^\star > 0$, then the complementary slackness and stationarity conditions imply that $\zeta_i^\star = 0$ and that $\phi_i-c(\bs {x},\bs {y_i}) + F_i^{-1}(1-p^\star_i) + \tau^\star = 0$, respectively. 
		Thus, we have $p_i^\star = 1 - F_i(c(\bs {x},\bs {y_i})-\phi_i -\tau^\star)$.
		If $p_i^\star = 0$, on the other hand, then similar arguments show that $\zeta_i^\star \geq 0$ and $\phi_i-c(\bs {x},\bs {y_i}) + F_i^{-1}(1) + \tau^\star \leq 0$. These two inequalities are equivalent to $1 - F_i(c(\bs {x},\bs {y_i})-\phi_i -\tau^\star) \leq 0$. As all values of~$F_i$ are smaller or equal to~$1$, the last equality must in in fact hold as an equality. Combining the insights gained so far thus yields $p_i^\star = 1 - F_i(c(\bs {x},\bs {y_i})-\phi_i -\tau^\star)$, which holds for all~$i\in[N]$ irrespective of the sign of $p_i^\star$. Primal feasibility therefore ensures that $\sum_{i=1}^N 1 - F_i(c(\bs {x},\bs {y_i})-\phi_i -\tau^\star) = 1$. Finding the unique optimizer $\bs p^\star$ of~\eqref{eq:regularized_c_transform} ({\em i.e.}, finding the gradient of $ \overline \psi_c(\bs \phi, \bs x)$) is therefore tantamount to finding a root $\tau^\star$ of the univariate equation 
		\begin{align}
		    \label{eq:root-finding}
		    \sum_{i=1}^N 1 - F_i(c(\bs {x},\bs {y_i})-\phi_i -\tau) = 1.
		\end{align}
		Note the function on the left hand side of~\eqref{eq:root-finding} is continuous and non-decreasing in $\tau$ because of the continuity (by assumption) and monotonicity (by definition) of the cumulative distribution functions $F_i$, $i\in[N]$.
		Hence, the root finding problem can be solved efficiently via bisection.
		To complete the proof, we first show that the interval between the constants $\underline{\tau}$ and $\overline{\tau}$ defined in Algorithm~\ref{algorithm:bisection} is guaranteed to contain~$\tau^\star$. Specifically, we will demonstrate that evaluating the function on the left hand side of~\eqref{eq:root-finding} at $\underline\tau$ or $\overline \tau$ yields a number that is not larger or not smaller than~1, respectively. 
		For $\tau=\underline\tau$ we have
		\begin{align*}
		    1 - F_i(c(\bs {x},\bs {y_i})-\phi_i -\underline\tau) &= 1 - F_i \left( c(\bs {x},\bs {y_i})-\phi_i - \min_{j \in [N]} \left\{ c \left( \bs {x}, \bs {y_j} \right) - \phi_j -F_j^{-1}(1-1/N) \right\} \right) \\
		    &\leq 1 - F_i \left( F_i^{-1}(1-1/N) \right) = 1 / N\qquad \forall i\in[N],
		\end{align*}
		where the inequality follows from the monotonicity of $F_i$. Summing the above inequality over all $i\in[N]$ then yields the desired inequality $\sum_{i =1}^N 1 - F_i(c(\bs {x},\bs {y_i})-\phi_i -\underline\tau) \leq 1$. Similarly, for $\tau=\overline\tau$ we have
		\begin{align*}
		    1 - F_i(c(\bs {x},\bs {y_i})-\phi_i -\overline\tau) &= 1 - F_i \left( c(\bs {x},\bs {y_i})-\phi_i - \max_{j \in [N]} \left\{ c \left( \bs {x}, \bs {y_j} \right) - \phi_j -F_j^{-1}(1-1/N) \right\} \right) \\
		    &\geq 1 - F_i \left( F_i^{-1}(1-1/N) \right) = 1/N \qquad \forall i\in[N].
		\end{align*}
		We may thus conclude that $\sum_{i =1}^N 1 - F_i(c(\bs {x},\bs {y_i})-\phi_i -\overline\tau) \geq 1$. Therefore, $[\underline \tau, \overline \tau]$ constitutes a valid initial search interval for the bisection algorithm.
		Note that the function $1 - F_i(c(\bs {x},\bs {y_i})-\phi_i -\tau)$, which defines $p_i$ in terms of $\tau$, is uniformly continuous in~$\tau$ throughout~$\mathbb R$. This follows from \citep[Problem~14.8]{billingsley1995} and our assumption that~$F_i$ is continuous. The uniform continuity ensures that the tolerance
		\begin{align}
		    \label{eq:delta(eps)}
		    \delta(\varepsilon) = \min_{i \in N} \left\{ \max_\delta \left\{ \delta : | F_i(t_1) - F_i(t_2) | \leq \varepsilon / \sqrt{N} ~~ \forall t_1,t_2\in\R \text{ with } | t_1 - t_2 | \leq \delta \right\} \right\}
		\end{align}
		is strictly positive for every~$\varepsilon>0$.
		As the length of the search interval is halved in each iteration, Algorithm~\ref{algorithm:bisection} outputs a near optimal solution $\tau$ with
		$| \tau - \tau^\star | \leq \delta(\varepsilon)$ after $\lceil \log_2 ((\overline{\tau} - \underline{\tau}) / \delta(\varepsilon)) \rceil$ iterations.
		Moreover, the construction of~$\delta(\varepsilon)$ guarantees that $|1 - F_i(c(\bs {x},\bs {y_i})-\phi_i -\tau) - p_i^\star| \leq \varepsilon / \sqrt{N}$ for all $\tau$ with~$|\tau - \tau^\star| \leq \delta(\varepsilon)$. Therefore, the output~$\bs p\in\R^N_+$ of Algorithm~\ref{algorithm:bisection} satisfies $|p_i - p_i^\star| \leq \varepsilon / \sqrt{N} $ for each~$i\in[N]$, which in turn implies that $ \| \bs p - \bs p^\star \| \leq \varepsilon$. By construction, finally, Algorithm~\ref{algorithm:bisection} outputs $\bs p\ge \bs 0$ with $\sum_{i \in [N]} p_i < 1$, which ensures that $\| p \| \leq 1$. Thus, the claim follows.
	\end{proof}
	If all cumulative distribution functions $F_i$, $i\in[N]$, are Lipschitz continuous with a common Lipschitz constant~$L>0$, then the uniform continuity parameter~$\delta(\varepsilon)$ required in Algorithm~\ref{algorithm:bisection} can simply be set to~$\delta(\varepsilon) = \varepsilon / (L \sqrt{N})$.
	We are now ready to prove that Algorithm~\ref{algorithm:asgd} offers different convergence guarantees depending on the continuity and smoothness properties of the marginal cumulative distribution functions. 
	
	\begin{corollary}
	\label{corollary:convergence}
	    Use $h(\bm \phi) = \mathbb{E}_{\bs x \sim \mu} [ \bs \nu^\top \bs \phi - \overline\psi_c(\bs \phi, \bs x)]$ as a shorthand for the objective function of the smooth optimal transport problem~\eqref{eq:smooth_ot}, and let $\bs \phi^\star$ be a maximizer of~\eqref{eq:smooth_ot}.
	    If $\Theta$ is a marginal ambiguity set of the form~\eqref{eq:marginal_ambiguity_set} with distribution functions $F_i$, $i\in[N]$, then for any $T \in \mathbb N$ and $\bar \varepsilon\ge 0$, the outputs $\ubar{\bs \phi}_T = \frac{1}{T} \sum_{t=1}^{T} \bs \phi_{t-1}$ and $\bar{\bs \phi}_T = \frac{1}{T} \sum_{t=1}^{T} \bs \phi_{t}$ of Algorithm~\ref{algorithm:asgd} satisfy the following inequalities.
		\begin{enumerate} [label=(\roman*)]
			\item \label{corollary:convergence:Lipschitz} If $\gamma = 1 / (2 (2 + \bar \varepsilon) \sqrt{T})$ and $F_i$ is continuous for every $i\in[N]$, then we have
			\begin{align*}
			\overline W_c (\mu, \nu) - \EE \left[ h \big(\ubar{\bs \phi}_T \big) \right]
			\leq \frac{(2 + \bar \varepsilon)^2}{\sqrt{T}} \| \bm \phi^\star \|^2 + \frac{1}{4\sqrt{T}} + \frac{\bar \varepsilon}{\sqrt{T}} \sqrt{2 \| \bm \phi^\star \|^2 + \frac{37}{2(2 + \bar \varepsilon)^2}}.
			\end{align*}
			\item \label{corollary:convergence:smooth} If $\gamma = 1 / (2 \sqrt{T} + L)$ and $F_i$ is Lipschitz continuous with Lipschitz constant $L>0$ for every $i\in[N]$, then we have
			\begin{align*}
			\overline W_c (\mu, \nu) - \EE \left[ h \big(\bar{\bs \phi}_T \big) \right]
			&\leq \frac{L}{2T}\| \bm \phi^\star \|^2 + 
			\frac{(2 + \bar \varepsilon)^2}{\sqrt{T}} \| \bm \phi^\star \|^2 + \frac{\bar\varepsilon^2 + 2}{4 (2+\bar \varepsilon)^2\sqrt{T}} + \frac{\bar \varepsilon}{\sqrt{T}} \sqrt{2 \| \bm \phi^\star \|^2 + \frac{37}{2(2 + \bar \varepsilon)^2}}.
			\end{align*}
			\item \label{corollary:convergence:concordance} If $\gamma = 1 / (2 G^2 \sqrt{T}) $ with $G = \max \{M, 2 + \bar\varepsilon\}$, $F_i$ satisfies the generalized self-concordance condition~\eqref{eq:cumulative:concoradance} with~$M> 0$ for every $i\in[N]$, and the smallest eigenvalue $\kappa$ of $-\nabla^2_{\bs \phi} h(\bs \phi^\star)$ 
			is strictly positive, then we have
			\begin{align*}
			\overline W_c (\mu, \nu) - \EE \left[ h \big(\ubar{\bs \phi}_T \big) \right]
			&\leq \frac{G^2}{T \kappa} \left( 4 G \| \bm \phi_0 - \bm \phi^\star \| + 20 \right)^4.
			\end{align*}
		\end{enumerate}
	\end{corollary}
	\begin{proof}
	    Recall that problem \eqref{eq:smooth_ot} can be viewed as an instance of the convex minimization problem~\eqref{eq:convex:problem} provided that its objective function is inverted. 
	    Throughout the proof we denote by $\bs p_t(\bs \phi_t, \bs x_t)$ the inexact estimate for $\nabla_{\bs \phi} \overline\psi(\bs \phi_t, \bs x_t)$ output by Algorithm~\ref{algorithm:bisection} in iteration $t$ of the averaged SGD algorithm. Note that
	    \begin{align*}
	        \left\| \EE \left[ \bm \nu - \bm p_t(\bm \phi_{t-1}, \bm x_t) \big| \mathcal F_{t-1} \right] - \nabla h(\bm \phi_{t-1}) \right\| 
    	    &= \left\| \EE \left[ \bm p_t(\bm \phi_{t-1}, \bm x_t) - \nabla_{\bs \phi} \overline\psi_c(\bs \phi_{t-1}, \bs x_t)\right] \right\| \\
    	    &\leq \EE \left[ \left\| \bm p_t(\bm \phi_{t-1}, \bm x_t) - \nabla_{\bs \phi} \overline\psi_c(\bs \phi_{t-1}, \bs x_t) \right\| \right]
    	    \leq \varepsilon_{t-1} \leq \frac{\bar \varepsilon}{2 \sqrt{t}},
	    \end{align*}
	    where the two inequalities follow from Jensen's inequality and the choice of $\varepsilon_{t-1}$ in Algorithm~\ref{algorithm:asgd}, respectively. The triangle inequality and Proposition~\ref{proposition:structural}\,\ref{proposition:gradient} further imply that
	    \begin{align*}
	        \left\| \nabla h(\bm \phi) \right\| 
			= \EE \left[ \left\| \bs \nu - \nabla_{\bs \phi} \overline\psi_c(\bs \phi, \bs x) \right\| \right]
			\leq \left\| \bs \nu \right\| + \EE \left[ \left\| \nabla_{\bs \phi} \overline\psi_c(\bs \phi, \bs x) \right\| \right] \leq 2.
	    \end{align*}
	    Assertion~\ref{corollary:convergence:Lipschitz} thus follows from Theorem~\ref{theorem:convergence}\,\ref{theorem:convergence:Lipschitz} with $R=2$. As for assertion~\ref{corollary:convergence:smooth}, we have
	    \begin{align*}
	        &\phantom{=}~\; \EE \left[\left\| \bm \nu - \bm p_t(\bm \phi_{t-1}, \bm x_t) - \nabla h(\bm \phi_{t-1}) \right\|^2 | \mathcal F_{t-1} \right] \\
	        &= \EE \left[\left\| \bm p_t(\bm \phi_{t-1}, \bm x_t) - \EE \left[ \nabla_{\bs \phi} \overline\psi_c(\bs \phi_{t-1}, \bs x) \right] \right\|^2 | \mathcal F_{t-1} \right] \\
	        &= \EE \left[\left\| \bm p_t(\bm \phi_{t-1}, \bm x_t) - \nabla_{\bs \phi} \overline\psi_c(\bs \phi_{t-1}, \bs x) + \nabla_{\bs \phi} \overline\psi_c(\bs \phi_{t-1}, \bs x) - \EE \left[ \nabla_{\bs \phi} \overline\psi_c(\bs \phi_{t-1}, \bs x) \right] \right\|^2 | \mathcal F_{t-1} \right] \\
	        & \leq \EE \left[ 2 \left\| \bm p_t(\bm \phi_{t-1}, \bm x_t) - \nabla_{\bs \phi} \overline\psi_c(\bs \phi_{t-1}, \bs x) \right\|^2 + 2 \left\| \nabla_{\bs \phi} \overline\psi_c(\bs \phi_{t-1}, \bs x) - \EE \left[ \nabla_{\bs \phi} \overline\psi_c(\bs \phi_{t-1}, \bs x) \right] \right\|^2 | \mathcal F_{t-1} \right] \\
	        & \leq 2\varepsilon_{t-1}^2 + 2 \leq \bar \varepsilon^2 + 2,
	    \end{align*}
	    where the second inequality holds because $\nabla_{\bs \phi} \overline\psi_c(\bs \phi_{t-1}, \bs x) \in \Delta^N$ and because $\| \nabla_{\bs \phi} \overline\psi_c(\bs \phi_{t-1}, \bs x) \|_2^2 \leq 1$, while the last inequality follows from the choice of $\varepsilon_{t-1}$ in Algorithm~1.
	    As $\overline\psi(\bs \phi, \bs x)$ is $L$-smooth with respect to~$\bs \phi$ by virtue of Proposition~\ref{proposition:structural}\,\ref{proposition:smooth}, we further have
	    \begin{align*}
	        \| \nabla h(\bm \phi) - \nabla h(\bm \phi') \| = \left\| \EE \left[ \nabla_{\bs \phi} \overline\psi_c(\bs \phi, \bs x) - \nabla_{\bs \phi} \overline\psi_c(\bs \phi', \bs x) \right] \right\| \leq L \| \bm \phi - \bm \phi' \| \quad \forall \bm \phi, \bm \phi' \in \R^n.
	    \end{align*}
	    Assertion~\ref{corollary:convergence:smooth} thus follows from Theorem~\ref{theorem:convergence}\,\ref{theorem:convergence:smooth} with $R=2$ and $\sigma = \sqrt{\bar \varepsilon^2 + 2}$.
	    As for assertion~\ref{proposition:concordance}, finally, we observe that $h$ is $M$-generalized self-concordant thanks to Proposition~\ref{proposition:structural}\,\ref{proposition:concordance}. Assertion~\ref{corollary:convergence:concordance} thus follows from Theorem~\ref{theorem:convergence}\,\ref{theorem:convergence:concordance} with $R=2$.
	\end{proof}
	One can show that the objective function of the smooth optimal transport problem~\eqref{eq:smooth_ot} with marginal exponential noise distributions as described in Example~\ref{ex:exp} is generalized self-concordant. Hence, the convergence rate of Algorithm~\ref{algorithm:asgd} for the exponential distribution model of Example~\ref{ex:exp} is of the order~$\mathcal O(1/T)$, which improves the state-of-the-art $\mathcal O(1/\sqrt{T})$ guarantee established by~\citet{genevay2016stochastic}.
	
	\section{Numerical Experiments}
	\label{sec:numerical}
    All experiments are run on a 2.6 GHz 6-Core Intel Core i7, and all optimization problems are implemented in MATLAB~R2020a. The corresponding codes are available at \url{https://github.com/RAO-EPFL/Semi-Discrete-Smooth-OT.git}.
 
     We now aim to assess the empirical convergence behavior of Algorithm~\ref{algorithm:asgd} and to showcase the effects of regularization in semi-discrete optimal transport. To this end, we solve the original dual optimal transport problem~\eqref{eq:ctans_dual_semidisc} as well as its smooth variant~\eqref{eq:smooth_ot} with a Fr\'echet ambiguity set corresponding to the exponential distribution model of Example~\ref{ex:exp}, to the uniform distribution model of Example~\ref{ex:uniform} {\color{black} and to the hyperbolic cosine distribution model of Example~\ref{ex:hyperbolic}}. 
     Recall from Theorem~\ref{theorem:primal_dual} that any Fr\'echet ambiguity set is uniquely determined by a marginal generating function~$F$ and a probability vector~$\bs \eta$. As for the exponential distribution model of Example~\ref{ex:exp}, we set~$F(s) = \exp(10 s - 1)$ and~$\eta_i = 1/N$ for all~$i\in[N]$. In this case problem~\eqref{eq:smooth_ot} is equivalent to the regularized primal optimal transport problem~\eqref{eq:reg_ot_pri_abstract} with an entropic regularizer, and the gradient $\nabla_{\bs \phi}\overline \psi_c(\bs \phi, \bs x)$, which is known to coincide with the vector~$\bs p^\star$ of optimal choice probabilities in problem~\eqref{eq:regularized_c_transform}, admits the closed-form representation~\eqref{eq:softmax}. We can therefore solve problem~\eqref{eq:smooth_ot} with a variant of Algorithm~\ref{algorithm:asgd} that calculates $\nabla_{\bs \phi}\overline \psi_c(\bs \phi, \bs x)$ exactly instead of approximately via bisection. 
     As for the uniform distribution model of Example~\ref{ex:uniform}, we set~$F(s) = s / 20 + 1/2$ and~$\eta_i = 1/N$ for all~$i\in[N]$. In this case problem~\eqref{eq:smooth_ot} is equivalent to the regularized primal optimal transport problem~\eqref{eq:reg_ot_pri_abstract} with a $\chi^2$-divergence regularizer, and the vector~$\bs p^\star$ of optimal choice probabilities can be computed exactly and highly efficiently by sorting thanks to Proposition~\ref{proposition:spmax} in the appendix. We can therefore again solve problem~\eqref{eq:smooth_ot} with a variant of Algorithm~\ref{algorithm:asgd} that calculates $\nabla_{\bs \phi}\overline \psi_c(\bs \phi, \bs x)$ exactly.
     {\color{black} As for the hyperbolic cosine model of Example~\ref{ex:hyperbolic}, we set~$F(s) = \sinh(10s - k)$ with $k=\sqrt{2} - 1 - \textrm{arcsinh}(1)$ and~$\eta_i = 1/N$ for all~$i \in [N]$. In this case problem~\eqref{eq:smooth_ot} is equivalent to the regularized primal optimal transport problem~\eqref{eq:reg_ot_pri_abstract} with a hyperbolic divergence regularizer. However, the vector~$\bm p^\star$ is not available in closed form, and thus we use~Algorithm~\ref{algorithm:bisection} to compute~$\bm p^\star$ approximately.} Lastly, note that the original dual optimal transport problem~\eqref{eq:ctans_dual_semidisc} can be interpreted as an instance of~\eqref{eq:smooth_ot} equipped with a degenerate singleton ambiguity set that only contains the Dirac measure at the origin of~$\R^N$. In this case $\overline \psi_c(\bs \phi,\bs x) = \psi_c(\bs \phi,\bs x)$ fails to be smooth in~$\bs \phi$, but an exact subgradient $\bs p^\star\in\partial_{\bs \phi} \overline \psi_c(\bs \phi,\bs x)$ is given by
     \[
        p_i^\star  = \begin{cases}
        1 \quad &\text{if}~ i = \min \argmax\limits_{i \in [N]}~\phi_i - c(\bs x, \bs y_i),\\
        0 &\text{otherwise.}
        \end{cases}
     \]
     We can therefore solve problem~\eqref{eq:ctans_dual_semidisc}  with a variant of Algorithm~\ref{algorithm:asgd} that has access to exact subgradients (instead of gradients) of~$\overline \psi_c(\bs \phi, \bs x)$. Note that the maximizer~$\bs \phi^\star$ of~\eqref{eq:ctans_dual_semidisc} may not be unique. In our experiments, we force Algorithm~\ref{algorithm:asgd} to converge to the maximizer with minimal Euclidean norm by adding a vanishingly small Tikhonov regularization term to~$\psi_c(\bs \phi,\bs x)$. Thus, we set $\overline \psi_c(\bs \phi,\bs x) = \psi_c(\bs \phi,\bs x) + \varepsilon\|\bs \phi\|_2^2$ for some small regularization weight~$\varepsilon> 0$, in which case~$\bs p^\star+2\varepsilon \bs \phi\in\partial_{\bs \phi}\overline \psi_c(\bs \phi,\bs x)$ is an exact subgradient.

     In the following we set~$\mu$ to the standard Gaussian measure on~$\mc X= \mathbb{R}^2$ and~$\nu$ to the uniform measure on 10 independent samples drawn uniformly from~$\mc Y=[-1,\, 1]^2$. We further set the transportation cost to~$c(\bs x, \bs y) = \|\bs x - \bs y\|_\infty$. Under these assumptions, we use Algorithm~\ref{algorithm:asgd} to solve the original as well as the~{\color{black}three} smooth optimal transport problems approximately {\color{black} for $T=1,\ldots, 10^5$. For each fixed~$T$ the step size is selected in accordance with Corollary~\ref{corollary:convergence}.}
     We emphasize that Corollary~\ref{corollary:convergence}\,\ref{corollary:convergence:Lipschitz} remains valid if~$\overline \psi_c(\bs \phi,\bs x)$ fails to be smooth in~$\bs \phi$ and we have only access to subgradients; see \cite[Corollary~1]{nesterov2008confidence}. Denoting by~$\bar{\bs \phi}_T$ the output of Algorithm~\ref{algorithm:asgd}, we record the suboptimality
	 \begin{equation*}
	     \overline W_c(\mu, \nu) - \mathbb{E}_{\bs x \sim \mu} \left[ \bs \nu^\top \bar{\bs \phi}_T - \overline\psi_c(\bar{\bs \phi}_T , \bs x)\right] 
	 \end{equation*}
     of~$\bar{\bs \phi}_T$ in~\eqref{eq:smooth_ot} as well as the discrepancy~$\| \bar {\bs \phi}_T - \bs \phi^\star \|^2_2$ of~$\bar{\bs \phi}_T$ to the exact maximizer~$\bs \phi^\star$ of problem~\eqref{eq:smooth_ot} as a function of~$T$. In order to faithfully measure the convergence rate of $\bar{\bs \phi}_T$ and its suboptimality, we need to compute~$\bs \phi^\star$ as well as~$\overline W_c(\mu, \nu)$ to within high accuracy. This is only possible if the dimension of~$\mc X$ is small ({\em e.g.}, if~$\mc X= \mathbb{R}^2$ as in our numerical example); even though Algorithm~\ref{algorithm:asgd} can efficiently solve optimal transport problems in high dimensions.
	We obtain high-quality approximations for~$\overline W_c(\mu, \nu)$ and~$\bs \phi^\star$ by solving the finite-dimensional optimal transport problem between~$\nu$ and the discrete distribution that places equal weight on $10 \times T$ samples drawn independently from~$\mu$. Note that only the first~$T$ of these samples are used by Algorithm~\ref{algorithm:asgd}. The proposed high-quality approximations of the {\color{black} entropic and $\chi^2$-divergence regularized} optimal transport problems are conveniently solved via Nesterov's accelerated gradient descent method, where the suboptimality gap of the $t^{\text{th}}$ iterate is guaranteed to decay as~$\mc O(1/ t^2)$ under the step size rule advocated in~\citep[Theorem~1]{Nesterov1983AMF}.
	{\color{black} To our best knowledge, Nesterov's accelerated gradient descent algorithm is not guaranteed to converge with inexact gradients. For the hyperbolic divergence regularized optimal transport problem, we thus use Algorithm~\ref{algorithm:asgd} with $50 \times T$ iterations to obtain an approximation for~$\overline W_c(\mu, \nu)$ and~$\bm \phi^\star$.}
	In contrast, we model the high-quality approximation of the original optimal transport problem~\eqref{eq:ctans_dual_semidisc} in YALMIP~\citep{yalmip} and solve it with MOSEK. If this problem has multiple maximizers, we report the one with minimal Euclidean norm.

    Figure~\ref{fig:num_results} shows how the suboptimality of~$\bar{\bs \phi}_T$ and the discrepancy between $\bar {\bs \phi}_T$ and the exact maximizer decay with~$T$, both for the original as well as for the entropic, the $\chi^2$-divergence and~{\color{black} hyperbolic divergence} regularized optimal transport problems, {\color{black} averaged across 20 independent simulation runs.}
    Figure~\ref{fig:suboptimality} suggests that the suboptimality decays as~$\mc O(1/\sqrt{T})$  for the original optimal transport problem, which is in line with the theoretical guarantees by~\citet[Corollary~1]{nesterov2008confidence},
    and as $\mc O(1/ T)$ for the {\color{black} entropic, the $\chi^2$-divergence and the hyperbolic divergence regularized} optimal transport problems, which is consistent with the theoretical guarantees established in Corollary~\ref{corollary:convergence}. Indeed, entropic regularization can be explained by the exponential distribution model of Example~\ref{ex:exp}, where the exponential distribution functions $F_i$ satisfy the generalized self-concordance condition~\eqref{eq:cumulative:concoradance} with~$M =1/ \lambda$. Similarly, $\chi^2$-divergence regularization can be explained by the uniform distribution model of Example~\ref{ex:uniform}, where the uniform distribution functions $F_i$ satisfy the generalized self-concordance condition with any~$M > 0$. {\color{black} Finally, hyperbolic divergence regularization can be explained by the hyperbolic cosine distribution model of Example~\ref{ex:hyperbolic}, where the hyperbolic cosine functions~$F_i$ satisfy the generalized self-concordance condition with $M = 1/\lambda$.} In all cases the smallest eigenvalue of $-\nabla_{\bs \phi}^2 \EE_{\bs x \sim \mu} [\bs \nu^\top \bs \phi^\star - \overline{\psi}_{c}(\bs \phi^\star, \bs x)]$, which we estimate when solving the high-quality approximations of the two smooth optimal transport problems, is strictly positive. Therefore, Corollary~\ref{corollary:convergence}~\ref{corollary:convergence:concordance} is indeed applicable and guarantees that the suboptimality gap is bounded above by~$\mc O (1/T)$.
    Finally, Figure~\ref{fig:dualvars} suggests {\color{black} that~$\| \bar {\bs \phi}_T - \bs \phi^\star \|^2_2$ converges to~$0$ at rate~$\mc O(1/T)$ for the entropic, the $\chi^2$-divergence and the hyperbolic divergence} regularized optimal transport problems, which is consistent with~\citep[Proposition~10]{bach2013adaptivity}.
    \begin{figure}[t]
     \centering
     \begin{subfigure}[h]{0.43\columnwidth}
         \includegraphics[width=\textwidth]{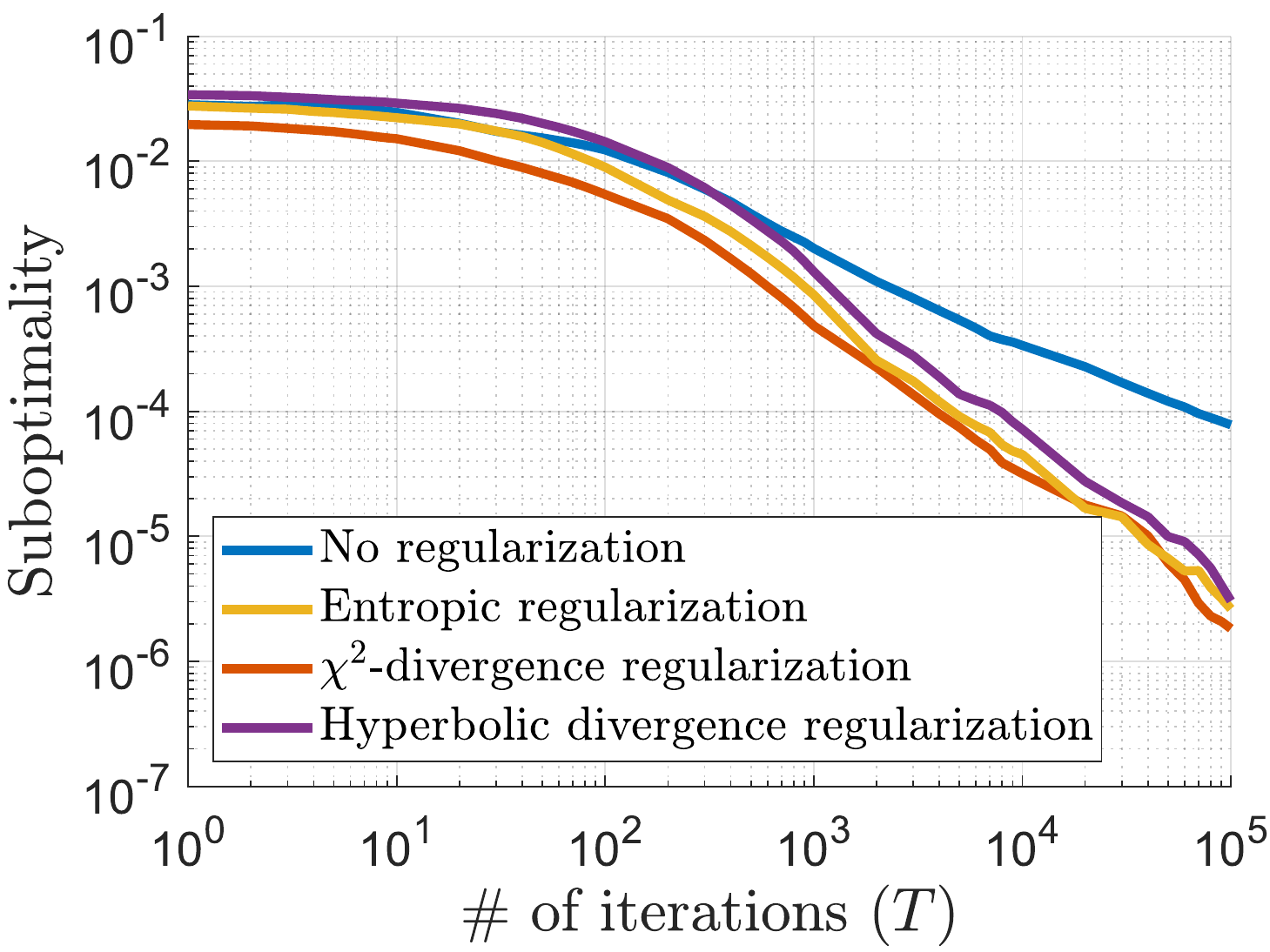}
         \caption{}
         \label{fig:suboptimality}
     \end{subfigure}
     \hspace{2cm}
     \begin{subfigure}[h]{0.43\columnwidth}
         \includegraphics[width=\textwidth]{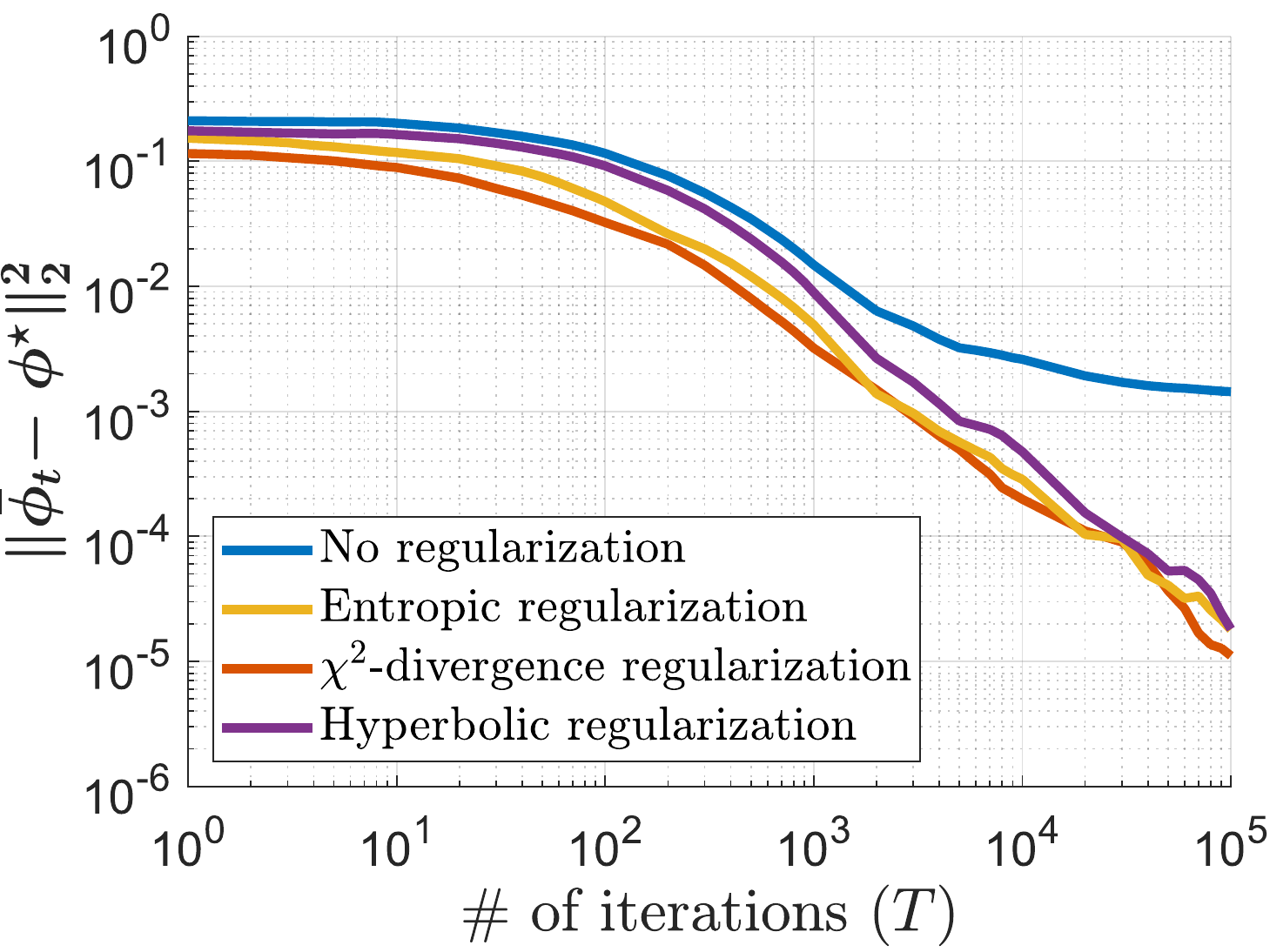}
         \caption{}
         \label{fig:dualvars}
     \end{subfigure}
        \caption{Suboptimality (a) and discrepancy to~$\bs\phi^\star$ (b) of the outputs~$\bar{\bs \phi}_T$ of Algorithm~\ref{algorithm:asgd} for the original (blue), the entropic regularized (orange), the $\chi^2$-divergence regularized (red) {\color{black} and the hyperbolic divergence regularized (purple)} optimal transport~problems.}  \label{fig:num_results}
    \end{figure}
    
    \textbf{Acknowledgements.} This research was supported by the Swiss National Science Foundation under the NCCR Automation, grant agreement~51NF40\_180545. The research of the second author is supported by an Early Postdoc.Mobility Fellowship, grant agreement P2ELP2\_195149.
    
	\appendix
	\renewcommand\thesection{\Alph{section}}
		{\color{black}
			\section{Approximating the Minimizer of a Strictly Convex Function}
			The following lemma is  key ingredient for the proofs of Theorem~\ref{theorem:hard} and Corollary~\ref{corollary:approximate-hard}.
	\begin{lemma}
	\label{lemma:strictly_convex_min}
	    Assume that $g:[0,1] \to \R_{+}$ is a strictly convex function with unique minimizer~$t^\star\in[0,1]$, and define $L = \ceil{\log_2(1/\delta)} +1$ for some prescribed tolerance~$\delta \in (0,1)$. Then, the following hold.
	    \begin{enumerate}[label=(\roman*)]
	        \item Given an oracle that evaluates~$g$ exactly, we can compute a $\delta$-approximation for~$t^\star$ with~$2L$ oracle~calls.
	        \item Given an oracle that evaluates~$g$ inexactly to within an absolute accuracy
	        \begin{equation*}
	             \eps = \frac{1}{4}\min_{l\in [2^L] } \left\{|g(t_l) - g(t_{l-1})| : g(t_l) \neq g(t_{l-1}) \right\} \quad \text{with} \quad  t_l = \frac{l}{2^{L}} \quad \forall l =0,\ldots,2^L,
	        \end{equation*}
	        we can compute a $2\delta$-approximation for~$t^\star$ with~$2L$ oracle~calls.
	    \end{enumerate}
	\end{lemma}
	\begin{proof}
Consider the uniform grid~$\{t_0, \ldots, t_{2^L}\}$, and note that the difference~$2^{-L}$ between consecutive grid points is strictly smaller than~$\delta$. Next, introduce a piecewise affine function~$\bar g : [0,1] \to \R_+$ that linearly interpolates~$g$ between consecutive grid points. By construction, $\bar g$ is affine on the interval~$[t_{l-1}, t_{l}]$ with slope~$a_l/\delta$ and~$a_l= g(t_{l}) - g(t_{l-1})$ for all $l \in [2^{L}]$. In addition,~$\bar g$ is continuous and inherits convexity from~$g$. As~$g$ is strictly convex, it is easy to verify that~$\bar g$ has also a kink at every inner grid point~$t_l$ for $l\in[2^L-1]$, and therefore the distance between the unique minimizer~$t^\star$ of~$g$ and any minimizer of~$\bar g$ is at most~$2^{-L}<\delta$. In order to compute a $\delta$-approximation for~$t^\star$, it thus suffices to find a minimizer of~$\bar g$.

Next, define $\bm a = (a_0, \ldots, a_{2^{L} })$ with~$a_0 = -\infty$. As~$\bar g$ has a kink at every inner grid point, we may conclude that the array~$\bm a$ is sorted in ascending order, that is, $a_l> a_{l-1}$ for all~$l \in [2^L]$. This implies that at most one element of~$\bm a$ can vanish. In the following, define~$l^\star = \max\{l\in \{0\}\cup [2^{L}]: a_l \leq 0 \}$. If~$l^\star=0$, then~$\bar g$ is uniquely minimized by~$t_{l^\star}=0$, and~$t^\star$ must fall within the interval~$[t_0, t_1]$. If~$l^\star>0$ and~$a_{l^\star}<0$, on the other hand, then~$\bar g$ is uniquely minimized by~$t_{l^\star}$, and~$t^\star$ must fall within the interval~$[t_{l^\star-1},t_{l^\star + 1}]$. If~$l^\star>0$ and~$a_{l^\star}=0$, finally, then every point in~$[t_{l^\star-1}, t_{l^\star}]$ minimizes~$\bar g$, and~$t^\star$ must also fall within~$[t_{l^\star-1}, t_{l^\star}]$. In any case, $t_{l^\star}$ provides a $\delta$-approximation for~$t^\star$. In the remainder of the proof we show that the index~$l^\star$ can be computed efficiently by Algorithm~\ref{algorithm:binary}. This bisection scheme maintains lower and upper bounds~$\underline l$ and~$\overline l$ on the sought index~$l^\star$, respectively, and reduces the search interval between~$\underline l$ and~$\overline l$ by a factor of two in each iteration. Thus, Algorithm~\ref{algorithm:binary}  computes~$l^\star$ in exactly~$L$ iterations \cite[\S~12]{cormen2009introduction}.
    
    \begin{table}[H]
	\centering
	\begin{minipage}{0.6\textwidth}
    \vspace{-1em}
	\begin{algorithm}[H]
	\color{black}
	\caption{Binary search algorithm \label{algorithm:binary}}
		\begin{algorithmic}[1]
 			\Require An array $\bm a\in \mathbb R^{2^L}$ sorted in ascending order
			\State Initialize $\underline{l} = 0$ and $\overline{l} = 2^L$
			\While{$\underline l<\overline l$}
			\State \hspace{-1ex}Set $l \gets {(\overline{l} + \underline{l})}/{2}$
			\State \hspace{-1ex}\algorithmicif~~$a_l \leq 0$~~\algorithmicthen~~$\underline{l} \gets l$~~\algorithmicelse~~$\overline{l} \gets l$
			\EndWhile	\vspace{0.1em}
			\State \algorithmicif~~$a_{\underline l} \leq 0$~~\algorithmicthen~~$l \gets \underline{l}$~~\algorithmicelse~~$l \gets \overline{l}$
			\Ensure ${l^\star \gets l}$
		\end{algorithmic}
	\end{algorithm}
	\vspace{-1em}
	\end{minipage}
    \end{table}
    
We remark that~$l$ is guaranteed to be an integer and thus represents a valid index in each iteration of the algorithm because~$\underline l$ and~$\overline l$ are initialized as~$0$ and~$2^L$, respectively. Note also that if we have access to an oracle for evaluating~$g$ exactly, then any element~$a_l$ of the array~$\bm a$ can be computed with merely two oracle calls.  Algorithm~\ref{algorithm:binary} thus computes $l^\star$ with~$2L$ oracle calls in total. Hence, assertion~(i) follows.
    
As for assertion~(ii), assume now that we have only access to an inexact oracle that outputs, for any fixed~$t\in[0,1]$, an approximate function value~$\widetilde g(t)$ with~$|\widetilde g(t)- g(t)| \leq \eps$, where~$\eps$ is defined as in the statement of the lemma. Reusing the notation from the first part of the proof, one readily verifies that~$\eps = \frac{1}{4}\min_{l \in [2^{L}]}\{a_l : a_l \neq 0\}$. Next, set~$\widetilde a_0 = -\infty$, and define~$\widetilde a_l = \widetilde g(t_{l}) -\widetilde g(t_{l-1})$ for all~$l \in [2^{L}]$. Therefore, $\widetilde{\bm a} = (\widetilde a_0, \ldots, \widetilde a_{2^{L}})$ can be viewed as an approximation of~$\bm a$. Moreover, Algorithm~\ref{algorithm:binary} with input~$\widetilde{\bm a}$ computes an approximation~$\tilde l^\star$ of~$l^\star$ in~$L$ iterations. Next, we will show that~$|\tilde l^\star - l^\star| \leq 1$ even though~$\widetilde{\bm a}$ is not necessarily sorted in ascending order. To see this, note that~$|a_l| \geq 4 \eps$ for all~$l\in[2^L]$ with~$a_l\neq 0$ by the definition of~$\eps$. By the triangle inequality and the assumptions about the inexact oracle, we further have
\[
    |a_l- \widetilde a_l| \leq |\widetilde g(t_{l}) - g(t_{l})| + |\widetilde g(t_{l-1}) - g(t_{l-1})| \leq 2\eps\quad \forall l\in[2^L].
\]
This reasoning reveals that~$\widetilde a_l$ has the same sign as~$a_l$ for every~$l\in[2^L]$ with~$a_l\neq 0$. In addition, it implies that~$t_{\tilde l^\star}$ approximates~$t^\star$ irrespective of whether or not the array~$\bm a$ has a vanishing element. Indeed, if no element of~$\bm a$ vanishes, then $\widetilde a_l$ has the same sign as~$a_l$ for all~$l \in [2^{L}]$. As Algorithm~\ref{algorithm:binary} only checks signs, this implies that~$\tilde l^\star=l^\star$ and that $t_{\tilde l^\star}$ provides a $\delta$-approximation for~$t^\star$ as in assertion~(i). If one element of~$\bm a$ vanishes, on the other hand, then $\widetilde a_l$ has the same sign as~$a_l$ for all~$l \in [2^{L}]$ with $l\neq l^\star$. As Algorithm~\ref{algorithm:binary} only checks signs, this implies that~$|\tilde l^\star - l^\star| \leq 1$. Recalling that~$|t^\star - t_{l^\star}| \leq \delta $, we thus have
    	\[
    	    |t_{\tilde l^\star} - t^\star| \leq |t_{\tilde l^\star} - t_{l^\star}| + |t^\star - t_{l^\star}| \leq 2\delta.
    	\]
In either case, $t_{\tilde l^\star}$ provides a $2\delta$-approximation for~$t^\star$. As any element of the array~$\widetilde{\bm a}$ can be evaluated with two oracle calls, Algorithm~\ref{algorithm:binary} computes~$\tilde l^\star$ with $2L$ oracle calls in total. Hence, assertion~(ii) follows.
\end{proof}

}

	\section{Efficiency of Binary Search}
	\label{appendix:polynomial_calls}
	\renewcommand\thesection{\Alph{section}}
	We adopt the conventions of \citet[\S~2.1]{schrijver1998theory} to measure the size of a computational problem, which is needed to reason about the problem's complexity. Specifically, the size of a scalar $x\in\mathbb R$ is defined as
	\begin{equation*}
	{\rm{size}}(x) = \left\{\begin{array}{ll} 1 + \ceil{\log_2{(|p|+1)}} + \ceil{\log_2{(q+1)}} & \text{if $x= p/q$ with $p\in\mathbb Z$ and $q\in\mathbb N$ are relatively prime,}\\
	\infty & \text{if $x$ is irrational,}
	\end{array}
	\right.
	\end{equation*}
	where we reserve one bit to encode the sign of $x$. The size of a real vector is defined as the sum of the sizes of its components plus its dimension. Thus, the input size of an instance $\bm w\in \mathbb{R}_+^d$ and $b\in\mathbb R_+$ of the knapsack problem described in Lemma~\ref{lemma:Grani} amounts to
	\begin{equation*}
	{\rm{size}}(\bm {w}, b) = d +  1+\sum_{i=1}^d{{\rm{size}}(w_i)}+ {\rm{size}}(b).
	\end{equation*}
	In the following we will prove that the number of iterations
 	{\color{black}
    \[L = \left\lceil   
    \log_2(6) + \log_2 d! + + d \log_2(\|\bm w \|_1 + 2) + (d+1) \log_2(d+1) + \sum\limits_{i=1}^d \log_2(w_i) \right\rceil + 1\]
	of the bisection algorithm used in the proof of Theorem~\ref{theorem:hard} is upper bounded by a polynomial in ${\rm{size}}(\bm {w}, b)$.} The claim holds trivially if any component of $(\bs w,b)$ is irrational.
	Below we may thus assume that $w_i = p_i / q_i$ and $b = p_{d+1}/q_{d+1}$, where $p_i\in\mathbb Z_+$ and $q_i\in\mathbb N$ are relatively prime for every $i\in[d+1]$. This implies that
	\begin{equation*}
	{\rm{size}}(\bm {w}, b) =2 d + 1+  \sum_{i=1}^{d+1}{\ceil{\log_2{(p_i+1)}} + \ceil{\log_2{(q_i+1)}} }. 
	\end{equation*}
	In order to show that $L$ is bounded by a polynomial in ${\rm{size}}(\bm {w}, b)$, we first note that
	\begin{equation}
	\label{eq:ineq-1}
	\log_2 d!\leq \log_2 d^d\leq d^2\leq \text{size}(\bs w, b)^2\quad \text{and}\quad (d+1) \log_2(d+1) \leq (d+1)^2 \leq \text{size}(\bs w, b)^2.
	\end{equation}
	This follows from the properties of the logarithm and the definition of the size function. Similarly, we find 
	\begin{align*}
	d \log_2(2 + \|\bs w\|_1) = d\log_2\left (2 + \sum\limits_{i=1}^d {p_i}/{q_i}\right) &\leq d\log_2((d+1) \max\{2, \max_{i \in [d]}\{p_i/q_i\}\})\\
	&=d\log_2(d+1)+ d\max\{1, \max_{i \in [d]}\{\log_2(p_i) - \log_2(q_i)\}\}\\
	&\leq (d+1)\log_2(d+1)+ d\max\{1,\max_{i \in [d]}\{ \log_2(p_i) + \log_2(q_i)\}\}\\
	&\leq \text{size}(\bs w, b)^2 + \text{size}(\bs w, b)\max_{i \in [d]}\{\log_2(p_i + 1) + \log_2(q_i+1)\}\\
	&\leq 2\,\text{size}(\bs w, b)^2,
	\end{align*}
	where the first inequality follows from the monotonicity of the logarithm, the second inequality holds because  $\log_2(q_i) \geq 0$ for all $q_i \in \mathbb N$, and the third inequality exploits the second bound in~\eqref{eq:ineq-1} as well as the trivial estimates $d\leq {\rm{size}}(\bm {w}, b)$ and $1=\log_2 2\leq \log_2(q_i+1)$ for all $q_i \in \mathbb N$. The last inequality, finally, follows from the observation that
	\begin{equation*}
	\max_{i \in [d]}\{\log_2(p_i + 1) + \log_2(q_i+1)\}
	\leq  \sum\limits_{i=1}^d \log_2(p_i + 1)+ \log_2(q_i + 1) \leq \text{size}(\bs w, b). 
	\end{equation*}
	Using a similar reasoning, we find 
	\begin{align*}
	\sum\limits_{i=1}^d \log_2(w_i) = \sum\limits_{i=1}^d \log_2(p_i / q_i ) \leq \sum\limits_{i=1}^d \log_2(p_i)+ \log_2(q_i) \leq \text{size}(\bs w, b),
	\end{align*}
	and thus all terms in the definition of $L$ grow at most quadratically with $\text{size}(\bs w, b)$. Hence, the number of iterations $L$ of the bisection algorithm is indeed bounded by a polynomial in $\text{size}(\bs w, b)$.
    
    \section{Detailed Derivations for the Examples of Marginal Ambiguity Sets}
    \label{appendix:derivations}
    
	\begin{example}[Exponential distribution model]
	If the marginal generating function in~\eqref{eq:marginal_dists} is set to $F(s) = \exp(s/\lambda - 1)$ for some~$\lambda>0$, then the marginal distribution function $F_i$ for any $i\in[N]$ reduces to 
	\[F_i(s) = \min\left\{1, \max\{0, 1 - \eta_i \exp(-s/\lambda - 1)\}  \right\},\]
	which characterizes a (shifted) exponential distribution. Defining $f$ as in Theorem~\ref{theorem:primal_dual}, we then obtain
	\begin{align*}
	    f(s)=\int_0^s F^{-1}(t)\diff t = \lambda \int_0^s (\log(t) + 1) \diff t = \lambda s\log(s),
	\end{align*}
	where the third equality exploits the standard convention that $0\log(0) = 0$. The proof of Theorem~\ref{theorem:primal_dual} further implies that $\int_{1-p_i}^1 F_i^{-1}(t) \diff t=-\eta_i f(p_i/\eta_i)$ for all~$i\in[N]$; see~\eqref{eq:integral_rep_f}. By Proposition~\ref{proposition:regularized_ctrans} we thus have
	\begin{align*}
	    \overline{\psi}_c(\bs \phi, \bs x) & = \max\limits_{\bs p \in \Delta^N} \sum\limits_{i=1}^N (\phi_i - c(\bs x, \bs y_i))p_i - \lambda \sum\limits_{i=1}^N p_i \log\left(\frac{p_i}{\eta_i}\right).
	\end{align*}
    Next, assign Lagrange multipliers $\tau$ and $\bs \zeta$ to the simplex constraints $\sum_{i=1}^N p_i=1$ and $\bs p\ge \bs 0$, respectively. 
    If we can find $\bs p^\star$, $\tau^\star$ and $\bm \zeta^\star$ that satisfy the Karush-Kuhn-Tucker optimality conditions 
	\begin{align*}
	    \begin{array}{lll}
	    \sum_{i=1}^N p^\star_i =1, ~p^\star_i \geq 0 & \forall i \in [N] & \text{(primal feasibility)}\\
	    \zeta^\star_i\geq 0 & \forall i \in [N] & \text{(dual feasibility)}\\ \zeta_i^\star p_i^\star=0 & \forall i \in [N] & \text{(complementary slackness)} \\
	    \phi_i - c(\bs x, \bs y_i) -\lambda \log\left(\frac{p_i}{\eta_i}\right) - \lambda - \tau^\star + \zeta^\star_i = 0 & \forall i \in [N] & \text{(stationarity)},
	    \end{array}
	\end{align*}
	then $\bs p^\star$ is optimal in the above maximization problem. To see that $\bs p^\star$, $\tau^\star$ and $\bm \zeta^\star$ exist, we use the stationarity condition to conclude that $p_i^\star = \eta_i \exp((\phi_i - c(\bs x, \bs y_i) - \lambda - \tau^\star + \zeta^\star_i) / \lambda)>0$. As $\eta_i > 0$, we have $\zeta_i^\star =0$ by complementary slackness. We may then conclude that $p_i^\star = \eta_i \exp((\phi_i - c(\bs x, \bs y_i)- \lambda - \tau^\star) / \lambda)$ for all $i \in [N]$, which implies via primal feasibility that $\sum_{i=1}^N \eta_i \exp((\phi_i - c(\bs x, \bs y_i) - \lambda - \tau^\star) / \lambda) = 1$. Solving this equation for $\tau^\star$ and substituting the resulting formula for $\tau^\star$ back into the formula for $p_i^\star$ yields
    \[\tau^\star = \lambda \log\left(\sum_{i=1}^N \eta_i \exp\left(\frac{\phi_i - c(\bs x, \bs y_i) - \lambda} {\lambda}\right)\right)\quad\text{and}\quad p_i^\star =\frac{\eta_i \exp\left(({\phi_i - c(\bs x, \bs {y_i}) )}/{\lambda}\right)}{ \sum_{j=1}^N \eta_j \exp\left(({\phi_j - c(\bs x,\bs {y_j}) })/{\lambda}  \right)}\quad \forall i \in [N].\]
    The vector $\bs p^\star$ constructed in this way constitutes an optimal solution for the maximization problem that defines $\overline{\psi}_c(\bs \phi, \bs x)$. Evaluating the objective function value of $\bs p^\star$ in this problem finally confirms that the smooth $c$-transform coincides with the log-partition function~\eqref{eq:partition:function}.
    \end{example}
    \begin{example}[Uniform distribution model]
	If the marginal generating function in~\eqref{eq:marginal_dists} is set to $F(s) = s/ (2\lambda) + 1/2$ for some~$\lambda>0$, then the marginal distribution function $F_i$ for any $i\in[N]$ reduces to 
	\[
	   F_i(s) = \min\{1, \max\{0, 1 + \eta_i s/ (2\lambda) -\eta_i/ 2 \}\},
	\]
	which characterizes a uniform distribution. Defining $f$ as in Theorem~\ref{theorem:primal_dual}, we then obtain
	\begin{align*}
	    f(s)=\int_0^s F^{-1}(t)\diff t = \lambda \int_0^s (2t-1) \diff t = \lambda(s^2 - s).
	\end{align*}
	The proof of Theorem~\ref{theorem:primal_dual} further implies that $\int_{1-p_i}^1 F_i^{-1}(t) \diff t=-\eta_i f(p_i/\eta_i)$ for all~$i\in[N]$; see~\eqref{eq:integral_rep_f}. By Proposition~\ref{proposition:regularized_ctrans}, the smooth $c$-transform thus simplifies to 
	\begin{align*}
	    \overline{\psi}_c(\bs \phi, \bs x) & = \max\limits_{\bs p \in \Delta^N} \sum\limits_{i=1}^N (\phi_i - c(\bs x, \bs y_i))p_i - \lambda \sum\limits_{i=1}^N \frac{p_i^2}{\eta_i} + \lambda = \lambda + \lambda\spmax_{i \in [N]} \;\frac{\phi_i - c(\bs x, \bs {y_i})}{\lambda}, 
	\end{align*}
	where the last equality follows from the definition of the sparse maximum operator in~\eqref{eq:spmax}.
    \end{example}
    
    \begin{example}[Pareto distribution model]
    If the marginal generating function in~\eqref{eq:marginal_dists} is set to $F(s) = (s (q-1) / (\lambda q)+1/q)^{1/(q-1)}$ for some~$\lambda, q>0$, then the marginal distribution function $F_i$ for any $i\in[N]$ reduces to 
	\[
	   F_i(s) = \min\left\{1, \max\left\{0, 1 - \eta_i \left( \frac{s (1-q)}{\lambda q} + \frac{1}{q} \right)^{\frac{1}{q-1}}  \right\} \right\},
	\]
	which characterizes a Pareto distribution. Defining $f$ as in Theorem~\ref{theorem:primal_dual}, we then obtain
	\begin{align*}
	    f(s)=\int_0^s F^{-1}(t) \diff t = \frac{\lambda}{q-1} \int_0^s (q t^{q-1} - 1) \diff t = \lambda \frac{s^q - s}{q-1}.
	\end{align*}
	\end{example}
	
	\begin{example}[Hyperbolic cosine distribution model]
	If the marginal generating function in~\eqref{eq:marginal_dists} is set to $F(s) = \sinh(s/\lambda - k)$ for some~$\lambda > 0$ and for $k = \sqrt{2} - 1 - \asinh(1)$, then the marginal distribution function $F_i$ for any $i\in[N]$ reduces to 
	\[
	   F_i(s) = \min\left\{1, \max\left\{0, 1 + \eta_i \sinh(s/\lambda + k) \right\} \right\},
	\]
	which characterizes a hyperbolic cosine distribution. Defining $f$ as in Theorem~\ref{theorem:primal_dual}, we then obtain
	\begin{align*}
	    f(s) = \int_0^s F^{-1}(t) \diff t 
	    = \lambda \int_0^s (\asinh(s) + k ) \diff t 
	    = \lambda(s \hspace{0.1em} \asinh(s) - \sqrt{s^2 + 1} + 1 + ks).
	\end{align*}
	\end{example} 
	
	\begin{example}[$t$-distribution model]
	If the marginal generating function in~\eqref{eq:marginal_dists} is set to 
	\[ F(s) = \frac{N}{2}\left(1 + \frac{s - \sqrt{N-1}} {\sqrt{\lambda^2 + (s - \sqrt{N-1})^{2}}}\right) \]
	for some~$\lambda, q>0$, then the marginal distribution function $F_i$ for any $i\in[N]$ reduces to
	\[
	   F_i(s) = \min\left\{1, \max\left\{0, 1 - \frac{\eta_i N}{2} \left(1 - \frac{s + \sqrt{N-1}} {\sqrt{\lambda^2 + (s + \sqrt{N-1})^{2}}}\right) \right\} \right\},
	\]
	which characterizes a $t$-distribution with $2$ degrees of freedom. Defining $f$ as in Theorem~\ref{theorem:primal_dual}, we then find
	\begin{align*}
	    f(s) = \int_0^s F^{-1}(t) \diff t = \lambda \int_0^s  \left( \frac{2s - N}{2 \sqrt{s(N-s)}} + \sqrt{N-1} \right) \diff t = -\lambda \sqrt{s(N-s)} + \lambda s \sqrt{N-1}.
	\end{align*}
	\end{example}
    
	\renewcommand\thesection{\Alph{section}}
	\section{The Sparse Maximum Function}
	\label{appendix:spmax}
	\renewcommand\thesection{\Alph{section}}
	The following proposition, which is a simple extension of \citep[Proposition~1]{sparsemax}, suggests that the solution of~\eqref{eq:spmax} can be computed by a simple sorting algorithm.
	\begin{proposition}
	\label{proposition:spmax}
	Given $\bs u\in\R^N$, let $\sigma$ be a permutation of $[N]$ with $u_{\sigma(1)} \geq u_{\sigma(2)} \geq \cdots \geq u_{\sigma(N)}$, and set
	$$k = \max \left\{ j \in [N]: 2 + \left( \sum_{i =1}^j \eta_{\sigma(i)} \right) u_{\sigma(j)} > \sum_{i =1}^j \eta_{\sigma(i)} u_{\sigma(i)} \right\}\quad \text{and} \quad \tau^\star = \frac{\left(\sum_{i=1}^k \eta_{\sigma(i)} u_{\sigma(i)}\right) - 2}{\sum_{i=1}^k \eta_{\sigma(i)}}.$$
	Then $p^\star_i =  \eta_i [u_i - \tau^\star]_+ / 2$, $i \in [N]$, is optimal in~\eqref{eq:spmax}, where $[\cdot]_+ = \max \{ 0, \cdot \}$ stands for the ramp function.
	\end{proposition}
	\begin{proof}
	Assign Lagrange multipliers $\tau$ and $\bs \zeta$ to the simplex constraints $\sum_{i=1}^N p_i=1$ and $\bs p\ge \bs 0$ in problem~\eqref{eq:spmax}, respectively. If we can find $\bs p^\star$, $\tau^\star$ and $\bm \zeta^\star$ that satisfy the Karush-Kuhn-Tucker conditions
    \begin{align*}
		    \begin{array}{lll}
		    \sum_{i=1}^N p^\star_i =1, ~p^\star_i \geq 0 & \forall i \in [N] & \text{(primal feasibility)}\\
		    \zeta^\star_i\geq 0 & \forall i \in [N] & \text{(dual feasibility)}\\ \zeta_i^\star p_i^\star=0 & \forall i \in [N] & \text{(complementary slackness)} \\
		    u_i - \frac{2 p_i^\star}{\eta_i} - \tau^\star + \zeta^\star_i = 0 & \forall i \in [N] & \text{(stationarity)},
		    \end{array}
		\end{align*}
		then $\bs p^\star$ is optimal in~\eqref{eq:spmax}.
	In the following, we show that $\bs p^\star$, $\tau^\star$ and $\bm \zeta^\star$ exist. Note first that if $p_i^\star > 0$, then $\zeta_i^\star = 0$ by complementary slackness and $p_i^\star = \eta_i (u_i - \tau^\star) / 2$ by stationarity. On the other hand, if $p_i^\star = 0$, then $\zeta_i^\star \geq 0$ by dual feasibility and $u_i - \tau^\star \leq 0$ by stationarity. In both cases we have $p_i^\star = \eta_i [u_i - \tau^\star]_+ / 2$ for all $i\in[N]$, which implies that $\sum_{i=1}^N \eta_i [u_i - \tau^\star]_+ = 2$ by primal feasibility. It thus remains to show that~$\tau^\star$ as defined in the proposition statement solves this nonlinear scalar equation. To this end, note that by the definitions of the permutation $\sigma$ and the index $k$ we have
	\begin{align*}
	    u_{\sigma(j)} \geq u_{\sigma(k)} > \frac{(\sum_{i=1}^k \eta_{\sigma(i)} u_{\sigma(i)}) - 2}{\sum_{i=1}^k \eta_{\sigma(i)}} = \tau^\star
	\end{align*}
	for all $j \le k$. The definition of the index $k$ further implies that
	\begin{align*}
	    2 + \left( \sum_{i=1}^{k+1} \eta_{\sigma(i)} \right) u_{\sigma(k+1)} \leq \sum_{i=1}^{k+1} \eta_{\sigma(i)} u_{\sigma(i)}.
	\end{align*}
	A simple reordering, dividing both sides of the above inequality by $\sum_{i=1}^k \eta_{\sigma(i)}$, and using the definition of~$\tau^\star$ then yields $u_{\sigma(k+1)} \leq \tau^\star$. In addition, by the definition of the permutation $\sigma$, we have $u_{\sigma(j)} \leq u_{\sigma(k+1)}$ for all $j > k$. Hence, we conclude that $u_{\sigma(j)} \leq \tau^\star$ for all $j > k$. One can then show that
	\begin{align*}
	    \sum_{i=1}^N \eta_i [u_i - \tau^\star]_+ = \sum_{i=1}^k \eta_{\sigma(i)} (u_{\sigma(i)} - \tau^\star) = 2,
	\end{align*}
	as desired. Therefore, problem~\eqref{eq:spmax} is indeed solved by $p^\star_{i} = \eta_i [u_i - \tau^\star]_+ / 2$, $i\in[N]$.
	\end{proof}
	\bibliographystyle{abbrvnat} 
	\bibliography{references}

\begin{thebibliography}{172}
\providecommand{\natexlab}[1]{#1}
\providecommand{\url}[1]{\texttt{#1}}
\expandafter\ifx\csname urlstyle\endcsname\relax
  \providecommand{\doi}[1]{doi: #1}\else
  \providecommand{\doi}{doi: \begingroup \urlstyle{rm}\Url}\fi

\bibitem[Abid and Gower(2018)]{abid2018greedy}
B.~K. Abid and R.~Gower.
\newblock Stochastic algorithms for entropy-regularized optimal transport
  problems.
\newblock In \emph{Artificial Intelligence and Statistics}, pages 1505--1512,
  2018.

\bibitem[Adler et~al.(2017)Adler, Ringh, {\"O}ktem, and
  Karlsson]{adler2017learning}
J.~Adler, A.~Ringh, O.~{\"O}ktem, and J.~Karlsson.
\newblock Learning to solve inverse problems using {W}asserstein loss.
\newblock \emph{arXiv:1710.10898}, 2017.

\bibitem[Ahipa{\c{s}}ao{\u{g}}lu et~al.(2016)Ahipa{\c{s}}ao{\u{g}}lu,
  Ar{\i}kan, and Natarajan]{ahipacsaouglu2016flexibility}
S.~D. Ahipa{\c{s}}ao{\u{g}}lu, U.~Ar{\i}kan, and K.~Natarajan.
\newblock On the {f}lexibility of {u}sing {m}arginal {d}istribution {c}hoice
  {m}odels in {t}raffic {e}quilibrium.
\newblock \emph{Transportation Research Part B: Methodological}, 91:\penalty0
  130--158, 2016.

\bibitem[Ahipa{\c{s}}ao{\u{g}}lu et~al.(2018)Ahipa{\c{s}}ao{\u{g}}lu, Li, and
  Natarajan]{ahipasaoglu2018convex}
S.~D. Ahipa{\c{s}}ao{\u{g}}lu, X.~Li, and K.~Natarajan.
\newblock A convex optimization approach for computing correlated choice
  probabilities with many alternatives.
\newblock \emph{IEEE Transactions on Automatic Control}, 64\penalty0
  (1):\penalty0 190--205, 2018.

\bibitem[Ajalloeian and Stich(2020)]{ajalloeian2020analysis}
A.~Ajalloeian and S.~U. Stich.
\newblock Analysis of {SGD} with biased gradient estimators.
\newblock \emph{arXiv:2008.00051}, 2020.

\bibitem[Altschuler et~al.(2017)Altschuler, Weed, and
  Rigollet]{altschuler2017near}
J.~Altschuler, J.~Weed, and P.~Rigollet.
\newblock Near-linear time approximation algorithms for optimal transport via
  {S}inkhorn iteration.
\newblock In \emph{Advances in Neural Information Processing Systems}, pages
  1964--1974, 2017.

\bibitem[Altschuler et~al.(2022)Altschuler, Niles-Weed, and
  Stromme]{ref:altschuler2022asymptotics}
J.~M. Altschuler, J.~Niles-Weed, and A.~J. Stromme.
\newblock Asymptotics for semidiscrete entropic optimal transport.
\newblock \emph{SIAM Journal on Mathematical Analysis}, 54\penalty0
  (2):\penalty0 1718--1741, 2022.

\bibitem[Alvarez-Melis et~al.(2018)Alvarez-Melis, Jaakkola, and
  Jegelka]{alvarez2017structured}
D.~Alvarez-Melis, T.~Jaakkola, and S.~Jegelka.
\newblock Structured optimal transport.
\newblock In \emph{Artificial Intelligence and Statistics}, pages 1771--1780,
  2018.

\bibitem[Ambrogioni et~al.(2018)Ambrogioni, Guclu, Gucluturk, and van
  Gerven]{ambrogioni2018wasserstein}
L.~Ambrogioni, U.~Guclu, Y.~Gucluturk, and M.~van Gerven.
\newblock Wasserstein variational gradient descent: From semi-discrete optimal
  transport to ensemble variational inference.
\newblock \emph{arXiv:1811.02827}, 2018.

\bibitem[Anderson et~al.(1988)Anderson, De~Palma, and
  Thisse]{anderson1988representative}
S.~P. Anderson, A.~De~Palma, and J.-F. Thisse.
\newblock A representative consumer theory of the logit model.
\newblock \emph{International Economic Review}, 29\penalty0 (3):\penalty0
  461--466, 1988.

\bibitem[Arjovsky et~al.(2017)Arjovsky, Chintala, and
  Bottou]{arjovsky2017wasserstein}
M.~Arjovsky, S.~Chintala, and L.~Bottou.
\newblock {W}asserstein generative adversarial networks.
\newblock In \emph{International Conference on Machine Learning}, pages
  214--223, 2017.

\bibitem[Aurenhammer et~al.(1998)Aurenhammer, Hoffmann, and
  Aronov]{aurenhammer1998minkowski}
F.~Aurenhammer, F.~Hoffmann, and B.~Aronov.
\newblock Minkowski-type theorems and least-squares clustering.
\newblock \emph{Algorithmica}, 20\penalty0 (1):\penalty0 61--76, 1998.

\bibitem[Bach(2010)]{bach2010self}
F.~Bach.
\newblock Self-concordant analysis for logistic regression.
\newblock \emph{Electronic Journal of Statistics}, 4:\penalty0 384--414, 2010.

\bibitem[Bach(2014)]{bach2013adaptivity}
F.~Bach.
\newblock Adaptivity of averaged stochastic gradient descent to local strong
  convexity for logistic regression.
\newblock \emph{Journal of Machine Learning Research}, 15\penalty0
  (19):\penalty0 595--627, 2014.

\bibitem[Bach and Moulines(2013)]{bach2013non}
F.~Bach and E.~Moulines.
\newblock {Non-strongly-convex smooth stochastic approximation with convergence
  rate $O(1/n)$}.
\newblock In \emph{Advances in Neural Information Processing Systems}, pages
  773--781, 2013.

\bibitem[Ben-Akiva and Lerman(1985)]{ben1985discrete}
M.~E. Ben-Akiva and S.~R. Lerman.
\newblock \emph{Discrete Choice Analysis: {T}heory and {A}pplication to
  {T}ravel {D}emand}.
\newblock MIT Press, 1985.

\bibitem[Benamou and Brenier(2000)]{benamou2000computational}
J.-D. Benamou and Y.~Brenier.
\newblock A computational fluid mechanics solution to the {M}onge-{K}antorovich
  mass transfer problem.
\newblock \emph{Numerische Mathematik}, 84\penalty0 (3):\penalty0 375--393,
  2000.

\bibitem[Benamou et~al.(2015)Benamou, Carlier, Cuturi, Nenna, and
  Peyr{\'e}]{benamou2015iterative}
J.-D. Benamou, G.~Carlier, M.~Cuturi, L.~Nenna, and G.~Peyr{\'e}.
\newblock Iterative {B}regman projections for regularized transportation
  problems.
\newblock \emph{SIAM Journal on Scientific Computing}, 37\penalty0
  (2):\penalty0 A1111--A1138, 2015.

\bibitem[Bertsekas(1981)]{bertsekas1981new}
D.~P. Bertsekas.
\newblock A new algorithm for the assignment problem.
\newblock \emph{Mathematical Programming}, 21\penalty0 (1):\penalty0 152--171,
  1981.

\bibitem[Bertsekas(1992)]{bertsekas1992auction}
D.~P. Bertsekas.
\newblock Auction algorithms for network flow problems: {A} tutorial
  introduction.
\newblock \emph{Computational Optimization and Applications}, 1\penalty0
  (1):\penalty0 7--66, 1992.

\bibitem[Bertsimas and Tsitsiklis(1997)]{bertsimas1997introduction}
D.~Bertsimas and J.~N. Tsitsiklis.
\newblock \emph{Introduction to {L}inear {O}ptimization}.
\newblock Athena Scientific Belmont, 1997.

\bibitem[Billingsley(1995)]{billingsley1995}
P.~Billingsley.
\newblock \emph{Probability and Measure}.
\newblock John Wiley and Sons, 1995.

\bibitem[Blanchet et~al.(2018)Blanchet, Jambulapati, Kent, and
  Sidford]{blanchet2018towards}
J.~Blanchet, A.~Jambulapati, C.~Kent, and A.~Sidford.
\newblock Towards {o}ptimal {r}unning {t}imes for {o}ptimal {t}ransport.
\newblock \emph{arXiv:1810.07717}, 2018.

\bibitem[Blondel et~al.(2018)Blondel, Seguy, and Rolet]{blondel2017smooth}
M.~Blondel, V.~Seguy, and A.~Rolet.
\newblock Smooth and sparse optimal transport.
\newblock In \emph{Artificial Intelligence and Statistics}, pages 880--889,
  2018.

\bibitem[Bonnotte(2013)]{bonnotte2013knothe}
N.~Bonnotte.
\newblock From {K}nothe's rearrangement to {B}renier's optimal transport map.
\newblock \emph{SIAM Journal on Mathematical Analysis}, 45\penalty0
  (1):\penalty0 64--87, 2013.

\bibitem[Boucheron et~al.(2013)Boucheron, Lugosi, and
  Massart]{boucheron2013concentration}
S.~Boucheron, G.~Lugosi, and P.~Massart.
\newblock \emph{Concentration {I}nequalities: {A} {N}onasymptotic {T}heory of
  {I}ndependence}.
\newblock Oxford University Press, 2013.

\bibitem[Brenier(1991)]{brenier1991polar}
Y.~Brenier.
\newblock Polar {f}actorization and {m}onotone {r}earrangement of
  {v}ector-{v}alued {f}unctions.
\newblock \emph{Communications on Pure and Applied Mathematics}, 44\penalty0
  (4):\penalty0 375--417, 1991.

\bibitem[Bubeck(2015)]{bubeck2015convex}
S.~Bubeck.
\newblock Convex optimization: Algorithms and complexity.
\newblock \emph{Foundations and Trends in Machine Learning}, 8\penalty0
  (3-4):\penalty0 231--357, 2015.

\bibitem[Cazelles et~al.(2018)Cazelles, Seguy, Bigot, Cuturi, and
  Papadakis]{cazelles2018geodesic}
E.~Cazelles, V.~Seguy, J.~Bigot, M.~Cuturi, and N.~Papadakis.
\newblock Geodesic {PCA} versus {l}og-{PCA} of {h}istograms in the
  {W}asserstein {s}pace.
\newblock \emph{SIAM Journal on Scientific Computing}, 40\penalty0
  (2):\penalty0 B429--B456, 2018.

\bibitem[Chakrabarty and Khanna(2020)]{chakrabarty2018better}
D.~Chakrabarty and S.~Khanna.
\newblock Better and {s}impler {e}rror {a}nalysis of the {S}inkhorn-{K}nopp
  {a}lgorithm for {m}atrix {s}caling.
\newblock \emph{Mathematical Programming}, pages 1--13, 2020.
\newblock Forthcoming.

\bibitem[Chizat et~al.(2018)Chizat, Peyr{\'e}, Schmitzer, and
  Vialard]{chizat2016scaling}
L.~Chizat, G.~Peyr{\'e}, B.~Schmitzer, and F.-X. Vialard.
\newblock Scaling {a}lgorithms for {u}nbalanced {o}ptimal {t}ransport
  {p}roblems.
\newblock \emph{Mathematics of Computation}, 87\penalty0 (314):\penalty0
  2563--2609, 2018.

\bibitem[Chizat et~al.(2020)Chizat, Roussillon, L{\'e}ger, Vialard, and
  Peyr{\'e}]{ref:chizat2020faster}
L.~Chizat, P.~Roussillon, F.~L{\'e}ger, F.-X. Vialard, and G.~Peyr{\'e}.
\newblock Faster {W}asserstein distance estimation with the {S}inkhorn
  divergence.
\newblock \emph{Advances in Neural Information Processing Systems}, pages
  2257--2269, 2020.

\bibitem[Clason et~al.(2021)Clason, Lorenz, Mahler, and
  Wirth]{ref:clason2019entropic}
C.~Clason, D.~A. Lorenz, H.~Mahler, and B.~Wirth.
\newblock {E}ntropic {r}egularization of {c}ontinuous {o}ptimal {t}ransport
  {p}roblems.
\newblock \emph{Journal of Mathematical Analysis and Applications},
  494\penalty0 (1):\penalty0 124432, 2021.

\bibitem[Cohen et~al.(2018)Cohen, Diakonikolas, and
  Orecchia]{cohen2018acceleration}
M.~Cohen, J.~Diakonikolas, and L.~Orecchia.
\newblock On acceleration with noise-corrupted gradients.
\newblock In \emph{International Conference on Machine Learning}, pages
  1019--1028, 2018.

\bibitem[Cominetti and San~Mart{\'\i}n(1994)]{cominetti1994asymptotic}
R.~Cominetti and J.~San~Mart{\'\i}n.
\newblock Asymptotic {A}nalysis of the {E}xponential {P}enalty {T}rajectory in
  {L}inear {P}rogramming.
\newblock \emph{Mathematical Programming}, 67\penalty0 (1-3):\penalty0
  169--187, 1994.

\bibitem[Conforti and Tamanini(2021)]{ref:conforti2021formula}
G.~Conforti and L.~Tamanini.
\newblock A formula for the time derivative of the entropic cost and
  applications.
\newblock \emph{Journal of Functional Analysis}, 280\penalty0 (11), 2021.

\bibitem[Cormen et~al.(2009)Cormen, Leiserson, Rivest, and
  Stein]{cormen2009introduction}
T.~H. Cormen, C.~E. Leiserson, R.~L. Rivest, and C.~Stein.
\newblock \emph{Introduction to Algorithms}.
\newblock MIT Press, 2009.

\bibitem[Courty et~al.(2016)Courty, Flamary, Tuia, and
  Rakotomamonjy]{courty2016optimal}
N.~Courty, R.~Flamary, D.~Tuia, and A.~Rakotomamonjy.
\newblock Optimal transport for domain adaptation.
\newblock \emph{IEEE Transactions on Pattern Analysis and Machine
  Intelligence}, 39\penalty0 (9):\penalty0 1853--1865, 2016.

\bibitem[Cuturi(2013)]{sinkhorn}
M.~Cuturi.
\newblock Sinkhorn distances: {L}ightspeed computation of optimal transport.
\newblock In \emph{Advances in Neural Information Processing Systems}, pages
  2292--2300, 2013.

\bibitem[Daganzo(2014)]{daganzo2014multinomial}
C.~Daganzo.
\newblock \emph{Multinomial {P}robit: the {T}heory and its Application to
  {D}emand {F}orecasting}.
\newblock Elsevier, 2014.

\bibitem[d'Aspremont(2008)]{d2008smooth}
A.~d'Aspremont.
\newblock Smooth optimization with approximate gradient.
\newblock \emph{SIAM Journal on Optimization}, 19\penalty0 (3):\penalty0
  1171--1183, 2008.

\bibitem[De~Goes et~al.(2012)De~Goes, Breeden, Ostromoukhov, and
  Desbrun]{de2012blue}
F.~De~Goes, K.~Breeden, V.~Ostromoukhov, and M.~Desbrun.
\newblock Blue noise through optimal transport.
\newblock \emph{ACM Transactions on Graphics}, 31\penalty0 (6):\penalty0 171,
  2012.

\bibitem[de~Goes et~al.(2015)de~Goes, Wallez, Huang, Pavlov, and
  Desbrun]{de2015power}
F.~de~Goes, C.~Wallez, J.~Huang, D.~Pavlov, and M.~Desbrun.
\newblock Power particles: {A}n incompressible fluid solver based on power
  diagrams.
\newblock \emph{ACM Transactions on Graphics}, 34\penalty0 (4):\penalty0
  50:1--50:11, 2015.

\bibitem[De~la Fuente(2000)]{de2000mathematical}
A.~De~la Fuente.
\newblock \emph{Mathematical {M}ethods and {M}odels for {E}conomists}.
\newblock Cambridge University Press, 2000.

\bibitem[Dekel et~al.(2012)Dekel, Gilad-Bachrach, Shamir, and
  Xiao]{dekel2012optimal}
O.~Dekel, R.~Gilad-Bachrach, O.~Shamir, and L.~Xiao.
\newblock Optimal distributed online prediction using mini-batches.
\newblock \emph{Journal of Machine Learning Research}, 13:\penalty0 165--202,
  2012.

\bibitem[Delalande(2021)]{ref:delalande2021nearly}
A.~Delalande.
\newblock Nearly tight convergence bounds for semi-discrete entropic optimal
  transport.
\newblock \emph{arXiv:2110.12678}, 2021.

\bibitem[Dessein et~al.(2018)Dessein, Papadakis, and
  Rouas]{dessein2018regularized}
A.~Dessein, N.~Papadakis, and J.-L. Rouas.
\newblock Regularized optimal transport and the rot mover's distance.
\newblock \emph{Journal of Machine Learning Research}, 19\penalty0
  (1):\penalty0 590--642, 2018.

\bibitem[Dick et~al.(2013)Dick, Kuo, and Sloan]{dick2013high}
J.~Dick, F.~Y. Kuo, and I.~H. Sloan.
\newblock High-dimensional integration: {T}he quasi-{M}onte {C}arlo way.
\newblock \emph{Acta Numerica}, 22:\penalty0 133--288, 2013.

\bibitem[Dubin and McFadden(1984)]{dubin1984econometric}
J.~A. Dubin and D.~L. McFadden.
\newblock An {e}conometric {a}nalysis of {r}esidential {e}lectric {a}ppliance
  {h}oldings and {c}onsumption.
\newblock \emph{Econometrica}, 52\penalty0 (2):\penalty0 345--362, 1984.

\bibitem[Duchi and Singer(2009)]{duchi2009efficient}
J.~Duchi and Y.~Singer.
\newblock Efficient online and batch learning using forward backward splitting.
\newblock \emph{The Journal of Machine Learning Research}, 10\penalty0
  (99):\penalty0 2899--2934, 2009.

\bibitem[Dvurechensky et~al.(2018)Dvurechensky, Gasnikov, and
  Kroshnin]{dvurechensky2018computational}
P.~Dvurechensky, A.~Gasnikov, and A.~Kroshnin.
\newblock Computational optimal transport: Complexity by accelerated gradient
  descent is better than by {S}inkhorn’s algorithm.
\newblock In \emph{International Conference on Machine Learning}, pages
  1367--1376, 2018.

\bibitem[Dyer and Frieze(1988)]{dyer1988complexity}
M.~E. Dyer and A.~M. Frieze.
\newblock On the complexity of computing the volume of a polyhedron.
\newblock \emph{SIAM Journal on Computing}, 17\penalty0 (5):\penalty0 967--974,
  1988.

\bibitem[Erbar et~al.(2015)Erbar, Maas, and Renger]{ref:erbar2015large}
M.~Erbar, J.~Maas, and M.~Renger.
\newblock From large deviations to {W}asserstein gradient flows in multiple
  dimensions.
\newblock \emph{Electronic Communications in Probability}, 20:\penalty0 1--12,
  2015.

\bibitem[Essid and Solomon(2018)]{essid2018quadratically}
M.~Essid and J.~Solomon.
\newblock Quadratically regularized optimal transport on graphs.
\newblock \emph{SIAM Journal on Scientific Computing}, 40\penalty0
  (4):\penalty0 A1961--A1986, 2018.

\bibitem[Evans(1997)]{evans1997partial}
L.~C. Evans.
\newblock Partial differential equations and {M}onge-{K}antorovich mass
  transfer.
\newblock \emph{Current Developments in Mathematics}, 1997\penalty0
  (1):\penalty0 65--126, 1997.

\bibitem[Fang(1992)]{fang1992unconstrained}
S.-C. Fang.
\newblock An unconstrained convex programming view of linear programming.
\newblock \emph{Zeitschrift f{\"u}r Operations Research}, 36\penalty0
  (2):\penalty0 149--161, 1992.

\bibitem[Feng et~al.(2017)Feng, Li, and Wang]{feng2017relation}
G.~Feng, X.~Li, and Z.~Wang.
\newblock On the relation between several discrete choice models.
\newblock \emph{Operations Research}, 65\penalty0 (6):\penalty0 1516--1525,
  2017.

\bibitem[Ferradans et~al.(2014)Ferradans, Papadakis, Peyr{\'e}, and
  Aujol]{ferradans2014regularized}
S.~Ferradans, N.~Papadakis, G.~Peyr{\'e}, and J.-F. Aujol.
\newblock Regularized {d}iscrete {o}ptimal {t}ransport.
\newblock \emph{SIAM Journal on Imaging Sciences}, 7\penalty0 (3):\penalty0
  1853--1882, 2014.

\bibitem[Feydy et~al.(2017)Feydy, Charlier, Vialard, and
  Peyr{\'e}]{feydy2017optimal}
J.~Feydy, B.~Charlier, F.-X. Vialard, and G.~Peyr{\'e}.
\newblock Optimal transport for diffeomorphic registration.
\newblock In \emph{International Conference on Medical Image Computing and
  Computer-Assisted Intervention}, pages 291--299, 2017.

\bibitem[Flamary et~al.(2018)Flamary, Cuturi, Courty, and
  Rakotomamonjy]{flamary2018wasserstein}
R.~Flamary, M.~Cuturi, N.~Courty, and A.~Rakotomamonjy.
\newblock Wasserstein discriminant analysis.
\newblock \emph{Machine Learning}, 107\penalty0 (12):\penalty0 1923--1945,
  2018.

\bibitem[F{\"o}llmer and Schied(2004)]{follmer2004stochastic}
H.~F{\"o}llmer and A.~Schied.
\newblock \emph{Stochastic Finance: {A}n Introduction in Discrete Time}.
\newblock Walter de Gruyter, 2004.

\bibitem[Fr\'echet(1951)]{frechet1951}
M.~Fr\'echet.
\newblock Sur les {t}ableaux de {c}orr\'elation dont les {m}arges sont
  {d}onn\'ees.
\newblock \emph{Annales de l'Universit\'e de Lyon, Sciences}, 4\penalty0
  (1/2):\penalty0 13--84, 1951.

\bibitem[Friedlander and Schmidt(2012)]{friedlander2012hybrid}
M.~P. Friedlander and M.~Schmidt.
\newblock Hybrid deterministic-stochastic methods for data fitting.
\newblock \emph{SIAM Journal on Scientific Computing}, 34\penalty0
  (3):\penalty0 A1380--A1405, 2012.

\bibitem[Genevay et~al.(2016)Genevay, Cuturi, Peyr{\'e}, and
  Bach]{genevay2016stochastic}
A.~Genevay, M.~Cuturi, G.~Peyr{\'e}, and F.~Bach.
\newblock Stochastic optimization for large-scale optimal transport.
\newblock In \emph{Advances in Neural Information Processing Systems}, pages
  3440--3448, 2016.

\bibitem[Genevay et~al.(2018)Genevay, Peyr{\'e}, and
  Cuturi]{genevay2017learning}
A.~Genevay, G.~Peyr{\'e}, and M.~Cuturi.
\newblock Learning generative models with {S}inkhorn divergences.
\newblock In \emph{Artificial Intelligence and Statistics}, pages 1608--1617,
  2018.

\bibitem[Ghai et~al.(2020)Ghai, Hazan, and Singer]{ghai2019exponentiated}
U.~Ghai, E.~Hazan, and Y.~Singer.
\newblock Exponentiated gradient meets gradient descent.
\newblock In \emph{International Conference on Algorithmic Learning Theory},
  pages 386--407, 2020.

\bibitem[Gordaliza et~al.(2019)Gordaliza, Barrio, Fabrice, and
  Loubes]{gordaliza2019obtaining}
P.~Gordaliza, E.~D. Barrio, G.~Fabrice, and J.-M. Loubes.
\newblock Obtaining fairness using optimal transport theory.
\newblock In \emph{International Conference on Machine Learning}, pages
  2357--2365, 2019.

\bibitem[Gulrajani et~al.(2017)Gulrajani, Ahmed, Arjovsky, Dumoulin, and
  Courville]{gulrajani2017improved}
I.~Gulrajani, F.~Ahmed, M.~Arjovsky, V.~Dumoulin, and A.~C. Courville.
\newblock Improved training of {W}asserstein {G}ans.
\newblock In \emph{Advances in Neural Information Processing Systems}, pages
  5767--5777, 2017.

\bibitem[Hackbarth and Madlener(2013)]{hackbarth2013consumer}
A.~Hackbarth and R.~Madlener.
\newblock Consumer preferences for alternative fuel vehicles: A discrete choice
  analysis.
\newblock \emph{Transportation Research Part D: Transport and Environment},
  25:\penalty0 5--17, 2013.

\bibitem[Hanasusanto et~al.(2016)Hanasusanto, Kuhn, and Wiesemann]{Grani}
G.~A. Hanasusanto, D.~Kuhn, and W.~Wiesemann.
\newblock A comment on ``computational complexity of stochastic programming
  problems''.
\newblock \emph{Mathematical Programming}, 159\penalty0 (1-2):\penalty0
  557--569, 2016.

\bibitem[Hazan et~al.(2014)Hazan, Koren, and Levy]{hazan2014logistic}
E.~Hazan, T.~Koren, and K.~Y. Levy.
\newblock Logistic regression: {T}ight bounds for stochastic and online
  optimization.
\newblock In \emph{Conference on Learning Theory}, pages 197--209, 2014.

\bibitem[Heitsch and R{\"o}misch(2007)]{heitsch2007note}
H.~Heitsch and W.~R{\"o}misch.
\newblock A {n}ote on {s}cenario {r}eduction for {t}wo-{s}tage {s}tochastic
  {p}rograms.
\newblock \emph{Operations Research Letters}, 35\penalty0 (6):\penalty0
  731--738, 2007.

\bibitem[Ho et~al.(2017)Ho, Nguyen, Yurochkin, Bui, Huynh, and
  Phung]{ho2017multilevel}
N.~Ho, X.~Nguyen, M.~Yurochkin, H.~H. Bui, V.~Huynh, and D.~Phung.
\newblock Multilevel clustering via {W}asserstein means.
\newblock In \emph{International Conference on Machine Learning}, pages
  1501--1509, 2017.

\bibitem[Hochreiter and Pflug(2007)]{hochreiter2007financial}
R.~Hochreiter and G.~C. Pflug.
\newblock Financial {s}cenario {g}eneration for {s}tochastic {m}ulti-{s}tage
  {d}ecision {p}rocesses as {f}acility {l}ocation {p}roblems.
\newblock \emph{Annals of Operations Research}, 152\penalty0 (1):\penalty0
  257--272, 2007.

\bibitem[Hoffman(1981)]{hoffman1981method}
K.~L. Hoffman.
\newblock A {m}ethod for {g}lobally {m}inimizing {c}oncave {f}unctions over
  {c}onvex {s}ets.
\newblock \emph{Mathematical Programming}, 20\penalty0 (1):\penalty0 22--32,
  1981.

\bibitem[Hu et~al.(2020)Hu, Seiler, and Lessard]{hu2020analysis}
B.~Hu, P.~Seiler, and L.~Lessard.
\newblock Analysis of biased stochastic gradient descent using sequential
  semidefinite programs.
\newblock \emph{Mathematical Programming}, pages 1--26, 2020.
\newblock Forthcoming.

\bibitem[Jambulapati et~al.(2019)Jambulapati, Sidford, and
  Tian]{jambulapati2019direct}
A.~Jambulapati, A.~Sidford, and K.~Tian.
\newblock A direct {$\tilde{\mathcal{O}}(1/e)$} iteration parallel algorithm
  for optimal transport.
\newblock In \emph{Advances in Neural Information Processing Systems}, pages
  11359--11370, 2019.

\bibitem[Kakade et~al.(2009)Kakade, Shalev-Shwartz, and
  Tewari]{kakade2009duality}
S.~Kakade, S.~Shalev-Shwartz, and A.~Tewari.
\newblock On the {d}uality of {s}trong {c}onvexity and {s}trong {s}moothness:
  {L}earning {a}pplications and {m}atrix {r}egularization.
\newblock Technical report, Toyota Technological Institute, 2009.

\bibitem[Kantorovich(1942)]{kantorovich1942transfer}
L.~Kantorovich.
\newblock On the transfer of masses (in {R}ussian).
\newblock \emph{Doklady Akademii Nauk}, 37\penalty0 (2):\penalty0 227--229,
  1942.

\bibitem[Karlsson and Ringh(2017)]{karlsson2017generalized}
J.~Karlsson and A.~Ringh.
\newblock Generalized {S}inkhorn iterations for regularizing inverse problems
  using optimal mass transport.
\newblock \emph{SIAM Journal on Imaging Sciences}, 10\penalty0 (4):\penalty0
  1935--1962, 2017.

\bibitem[Kavis et~al.(2019)Kavis, Levy, Bach, and Cevher]{kavis2019unixgrad}
A.~Kavis, K.~Y. Levy, F.~Bach, and V.~Cevher.
\newblock {UniXGrad: A universal, adaptive algorithm with optimal guarantees
  for constrained optimization}.
\newblock In \emph{Advances in Neural Information Processing Systems}, pages
  6257--6266, 2019.

\bibitem[Kitagawa et~al.(2016)Kitagawa, M{\'e}rigot, and
  Thibert]{kitagawa2016convergence}
J.~Kitagawa, Q.~M{\'e}rigot, and B.~Thibert.
\newblock Convergence of a {N}ewton algorithm for semi-discrete optimal
  transport.
\newblock \emph{arXiv:1603.05579}, 2016.

\bibitem[Kolouri and Rohde(2015)]{kolouri2015transport}
S.~Kolouri and G.~K. Rohde.
\newblock Transport-based single frame super resolution of very low resolution
  face images.
\newblock In \emph{IEEE Conference on Computer Vision and Pattern Recognition},
  pages 4876--4884, 2015.

\bibitem[Kolouri et~al.(2017)Kolouri, Park, Thorpe, Slepcev, and
  Rohde]{kolouri2017optimal}
S.~Kolouri, S.~R. Park, M.~Thorpe, D.~Slepcev, and G.~K. Rohde.
\newblock Optimal {m}ass {t}ransport: {S}ignal {p}rocessing and
  {m}achine-{l}earning {a}pplications.
\newblock \emph{IEEE Signal Processing Magazine}, 34\penalty0 (4):\penalty0
  43--59, 2017.

\bibitem[Kuhn(1955)]{kuhn1955hungarian}
H.~W. Kuhn.
\newblock The {H}ungarian method for the assignment problem.
\newblock \emph{Naval Research Logistics Quarterly}, 2\penalty0 (1-2):\penalty0
  83--97, 1955.

\bibitem[Kundu et~al.(2018)Kundu, Kolouri, Erickson, Kramer, McAuley, and
  Rohde]{kundu2018discovery}
S.~Kundu, S.~Kolouri, K.~I. Erickson, A.~F. Kramer, E.~McAuley, and G.~K.
  Rohde.
\newblock Discovery and visualization of structural biomarkers from {MRI} using
  transport-based morphometry.
\newblock \emph{NeuroImage}, 167:\penalty0 256--275, 2018.

\bibitem[Lacoste-Julien et~al.(2012)Lacoste-Julien, Schmidt, and
  Bach]{lacoste2012simpler}
S.~Lacoste-Julien, M.~Schmidt, and F.~Bach.
\newblock A simpler approach to obtaining an {$\mathcal O (1/t)$} convergence
  rate for the projected stochastic subgradient method.
\newblock \emph{arXiv:1212.2002}, 2012.

\bibitem[Lan(2012)]{lan2012optimal}
G.~Lan.
\newblock An optimal method for stochastic composite optimization.
\newblock \emph{Mathematical Programming}, 133\penalty0 (1-2):\penalty0
  365--397, 2012.

\bibitem[Lee and Sidford(2014)]{lee2014path}
Y.~T. Lee and A.~Sidford.
\newblock Path finding methods for linear programming: Solving linear programs
  in {$\tilde{\mathcal{O}}(\sqrt{rank})$} iterations and faster algorithms for
  maximum flow.
\newblock In \emph{IEEE Symposium on Foundations of Computer Science}, pages
  424--433, 2014.

\bibitem[L{\'e}vy(2015)]{levy2015numerical}
B.~L{\'e}vy.
\newblock A numerical algorithm for {$L_2$} semi-discrete optimal transport in
  3{D}.
\newblock \emph{ESAIM: Mathematical Modelling and Numerical Analysis},
  49\penalty0 (6):\penalty0 1693--1715, 2015.

\bibitem[Li et~al.(2019)Li, Webster, Mason, and Kempf]{qi2019product}
H.~Li, S.~Webster, N.~Mason, and K.~Kempf.
\newblock Product-{l}ine {p}ricing {u}nder {d}iscrete {m}ixed {m}ultinomial
  {l}ogit {d}emand.
\newblock \emph{Manufacturing and Service Operations Management}, 21:\penalty0
  14--28, 2019.

\bibitem[Li et~al.(2016)Li, Osher, and Gangbo]{li2016fast}
W.~Li, S.~Osher, and W.~Gangbo.
\newblock A fast algorithm for earth mover's distance based on optimal
  transport and ${l_1}$ type regularization.
\newblock \emph{arXiv:1609.07092}, 2016.

\bibitem[Lin et~al.(2019{\natexlab{a}})Lin, Ho, and
  Jordan]{lin2019acceleration}
T.~Lin, N.~Ho, and M.~I. Jordan.
\newblock On the {e}fficiency of the {S}inkhorn and {G}reenkhorn {a}lgorithms
  for {o}ptimal {t}ransport.
\newblock \emph{arXiv:1906.01437}, 2019{\natexlab{a}}.

\bibitem[Lin et~al.(2019{\natexlab{b}})Lin, Ho, and Jordan]{lin2019efficient}
T.~Lin, N.~Ho, and M.~I. Jordan.
\newblock On efficient optimal transport: {A}n analysis of greedy and
  accelerated mirror descent algorithms.
\newblock In \emph{International Conference on Machine Learning}, pages
  3982--3991, 2019{\natexlab{b}}.

\bibitem[{L\"ofberg}(2004)]{yalmip}
J.~{L\"ofberg}.
\newblock {YALMIP}: {A} toolbox for modeling and optimization in {MATLAB}.
\newblock In \emph{IEEE International Conference on Robotics and Automation},
  pages 284--289, 2004.

\bibitem[Luo and Tseng(1993)]{luo1993error}
Z.-Q. Luo and P.~Tseng.
\newblock Error bounds and convergence analysis of feasible descent methods:
  {A} general approach.
\newblock \emph{Annals of Operations Research}, 46\penalty0 (1):\penalty0
  157--178, 1993.

\bibitem[Mak et~al.(2015)Mak, Rong, and Zhang]{mak2015appointment}
H.-Y. Mak, Y.~Rong, and J.~Zhang.
\newblock Appointment {s}cheduling with {l}imited {d}istributional
  {i}nformation.
\newblock \emph{Management Science}, 61\penalty0 (2):\penalty0 316--334, 2015.

\bibitem[Martins and Astudillo(2016)]{sparsemax}
A.~Martins and R.~Astudillo.
\newblock From softmax to sparsemax: A sparse model of attention and
  multi-label classification.
\newblock In \emph{International Conference on Machine Learning}, pages
  1614--1623, 2016.

\bibitem[McFadden(1974)]{mcfadden1973conditional}
D.~McFadden.
\newblock Conditional {l}ogit {a}nalysis of {q}ualitative {c}hoice {b}ehavior.
\newblock In P.~Zarembka, editor, \emph{Frontiers in Econometrics}, pages
  105--142. Academic Press, 1974.

\bibitem[McFadden(1978)]{mcfadden1978modeling}
D.~McFadden.
\newblock Modeling the {c}hoice of {r}esidential {l}ocation.
\newblock \emph{Transportation Research Record}, 673:\penalty0 72--77, 1978.

\bibitem[McFadden(1981)]{mcfadden1981econometric}
D.~McFadden.
\newblock Econometric models of probabilistic choice.
\newblock In C.~Manski and D.~McFadden, editors, \emph{Structural Analysis of
  Discrete Data with Econometric Application}, pages 198--272. MIT Press, 1981.

\bibitem[M{\'e}rigot(2011)]{merigot2011multiscale}
Q.~M{\'e}rigot.
\newblock A multiscale approach to optimal transport.
\newblock \emph{Computer Graphics Forum}, 5\penalty0 (30):\penalty0 1583--1592,
  2011.

\bibitem[Mirebeau(2015)]{ref:mirebeau2015discretization}
J.-M. Mirebeau.
\newblock Discretization of the 3{D} {M}onge-{A}mp\`ere operator, between wide
  stencils and power diagrams.
\newblock \emph{Mathematical Modelling and Numerical Analysis}, 49\penalty0
  (5):\penalty0 1511--1523, 2015.

\bibitem[Mishra et~al.(2012)Mishra, Natarajan, Tao, and Teo]{mishra2012choice}
V.~K. Mishra, K.~Natarajan, H.~Tao, and C.-P. Teo.
\newblock Choice {p}rediction with {s}emidefinite {o}ptimization when
  {u}tilities are {c}orrelated.
\newblock \emph{IEEE Transactions on Automatic Control}, 57\penalty0
  (10):\penalty0 2450--2463, 2012.

\bibitem[Mishra et~al.(2014)Mishra, Natarajan, Padmanabhan, Teo, and
  Li]{mishra2014theoretical}
V.~K. Mishra, K.~Natarajan, D.~Padmanabhan, C.-P. Teo, and X.~Li.
\newblock On {t}heoretical and {e}mpirical {a}spects of {m}arginal
  {d}istribution {c}hoice {m}odels.
\newblock \emph{Management Science}, 60\penalty0 (6):\penalty0 1511--1531,
  2014.

\bibitem[Mohajerin~Esfahani and Kuhn(2018)]{esfahani2018data}
P.~Mohajerin~Esfahani and D.~Kuhn.
\newblock Data-driven distributionally robust optimization using the
  {W}asserstein metric: Performance guarantees and tractable reformulations.
\newblock \emph{Mathematical Programming}, 171\penalty0 (1-2):\penalty0
  115--166, 2018.

\bibitem[Monge(1781)]{monge1781memoire}
G.~Monge.
\newblock M{\'e}moire sur la {t}h{\'e}orie des d{\'e}blais et des {r}emblais.
\newblock \emph{Histoire de l'Acad{\'e}mie Royale des Sciences de Paris}, 1781.

\bibitem[Moulines and Bach(2011)]{moulines2011non}
E.~Moulines and F.~Bach.
\newblock Non-asymptotic analysis of stochastic approximation algorithms for
  machine learning.
\newblock In \emph{Advances in Neural Information Processing Systems}, pages
  451--459, 2011.

\bibitem[Murez et~al.(2018)Murez, Kolouri, Kriegman, Ramamoorthi, and
  Kim]{murez2018image}
Z.~Murez, S.~Kolouri, D.~Kriegman, R.~Ramamoorthi, and K.~Kim.
\newblock Image to image translation for domain adaptation.
\newblock In \emph{IEEE Conference on Computer Vision and Pattern Recognition},
  pages 4500--4509, 2018.

\bibitem[Muzellec et~al.(2017)Muzellec, Nock, Patrini, and
  Nielsen]{muzellec2017tsallis}
B.~Muzellec, R.~Nock, G.~Patrini, and F.~Nielsen.
\newblock Tsallis regularized optimal transport and ecological inference.
\newblock In \emph{Association for the Advancement of Artificial Intelligence},
  pages 2387--2393, 2017.

\bibitem[Natarajan et~al.(2009)Natarajan, Song, and
  Teo]{natarajan2009persistency}
K.~Natarajan, M.~Song, and C.-P. Teo.
\newblock Persistency model and its applications in choice modeling.
\newblock \emph{Management Science}, 55\penalty0 (3):\penalty0 453--469, 2009.

\bibitem[Nedi{\'c} and Bertsekas(2000)]{nedic2001convergence}
A.~Nedi{\'c} and D.~Bertsekas.
\newblock Convergence rate of incremental subgradient algorithms.
\newblock In S.~Uryasev and P.~M. Pardalos, editors, \emph{Stochastic
  Optimization: Algorithms and Applications}, pages 263--304. Kluwer Academic
  Publishers, 2000.

\bibitem[Nemirovski et~al.(2009)Nemirovski, Juditsky, Lan, and
  Shapiro]{nemirovski2009robust}
A.~Nemirovski, A.~Juditsky, G.~Lan, and A.~Shapiro.
\newblock Robust stochastic approximation approach to stochastic programming.
\newblock \emph{SIAM Journal on Optimization}, 19\penalty0 (4):\penalty0
  1574--1609, 2009.

\bibitem[Nesterov(1983)]{Nesterov1983AMF}
Y.~Nesterov.
\newblock A method for solving the convex programming problem with convergence
  rate {$\mc O (1/k^2)$}.
\newblock \emph{Proceedings of the USSR Academy of Sciences}, 269:\penalty0
  543--547, 1983.

\bibitem[Nesterov and Nemirovskii(1994)]{nesterov1994interior}
Y.~Nesterov and A.~Nemirovskii.
\newblock \emph{Interior-{P}oint {P}olynomial {A}lgorithms in {C}onvex
  {P}rogramming}.
\newblock SIAM, 1994.

\bibitem[Nesterov and Vial(2008)]{nesterov2008confidence}
Y.~Nesterov and J.~P. Vial.
\newblock Confidence level solutions for stochastic programming.
\newblock \emph{Automatica}, 44\penalty0 (6):\penalty0 1559--1568, 2008.

\bibitem[Nguyen et~al.(2020)Nguyen, Zhang, Blanchet, Delage, and
  Ye]{nguyen2020distributionally}
V.~A. Nguyen, F.~Zhang, J.~Blanchet, E.~Delage, and Y.~Ye.
\newblock Distributionally robust local non-parametric conditional estimation.
\newblock In \emph{Advances in Neural Information Processing Systems}, 2020.

\bibitem[Nguyen et~al.(2013)]{nguyen2013convergence}
X.~Nguyen et~al.
\newblock Convergence of {l}atent {m}ixing {m}easures in {f}inite and
  {i}nfinite {m}ixture {m}odels.
\newblock \emph{The Annals of Statistics}, 41\penalty0 (1):\penalty0 370--400,
  2013.

\bibitem[Orlin(1997)]{orlin1997polynomial}
J.~B. Orlin.
\newblock A polynomial time primal network simplex algorithm for minimum cost
  flows.
\newblock \emph{Mathematical Programming}, 78\penalty0 (2):\penalty0 109--129,
  1997.

\bibitem[Pal(2019)]{ref:pal2019difference}
S.~Pal.
\newblock On the difference between entropic cost and the optimal transport
  cost.
\newblock \emph{arXiv preprint arXiv:1905.12206}, 2019.

\bibitem[Papadakis and Rabin(2017)]{papadakis2017convex}
N.~Papadakis and J.~Rabin.
\newblock Convex {h}istogram-{b}ased {j}oint {i}mage {s}egmentation with
  {r}egularized {o}ptimal {t}ransport {c}ost.
\newblock \emph{Journal of Mathematical Imaging and Vision}, 59\penalty0
  (2):\penalty0 161--186, 2017.

\bibitem[Papadakis et~al.(2014)Papadakis, Peyr{\'e}, and
  Oudet]{papadakis2014optimal}
N.~Papadakis, G.~Peyr{\'e}, and E.~Oudet.
\newblock Optimal transport with proximal splitting.
\newblock \emph{SIAM Journal on Imaging Sciences}, 7\penalty0 (1):\penalty0
  212--238, 2014.

\bibitem[Paty and Cuturi(2020)]{paty2020regularized}
F.-P. Paty and M.~Cuturi.
\newblock Regularized optimal transport is ground cost adversarial.
\newblock In \emph{International Conference on Machine Learning}, pages
  7532--7542. PMLR, 2020.

\bibitem[Pele and Werman(2008)]{pele2008linear}
O.~Pele and M.~Werman.
\newblock A linear time histogram metric for improved sift matching.
\newblock In \emph{European Conference on Computer Vision}, pages 495--508,
  2008.

\bibitem[Pele and Werman(2009)]{pele2009fast}
O.~Pele and M.~Werman.
\newblock Fast and robust earth mover's distances.
\newblock In \emph{IEEE International Conference on Computer Vision}, pages
  460--467, 2009.

\bibitem[Peyr{\'e}(2015)]{peyre2015entropic}
G.~Peyr{\'e}.
\newblock Entropic approximation of {W}asserstein gradient flows.
\newblock \emph{SIAM Journal on Imaging Sciences}, 8\penalty0 (4):\penalty0
  2323--2351, 2015.

\bibitem[Peyr{\'e} and Cuturi(2019)]{peyre2019computational}
G.~Peyr{\'e} and M.~Cuturi.
\newblock Computational optimal transport.
\newblock \emph{Foundations and Trends in Machine Learning}, 11\penalty0
  (5-6):\penalty0 355--607, 2019.

\bibitem[Peyr{\'e} et~al.(2017)Peyr{\'e}, Chizat, Vialard, and
  Solomon]{peyre2017quantum}
G.~Peyr{\'e}, L.~Chizat, F.-X. Vialard, and J.~Solomon.
\newblock Quantum entropic regularization of matrix-valued optimal transport.
\newblock \emph{European Journal of Applied Mathematics}, pages 1--24, 2017.

\bibitem[Pflug(2001)]{pflug2001scenario}
G.~C. Pflug.
\newblock Scenario {t}ree {g}eneration for {m}ultiperiod {f}inancial
  {o}ptimization by {o}ptimal {d}iscretization.
\newblock \emph{Mathematical Programming}, 89\penalty0 (2):\penalty0 251--271,
  2001.

\bibitem[Pinelis(1994)]{pinelis1994optimum}
I.~Pinelis.
\newblock Optimum bounds for the distributions of martingales in {B}anach
  spaces.
\newblock \emph{The Annals of Probability}, 22\penalty0 (4):\penalty0
  1679--1706, 1994.

\bibitem[Piti{\'e} et~al.(2007)Piti{\'e}, Kokaram, and
  Dahyot]{pitie2007automated}
F.~Piti{\'e}, A.~C. Kokaram, and R.~Dahyot.
\newblock Automated colour grading using colour distribution transfer.
\newblock \emph{Computer Vision and Image Understanding}, 107\penalty0
  (1-2):\penalty0 123--137, 2007.

\bibitem[Polyak and Juditsky(1992)]{polyak1992acceleration}
B.~T. Polyak and A.~B. Juditsky.
\newblock Acceleration of stochastic approximation by averaging.
\newblock \emph{SIAM Journal on Control and Optimization}, 30\penalty0
  (4):\penalty0 838--855, 1992.

\bibitem[Qin et~al.(2017)Qin, Chen, He, and Chen]{qin2017wasserstein}
H.~Qin, Y.~Chen, J.~He, and B.~Chen.
\newblock Wasserstein blue noise sampling.
\newblock \emph{ACM Transactions on Graphics}, 36\penalty0 (4):\penalty0 1--14,
  2017.

\bibitem[Quanrud(2019)]{quanrud2018approximating}
K.~Quanrud.
\newblock Approximating optimal transport with linear programs.
\newblock In \emph{Symposium on Simplicity in Algorithms}, pages 6:1--6:9,
  2019.

\bibitem[Rigollet and Weed(2018)]{rigollet2018entropic}
P.~Rigollet and J.~Weed.
\newblock Entropic {o}ptimal {t}ransport is {m}aximum-{l}ikelihood
  {d}econvolution.
\newblock \emph{Comptes Rendus Mathematique}, 356\penalty0 (11-12):\penalty0
  1228--1235, 2018.

\bibitem[Robbins and Monro(1951)]{robbins1951stochastic}
H.~Robbins and S.~Monro.
\newblock A stochastic approximation method.
\newblock \emph{The Annals of Mathematical Statistics}, 22\penalty0
  (3):\penalty0 400--407, 1951.

\bibitem[Rockafellar(1974)]{rockafellar1974conjugate}
R.~T. Rockafellar.
\newblock \emph{Conjugate {D}uality and {O}ptimization}.
\newblock SIAM, 1974.

\bibitem[Rockafellar and Wets(2009)]{rockafellar2009variational}
R.~T. Rockafellar and R.~J.-B. Wets.
\newblock \emph{Variational {A}nalysis}.
\newblock Springer Science \& Business Media, 2009.

\bibitem[Rolet et~al.(2016)Rolet, Cuturi, and Peyr{\'e}]{rolet2016fast}
A.~Rolet, M.~Cuturi, and G.~Peyr{\'e}.
\newblock Fast dictionary learning with a smoothed {W}asserstein loss.
\newblock In \emph{Artificial Intelligence and Statistics}, pages 630--638,
  2016.

\bibitem[Rubner et~al.(2000)Rubner, Tomasi, and Guibas]{rubner2000earth}
Y.~Rubner, C.~Tomasi, and L.~J. Guibas.
\newblock The {e}arth {m}over's {d}istance as a {m}etric for {i}mage
  {r}etrieval.
\newblock \emph{International Journal of Computer Vision}, 40\penalty0
  (2):\penalty0 99--121, 2000.

\bibitem[Rudin(1964)]{rudin1964principles}
W.~Rudin.
\newblock \emph{Principles of {M}athematical {A}nalysis}.
\newblock McGraw-Hill Education, 1964.

\bibitem[Rujeerapaiboon et~al.(2018)Rujeerapaiboon, Schindler, Kuhn, and
  Wiesemann]{rujeerapaiboon2018scenario}
N.~Rujeerapaiboon, K.~Schindler, D.~Kuhn, and W.~Wiesemann.
\newblock Scenario reduction revisited: {F}undamental limits and guarantees.
\newblock \emph{Mathematical Programming}, 2018.
\newblock Forthcoming.

\bibitem[Ruppert(1988)]{ruppert1988efficient}
D.~Ruppert.
\newblock Efficient estimations from a slowly convergent {R}obbins-{M}onro
  process.
\newblock Technical report, School of Operations Research and Industrial
  Engineering, Cornell University, 1988.

\bibitem[Schmidt et~al.(2011)Schmidt, Roux, and Bach]{schmidt2011convergence}
M.~Schmidt, N.~L. Roux, and F.~Bach.
\newblock Convergence rates of inexact proximal-gradient methods for convex
  optimization.
\newblock In \emph{Advances in Neural Information Processing Systems}, pages
  1458--1466, 2011.

\bibitem[Schmitzer(2016)]{schmitzer2016sparse}
B.~Schmitzer.
\newblock A {s}parse {m}ultiscale {a}lgorithm for {d}ense {o}ptimal
  {t}ransport.
\newblock \emph{Journal of Mathematical Imaging and Vision}, 56\penalty0
  (2):\penalty0 238--259, 2016.

\bibitem[Schrijver(1998)]{schrijver1998theory}
A.~Schrijver.
\newblock \emph{Theory of Linear and Integer Programming}.
\newblock John Wiley \& Sons, 1998.

\bibitem[Schr{\"o}dinger(1931)]{schrodinger1931umkehrung}
E.~Schr{\"o}dinger.
\newblock {\"Uber die Umkehrung der Naturgesetze}.
\newblock \emph{Sitzungsberichte der Preussischen Akademie der Wissenschaften.
  Physikalisch-Mathematische Klasse}, 144\penalty0 (3):\penalty0 144--153,
  1931.

\bibitem[Seguy and Cuturi(2015)]{seguy2015principal}
V.~Seguy and M.~Cuturi.
\newblock Principal geodesic analysis for probability measures under the
  optimal transport metric.
\newblock In \emph{Advances in Neural Information Processing Systems}, pages
  3312--3320, 2015.

\bibitem[Seguy et~al.(2018)Seguy, Damodaran, Flamary, Courty, Rolet, and
  Blondel]{seguy2017large}
V.~Seguy, B.~B. Damodaran, R.~Flamary, N.~Courty, A.~Rolet, and M.~Blondel.
\newblock Large-scale optimal transport and mapping estimation.
\newblock \emph{International Conference on Learning Representations}, 2018.

\bibitem[Shafieezadeh-Abadeh et~al.(2015)Shafieezadeh-Abadeh,
  Mohajerin~Esfahani, and Kuhn]{NIPS2015_5745}
S.~Shafieezadeh-Abadeh, P.~Mohajerin~Esfahani, and D.~Kuhn.
\newblock Distributionally robust logistic regression.
\newblock In \emph{Advances in Neural Information Processing Systems}, pages
  1576--1584, 2015.

\bibitem[Shafieezadeh-Abadeh et~al.(2019)Shafieezadeh-Abadeh, Kuhn, and
  Esfahani]{shafieezadeh2019regularization}
S.~Shafieezadeh-Abadeh, D.~Kuhn, and P.~M. Esfahani.
\newblock Regularization via mass transportation.
\newblock \emph{Journal of Machine Learning Research}, 20\penalty0
  (103):\penalty0 1--68, 2019.

\bibitem[Shalev-Shwartz et~al.(2009)Shalev-Shwartz, Shamir, Srebro, and
  Sridharan]{shalev2009stochastic}
S.~Shalev-Shwartz, O.~Shamir, N.~Srebro, and K.~Sridharan.
\newblock Stochastic convex optimization.
\newblock In \emph{Conference on Learning Theory}, 2009.

\bibitem[Shalev-Shwartz et~al.(2011)Shalev-Shwartz, Singer, Srebro, and
  Cotter]{shalev2011pegasos}
S.~Shalev-Shwartz, Y.~Singer, N.~Srebro, and A.~Cotter.
\newblock Pegasos: {P}rimal estimated sub-gradient solver for {SVM}.
\newblock \emph{Mathematical programming}, 127\penalty0 (1):\penalty0 3--30,
  2011.

\bibitem[Shapiro(2017)]{shapiro2017distributionally}
A.~Shapiro.
\newblock Distributionally robust stochastic programming.
\newblock \emph{SIAM Journal on Optimization}, 27\penalty0 (4):\penalty0
  2258--2275, 2017.

\bibitem[Sinkhorn(1967)]{sinkhorn1967diagonal}
R.~Sinkhorn.
\newblock Diagonal {e}quivalence to {m}atrices with {p}rescribed {r}ow and
  {c}olumn {s}ums.
\newblock \emph{The American Mathematical Monthly}, 74\penalty0 (4):\penalty0
  402--405, 1967.

\bibitem[Solomon et~al.(2014)Solomon, Rustamov, Guibas, and
  Butscher]{solomon2014earth}
J.~Solomon, R.~Rustamov, L.~Guibas, and A.~Butscher.
\newblock Earth {m}over's {d}istances on {d}iscrete {s}urfaces.
\newblock \emph{ACM Transactions on Graphics}, 33\penalty0 (4):\penalty0 67,
  2014.

\bibitem[Solomon et~al.(2015)Solomon, De~Goes, Peyr{\'e}, Cuturi, Butscher,
  Nguyen, Du, and Guibas]{solomon2015convolutional}
J.~Solomon, F.~De~Goes, G.~Peyr{\'e}, M.~Cuturi, A.~Butscher, A.~Nguyen, T.~Du,
  and L.~Guibas.
\newblock Convolutional {W}asserstein {d}istances: {E}fficient {o}ptimal
  {t}ransportation on {g}eometric {d}omains.
\newblock \emph{ACM Transactions on Graphics}, 34\penalty0 (4):\penalty0 66,
  2015.

\bibitem[Srebro et~al.(2010)Srebro, Sridharan, and
  Tewari]{srebro2010optimistic}
N.~Srebro, K.~Sridharan, and A.~Tewari.
\newblock Optimistic rates for learning with a smooth loss.
\newblock \emph{arXiv:1009.3896}, 2010.

\bibitem[Sun and Tran-Dinh(2019)]{sun2019generalized}
T.~Sun and Q.~Tran-Dinh.
\newblock Generalized self-concordant functions: {A} recipe for {N}ewton-type
  methods.
\newblock \emph{Mathematical Programming}, 178\penalty0 (1-2):\penalty0
  145--213, 2019.

\bibitem[Tartavel et~al.(2016)Tartavel, Peyr{\'e}, and
  Gousseau]{tartavel2016wasserstein}
G.~Tartavel, G.~Peyr{\'e}, and Y.~Gousseau.
\newblock Wasserstein {l}oss for {i}mage {s}ynthesis and {r}estoration.
\newblock \emph{SIAM Journal on Imaging Sciences}, 9\penalty0 (4):\penalty0
  1726--1755, 2016.

\bibitem[Ta{\c{s}}kesen et~al.(2020)Ta{\c{s}}kesen, Nguyen, Kuhn, and
  Blanchet]{taskesen2020distributionally}
B.~Ta{\c{s}}kesen, V.~A. Nguyen, D.~Kuhn, and J.~Blanchet.
\newblock A distributionally robust approach to fair classification.
\newblock \emph{arXiv:2007.09530}, 2020.

\bibitem[Ta{\c{s}}kesen et~al.(2021)Ta{\c{s}}kesen, Blanchet, Kuhn, and
  Nguyen]{taskesen2020statistical}
B.~Ta{\c{s}}kesen, J.~Blanchet, D.~Kuhn, and V.~A. Nguyen.
\newblock A statistical test for probabilistic fairness.
\newblock In \emph{ACM Conference on Fairness, Accountability, and
  Transparency}, 2021.

\bibitem[Thorpe et~al.(2017)Thorpe, Park, Kolouri, Rohde, and
  Slep{\v{c}}ev]{thorpe2017transportation}
M.~Thorpe, S.~Park, S.~Kolouri, G.~K. Rohde, and D.~Slep{\v{c}}ev.
\newblock A transportation {$L^p$} distance for signal analysis.
\newblock \emph{Journal of Mathematical Imaging and Vision}, 59\penalty0
  (2):\penalty0 187--210, 2017.

\bibitem[Thurstone(1927)]{thurstone1927law}
L.~L. Thurstone.
\newblock A law of comparative judgment.
\newblock \emph{Psychological Review}, 34\penalty0 (4):\penalty0 273, 1927.

\bibitem[Train(2009)]{train2009discrete}
K.~E. Train.
\newblock \emph{Discrete {C}hoice {M}ethods with {S}imulation}.
\newblock Cambridge University Press, 2009.

\bibitem[Tsybakov(2003)]{tsybakov2003optimal}
A.~B. Tsybakov.
\newblock Optimal rates of aggregation.
\newblock In \emph{Conference on Learning Theory}, pages 303--313, 2003.

\bibitem[Van~Leeuwen(1990)]{ref:van1990handbook}
J.~Van~Leeuwen.
\newblock \emph{Handbook of Theoretical Computer Science: Algorithms and
  Complexity}.
\newblock Elsevier, 1990.

\bibitem[Villani(2008)]{villani}
C.~Villani.
\newblock \emph{Optimal {T}ransport: {O}ld and {N}ew}.
\newblock Springer Science \& Business Media, 2008.

\bibitem[Wang et~al.(2010)Wang, Ozolek, Slepcev, Lee, Chen, and
  Rohde]{wang2010optimal}
W.~Wang, J.~A. Ozolek, D.~Slepcev, A.~B. Lee, C.~Chen, and G.~K. Rohde.
\newblock An {o}ptimal {t}ransportation {a}pproach for {n}uclear
  {s}tructure-{b}ased {p}athology.
\newblock \emph{IEEE Transactions on Medical Imaging}, 30\penalty0
  (3):\penalty0 621--631, 2010.

\bibitem[Wassenaar and Chen(2003)]{wassenaar2003approach}
H.~J. Wassenaar and W.~Chen.
\newblock An approach to decision-based design with discrete choice analysis
  for demand modeling.
\newblock \emph{Transactions of ASME: Journal of Mechanical Design},
  125\penalty0 (3):\penalty0 490--497, 2003.

\bibitem[Weed(2018)]{weed2018explicit}
J.~Weed.
\newblock An explicit analysis of the entropic penalty in linear programming.
\newblock In \emph{Conference On Learning Theory}, pages 1841--1855, 2018.

\bibitem[Xiao(2009)]{xiao2010dual}
L.~Xiao.
\newblock Dual averaging method for regularized stochastic learning and online
  optimization.
\newblock In \emph{Advances in Neural Information Processing Systems}, pages
  2116--2124, 2009.

\end{thebibliography}
\end{document}